\documentclass[format=acmsmall, review=false]{acmart}
\usepackage{acm-ec-26}
\usepackage{booktabs} 
\usepackage[ruled]{algorithm2e}
\usepackage{amsthm}
\usepackage{aliascnt}
\usepackage{url}
\usepackage{ulem}
\usepackage{csquotes}

\usepackage{algorithm}
\usepackage{algpseudocode}
\usepackage{tcolorbox}
\tcbuselibrary{theorems}
\usepackage{nicefrac}
\usepackage{wrapfig}
\usepackage{booktabs}
\usepackage{longtable}
\usepackage{array}
\usepackage{ragged2e}
\usepackage{subcaption}

\newcolumntype{P}[1]{>{\RaggedRight\arraybackslash}p{#1}}

\newcommand{\prules}{\phi^{\text{rules}}}
\newcommand{\ppriors}{\phi^{\text{priorities}}}
\newcommand{\Msep}{\cM^{\mathrm{sep}}}
\newcommand{\Reg}{\operatorname{Regret}}
\newcommand{\Range}{\operatorname{Range}}

\makeatletter
\AtBeginDocument{
  \renewcommand\subsubsection{
    \@startsection{subsubsection}{3}{\z@}
      {-.5\baselineskip \@plus -2\p@ \@minus -.2\p@}
      {3.5\p@}
      {\sffamily\itshape}
  }
  \let\ACM@origsubsubsection\subsubsection
}
\makeatother

 \newtcbtheorem{boxexample}{Example}{
   colback=gray!5,
   colframe=black!70,
   fonttitle=\bfseries
 }{ex}

 \newtcolorbox{boxexamplecont}[1]{%
   colback=gray!5,
  colframe=black!70,
   fonttitle=\bfseries,
   title=Example~\ref{ex:#1}\ (continued)
 }

\newtheorem{example}{Example}

\SetAlFnt{\small}
\usepackage[skip=5pt]{parskip}
\SetAlCapFnt{\small}
\SetAlCapNameFnt{\small}
\SetAlCapHSkip{0pt}
\IncMargin{-\parindent}

\let\todo\relax
\usepackage{todonotes}

\setcitestyle{authoryear}

\usepackage{amsmath,amsfonts,bm,bbm}
\usepackage{amsthm}
\usepackage{aliascnt}
\usepackage{enumitem}
\usepackage{xurl}
\usepackage{mathtools}

\def\eqref#1{equation~(\ref{#1})}

\def\1{\bm{1}}

\DeclareMathAlphabet{\mathsfit}{\encodingdefault}{\sfdefault}{m}{sl}
\SetMathAlphabet{\mathsfit}{bold}{\encodingdefault}{\sfdefault}{bx}{n}

\newcommand{\E}{\mathbb{E}}

\newcommand{\R}{\mathbb{R}}

\DeclareMathOperator*{\argmax}{arg\,max}

\usepackage{tikz}

\usetikzlibrary{shapes,decorations,arrows,calc,arrows.meta,fit,positioning}
\tikzset{
    -Latex,auto,node distance =1 cm and 1 cm,semithick,
    state/.style ={ellipse, draw, minimum width = 0.7 cm},
    point/.style = {circle, draw, inner sep=0.04cm,fill,node contents={}},
    bidirected/.style={Latex-Latex,dashed},
    el/.style = {inner sep=2pt, align=left, sloped}
}

\usepackage{color}

\newtheorem{theorem}{Theorem}[section]
\newaliascnt{remark}{theorem}
\newtheorem{remark}[remark]{Remark}
\aliascntresetthe{remark}
\newtheorem{definition}{Definition}[section]
\newaliascnt{assumption}{theorem}

\aliascntresetthe{assumption}

\newaliascnt{corollary}{theorem}
\newtheorem{corollary}[corollary]{Corollary}
\aliascntresetthe{corollary}

\newaliascnt{lemma}{theorem}
\newtheorem{lemma}[lemma]{Lemma}
\aliascntresetthe{lemma}

\newaliascnt{fact}{theorem}

\aliascntresetthe{fact}

\newaliascnt{observation}{theorem}

\aliascntresetthe{observation}

\newaliascnt{proposition}{theorem}
\newtheorem{proposition}[proposition]{Proposition}
\aliascntresetthe{proposition}

\newcommand{\cW}{\mathcal{W}}
\newcommand{\cM}{\mathcal{M}}

\newcommand{\rules}{\mathcal{F}}
\newcommand{\cR}{\mathcal{R}}

\newcommand{\cY}{\mathcal{Y}}
\newcommand{\cX}{\mathcal{X}}
\newcommand{\cQ}{\mathcal{Q}}

\newcommand{\cT}{\mathcal{T}}
\newcommand{\cD}{\mathcal{D}}
\usepackage{cleveref}

\crefname{theorem}{Theorem}{Theorems}
\Crefname{theorem}{Theorem}{Theorems}

\crefname{assumption}{Assumption}{Assumptions}
\Crefname{assumption}{Assumption}{Assumptions}

\crefname{lemma}{Lemma}{Lemmas}
\Crefname{lemma}{Lemma}{Lemmas}

\crefname{definition}{Definition}{Definitions}
\Crefname{definition}{Definition}{Definitions}

\crefname{remark}{Remark}{Remarks}
\Crefname{remark}{Remark}{Remarks}

\crefname{observation}{Observation}{Observations}
\Crefname{observation}{Observation}{Observations}

\crefname{corollary}{Corollary}{Corollaries}
\Crefname{corollary}{Corollary}{Corollaries}

\crefname{fact}{Fact}{Facts}
\Crefname{fact}{Fact}{Facts}

\crefname{proposition}{Proposition}{Propositions}
\Crefname{proposition}{Proposition}{Propositions}

\title[Title]{Internal Pluralism and the Limits of Pairwise Comparisons}

\author{Bailey Flanigan}
\affiliation{%
  \institution{MIT}
  \city{Cambridge}
  \state{MA}
  \country{USA}
}
\email{baileyf@mit.edu}

\author{Michelle Si}
\affiliation{%
  \institution{Harvard}
  \city{Cambridge}
  \state{MA}
  \country{USA}
}
\email{msi@g.harvard.edu}

\begin{document}

\begin{abstract}

Local pairwise comparisons are a standard tool for learning how people want decision rules to work, e.g., in participatory design or alignment. However, their use builds in two strong assumpions: that local comparisons are sufficient evidence about how a person wants an automated decision rule to behave, and that people can always answer those comparisons decisively. We investigate how these assumptions may be compromised under \textit{internal pluralism}: the idea that an individual evaluates decision rules according to multiple authoritative priorities about how the rule should behave. We provide a formal model of such pluralistic preferences over decision rules, which then lets us identify two distinct failures of forced local pairwise comparison data. First, priorities such as proportionality, egalitarianism, and equal treatment are inherently \textit{global}: what they imply in one case can depend on what happens elsewhere, so local comparisons may fail to capture them. Second, even when priorities \textit{are} representable locally, tension between strongly-held priorities can generate internal conflict, producing potentially costly behavioral distortions when comparisons are forced. We then use our model to investigate the alternative --- allowing people to report indecision --- and our findings suggest that doing so can considerably reduce the number of queries needed to learn preferences accurately. We conclude by describing how our model points toward preference-learning methods that elicit these priorities directly, yielding more faithful and interpretable accounts of what people value.
\end{abstract}

\maketitle

\setcounter{tocdepth}{1} 
\tableofcontents

\textsc{Acknowledgements.}
We thank Serena Wang, Kate Donahue, and Chara Podimata for their detailed feedback on the paper. We thank Finale Doshi-Velez, Charis Pipis, Daniel Lee, Jakob de Raaij, Marina Mancoridis, Sarah Bentley, Colin Camerer, and Maxon Rubin-Toles, for their feedback on presentations related to the paper. We thank Bernardo Zacka, Carmel Baharav, Elias Bareinboim, and Andre Ye for helpful conversations, and we also thank Vijay Keswani and Min Kyung Lee for providing access to their data. 

\textsc{AI Disclosure.} ChatGPT and Gemini were used in the course of brainstorming and exploring related work, and Claude Code was used in implementing the algorithms. We reviewed and edited the content produced by these tools and are ultimately responsible for the publication's content.

\pagebreak

\section{Introduction} \label{sec:intro}
Across areas like participatory design and AI alignment, there is increasing interest in trying to embed human values in automated systems for making societal decisions. Existing work has considered how to design algorithms emulating human judgments in, e.g., the trolley problem, motivated by automated vehicles \citep{awad2018moral, noothigattu2018voting}; how to prioritize who should receive kidneys \citep{freedman2020adapting, keswani_pros_2024, keswani2025can, boerstler2024stability, cousins2025cognitivelyfaithful}; how to allocate food to food banks \cite{lee2019webuildai, lee2017human}; how LLMs should respond to prompts \cite{ge2024axioms, huang2024collective}; how to hypothetically allocate life jackets \cite{mohsin2022learning}; and how to allocate resources to medical patients \cite{johnston_deploying_2023, dieteren2022should,gatto2026medical}. 

We will refer to a generic such setting as a \textit{decision task}, consisting of a fixed set of decision inputs $\mathcal{X}$ and decision outputs $\mathcal{Y}$, with the valid decision outputs at $x \in \mathcal{X}$ being $\mathcal{Y}(x) \subseteq \mathcal{Y}$. A \textit{decision rule} $F : \mathcal{X} \to \mathcal{Y}$ is then any automated rule for selecting an output given any input. For example, in a simple version of the trolley problem (our running example), an input $x \in \mathcal{X}$ consists of a trolley scenario and two groups of people $N_1, N_2$; the decision rule must choose a group in $\mathcal{Y}(x) = \{N_1,N_2\}$ to be spared. Let the set of decision rules $\rules$ be the set of all mappings from $\cX \to \cY$. The goal is to learn a decision rule $F \in \rules$ that is most ``aligned'' with either an individual or a group; for the purposes of this paper, we focus on aligning a decision rule to a \textit{single individual}.

Across the applications above, a common way to make such alignment problems empirically tractable is to elicit \textit{local pairwise comparisons}, which ask: given an input $x \in \cX$, would you prefer the rule output $y$ or $y'$ (where $y,y' \in \cY(x)$)?\footnote{Pairwise comparisons are a dominant modality in preference learning for alignment of language models and simpler decision rules (e.g., \cite{jiang2024survey}). Local pairwise comparisons are also part of a broader tradition of discrete-choice elicitation methods in economics, political science, healthcare, marketing, and transportation \cite{ben1985discrete,green1978conjoint,hainmueller2014conjoint, lancsar2008conducting}.} 
These comparisons are often forced, i.e., the person must pick one option or the other. After many such comparisons are collected, the learner fits a choice model---often based on Bradley-Terry \cite{Bradley1952rank} or related random-utility models---whose induced choices reproduce the person's responses as faithfully as possible.

This approach reflects two important and potentially non-neutral assumptions. First, it treats local pairwise comparisons as sufficient evidence for learning preferences over the actual decision space, $\rules$. However, it could easily be that people's beliefs about how the rule should work sometimes operate on the \textit{entire rule}, in a way that cannot be decomposed over individual inputs: for example, someone might want the rule to avoid imposing harms disproportionately on protected groups (we call this \textit{proportionality}). According to proportionality, who should be spared at \(x\) depends crucially on who is spared at other \(x' \in \cX\setminus {x}\) --- a dependence that must be somehow eliminated when the individual responds to a local pairwise comparison. We refer to such beliefs as \textit{inseparable}, as the response they dictate at an individual $x$ cannot be separated from what is done at other inputs.

The second key assumption is that people can always give decisive answers to local pairwise comparison queries. We question this assumption on the grounds of \textit{internal pluralism}: the widely-discussed idea that an individual's judgments may derive from multiple core values or objectives, which can be brought into irreducible conflict by hard trade-offs. For instance, if asked how the decision rule should work in the trolley case, the individual may articulate several priorities: in addition to proportionality, they may want the rule to save as many people as possible (call this priority \textit{size}) and to save members of their immediate family (call this priority \textit{family}). These priorities can easily be brought into conflict: for example, in the classic dilemma where \(N_1\) contains 1000 people but \(N_2\) contains the individual's mother, the size and family priorities strongly advocate opposite responses, and the individual may not be able to produce a morally authoritative decision. In such cases, forcing a decisive response erases this latent internal conflict at best, and at worst may prompt the individual to respond arbitrarily, leading to response inconsistency and incorrect conclusions by the learner.

Using a formal model that captures such pluralistic priorities, we formalize these two concerns: we examine how inseparable beliefs can be erased by local pairwise comparisons, and how even absent this problem, internal conflict between these priorities can lead to misspecification under forced comparisons. Our aim is not to reject pairwise comparisons as an elicitation tool, but to understand when they can and cannot faithfully reflect how people think the rule should work. These three contributions --- plus how our model can be applied as a learning tool --- are described below in I-IV.

\textbf{Contribution I: A Pluralistic Model of Priorities (\Cref{sec:model}).} Our first contribution is to introduce this formal model, shown in \Cref{fig:priorities}. This model's core feature is that it represents \textit{internal pluralism:} the individual's beliefs about how the rule should work are not represented by a single preference ordering over rules, but rather by a \textit{collection} of individually authoritative preference orderings over rules. We call these orderings \textit{priorities}. Each priority represents one coherent way the individual evaluates rules, such as saving as many people as possible, protecting family members, or avoiding disproportionate harm to protected groups. 

Formally, let there be \(m \in \mathbb{N}\) priorities, each represented by a complete, transitive, and weak preference relation over \(\rules\). For a given priority \(j \in [m]\), we model this relation with a utility function \(u_j : \rules \to \mathbb{R}\). Note that this captures the single-objective case when there is one priority.
\begin{figure}[h!]
    \centering
    \includegraphics[width=0.6\linewidth]{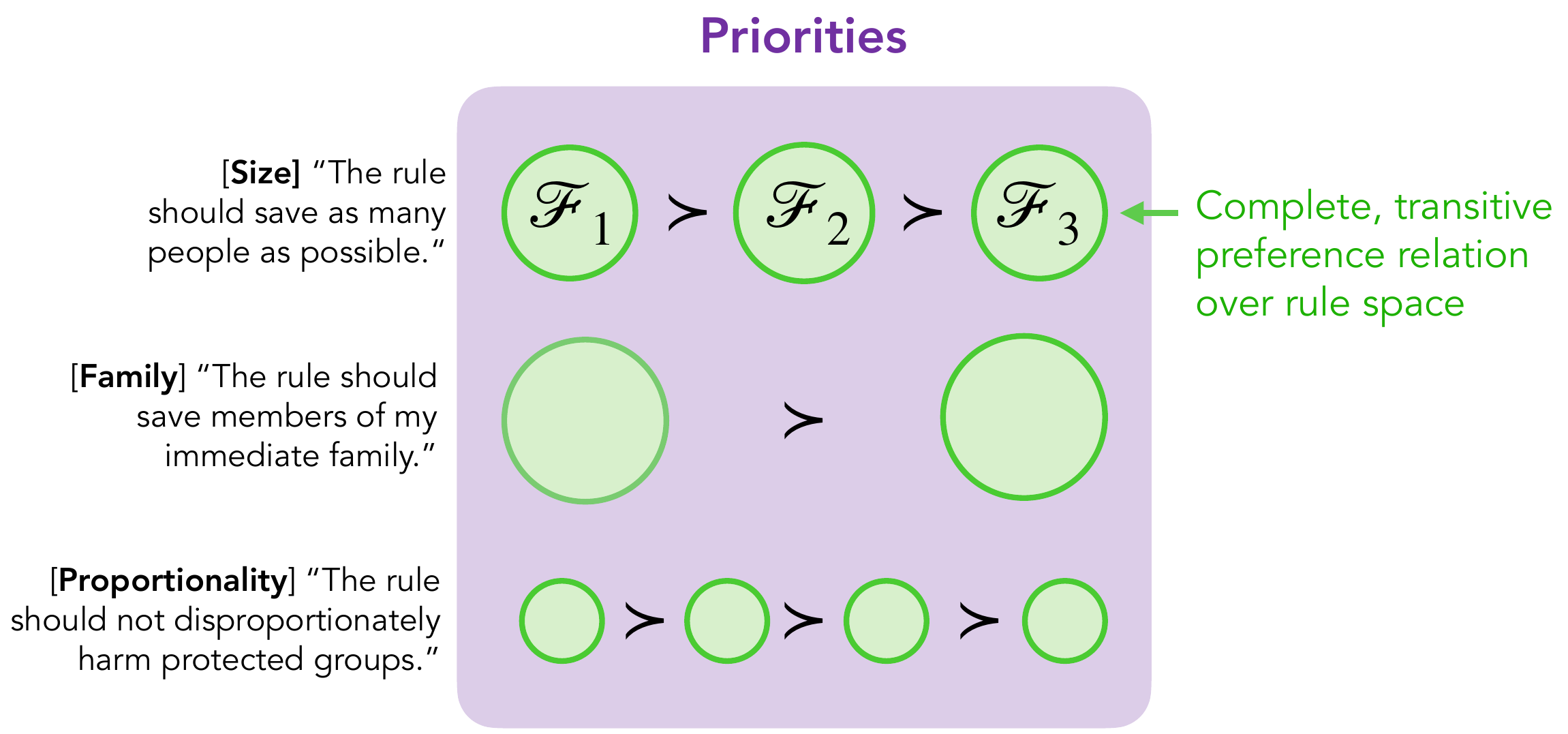}
    \caption{Depiction of priority model over decision rules. Each priority is represented by a weak, complete, transitive preference relation over the rule space. In the diagram, circles represent subsets of $\rules$ over which the individual is indifferent; because the rankings are complete, each row of circles represents a partition over $\rules$, e.g., $\rules_1 \cup \rules_2 \cup \rules_3 = \rules$. In the \textit{size} priority applied to the trolley example, $\rules_1$ could, e.g., represent the set of all rules that choose to spare the maximum number of people whenever the two sets of people are of different size.}
    \label{fig:priorities}
\end{figure}
With this pluralistic priority model in hand, we then formally model how such rule-level priorities are translated into responses to local pairwise comparison queries. This is where the two concerns above formally arise. 

First, an individual priority may be \textit{inseparable}: what it says about output \(y\) versus \(y'\) at input \(x\) may depend on what the rule does at other inputs. Then, to determine what local query response dictates, the individual must somehow collapse competing evidence depending on what the rule could do at other inputs. Second, even if each priority can individually evaluate the local query, the resulting evidence \textit{across priorities} may still fail to support a decisive answer. Some priorities may provide evidence for choosing \(y\) at \(x\), while others provide evidence for choosing \(y'\), producing \textit{conflict} (e.g., a choice between the trolley sparing your mother and 1000 people). Alternatively, no priority may provide meaningful evidence for either \(y\) or \(y'\), producing \textit{indifference} (e.g., a choice between the trolley sparing an old sandwich or a plastic bag). We refer to these states of conflict and indifference as \textit{latent states}, allowing how they manifest behaviorally to vary.

In its full generality, our pluralistic priority model thus captures both inseparable priorities and indecisive latent states --- the two components we worried may not be captured by standard pairwise learning pipelines. We use our model to formalize this intuition, too: we show that standard \textit{score-based random utility models} (S-RUMs), including Bradley-Terry, correspond exactly to the special case of our framework in which both complications disappear. In this special case, priorities are perfectly \textit{separable}, so local evidence at \(x\) is independent of the rule's behavior elsewhere; moreover, all evidence is reduced to a single score difference, and any ``indecision'' collapses to noise between decisive responses. Formally, we show that these restrictions of our model render latent indifference and conflict behaviorally irrelevant, and forced comparisons are consequently justified mechanically.

For Contributions 2 and 3, we start from this special case of our model that recovers S-RUMs, and then we study the effect of relaxing each restriction separately: first we allow priorities to be inseparable, and then we allow local queries to generate latent indifference or conflict.

\textbf{Contribution II (\Cref{sec:insep}).} We first isolate the consequences of allowing priorities to be inseparable. We show that natural classes of priorities, including proportionality, egalitarianism, and equal treatment, satisfy strong notions of inseparability. We then show that in standard pipelines, these inseparable priorities can be erased or misinterpreted, leading the learning system to infer highly suboptimal rules. More specifically, we first demonstrate that the most strongly inseparable priorities are erased without a trace when perfect separability is assumed, and moreover, their existence is not identifiable by local pairwise comparisons at all, even if the possibility of their existence is known ahead of time. For inseparable priorities that can in principle be identified, we show that inferences by learners assuming separability can be systematically distorted, due to the misinterpretation of this inseparable evidence. In this way, inseparability offers one formal explanation for apparent inconsistency in pairwise-comparison data: the inconsistency may reflect a mismatch between inseparable rule-level priorities and local elicitation.

\textbf{Contribution III (\Cref{sec:conflict}).} We then isolate the consequences of allowing local queries to generate latent indifference and conflict. We do this in a linear special case of our model, where priorities exist in a known feature space, and each priority advocates the importance of a single feature. This separable restriction of our model captures linear social choice \cite{ge2024axioms,ge_linear_2026} and related feature-based preference-learning models \cite{lee2019webuildai,freedman2020adapting,ge2024learning,noothigattu2018voting,boerstler2024stability}, but strictly generalizes them by keeping the evidence from each feature separate: some priorities may decisively support $y$ and others $y'$, and these trade-offs need not be resolvable. In simulated settings, we learn our model via Bayesian active learning under various conditions. We first analyze losses in learning accuracy when individuals resolve indecision by deviating from the learner's assumed response model, e.g., in lexicographic, biased, or random ways. We find large learning losses on certain metrics and modest losses on others. One potential way to avoid such losses would be to allow individuals to directly report their indecision. After providing evidence that such richer responses are plausible to elicit, we generalize our learner to use indecision reports and demonstrate that allowing such responses has an additional advantage: when individuals can report conflict and indifference --- or even just generic indecision conflating these two states --- such reports can substantially accelerate learning.

\textbf{Contribution IV (\Cref{sec:discussion}).} Although we apply our pluralistic priority model to formalize intuitions about the limitations of local pairwise comparisons, this model opens up a much broader possibility: rather than inferring an entire preference structure from local comparisons alone, one can learn our model to capture richer elements of preferences about how decision rules should work. In \Cref{sec:discussion} we build on our results to propose an approach to doing so: \textit{priority-aware learning}, which elicits priorities directly---through text and other structured queries---and then interprets all comparison responses through them. We discuss how such priority-based preference learning methods --- applied either within individuals or across them --- could help make preference learning more faithful, more efficient in complex decision spaces, and more interpretable, as the resulting model is composed of priorities the individual(s) themselves specified.

\subsection{Related Work}

\textbf{Model of Pluralistic Priorities (Contribution I).} At a conceptual level, our model reflects many behavioral-science accounts of value pluralism, moral conflict, and ambivalence, which we discuss in \Cref{rem:behavioral-sciences} and Appendix~\ref{app:behavioral-science}. From the elicitation literature, our model is conceptually closest to recent work on value inference \cite{SiebertEtAl2022HHAI,liscio2025value}, whose model similarly supposes that individuals make choices by aggregating over multiple values (priorities), and that the goal of learning should be to recover values' relative importance. However, their model is formally quite different: in their work, a `value'' is a textual tag (e.g., ``Cost-Effectiveness'') rather than a preference model over the decision space; these values are of ranked importance, and they are applied linearly to imply a choice over a small set of policy options. This setup eliminates the two worries our model aims to capture. First, because participants can express preferences directly over the decision space (rather than answering local questions about global decision rules, as in our case), the possibility of inseparability is removed. Second, linear aggregation over values automatically resolves trade-offs, thereby smoothing over the conflict between values we explicitly model. In the same vein, our pluralistic priority model is similar in spirit to multiattribute value and utility theory, which represents preferences as trade-offs among multiple objectives \citep{keeneyraiffa1976}. However, whereas this tradition ultimately models an overall preference or utility representation, our model keeps the underlying priorities separate and allows them to conflict.

 More broadly, our model is one of many recent attempts to learn more cognitively faithful models of moral preferences in \textit{feature-based decision settings}, where individuals must trade off morally salient features \cite{kim2018computational,mohsin2022learning,cousins2025cognitivelyfaithful}. This work is complementary to ours, presenting decision models in the restricted setting where priorities exist over predefined features (a setting our model formally captures), and exploring different methods of resolving trade-offs between features, including linear utility models, heuristic or lexicographic rules, and feature-processing rules followed by fixed comparison rules. These different ways of aggregating over features represent different instantiations of our \textit{priority aggregators}, as defined in \Cref{sec:model}.

\textbf{Separability and Inseparability (Contribution II).} 
To our knowledge, there has not been a paper formally describing the problem of inseparability in the domain of decision rule design.\footnote{Another recent alignment paper may initially appear related --- it shows that sparse pairwise comparisons cannot identify higher-order population information needed for inequality-aware objectives \cite{ge_linear_2026}. However, theirs is a distinct concern from ours: their inequality-aware objective is not inseparable in our sense, since it is evaluated on a single candidate/input.} The closest is a paper in participatory design of fair division outcomes \cite{shirali2024participatory}, which is motivated by a conceptually similar concern: that social preferences like inequality aversion cannot be ``expressed'' by standard cardinal utilities. They state this observation in passing, while our work formalizes the analogous concern as an inseparability problem. Their model is also conceptually related to ours: they propose to let people want to optimize a linear combination of multiple objectives, where these objectives correspond in spirit to our priorities. They apply their model toward a different goal; \textit{given} each agent's desired weighting over objectives, they analyze the welfare of the allocation produced by optimizing an aggregate objective.

Outside of alignment and participatory design, the importance of separability assumptions --- and their potential lack of realism, is well-recognized in the distinct setting of committee selection / multi-winner voting. In this domain, separability is a standard assumption, in that case meaning that a voter’s preference over one alternative does not depend on the rest of the set of winners \cite{lang_xia_combinatorial_2016}. This literature also points out that when preferences are inseparable, local voting becomes misspecified \cite{lang_xia_combinatorial_2016,barrot_lang_conditional_2016,ratliff_selecting_2006,ratliff_saari_diverse_2014,uckelman_alice_2010}; our model captures this same concern when the decision space consists of decision rules instead of committees.
Our work also differs in its goal: while these papers aim to document the problem with examples and/or develop richer ballots and aggregation methods, we formalize what is lost downstream when inseparability forces individuals to compress their inseparable priorities in response to local ballots.

\textbf{Conflict and Indifference (Contribution III).} The importance of indecision has been broadly highlighted by empirical work showing that pairwise comparisons can be difficult, unstable, or low-confidence, and suggesting that forced binary responses may obscure meaningful non-decisive states \citep{boerstler2024stability,keswani2025can,keswani2025moralchange}. The paper that comes the closest to our model of indecision is \citet{mcelfresh2021indecision}, which presents several formal models of indecision in discrete choices. Our indifference notion is close to their desirability-based model \textit{Min-$U$}, and our notion of conflict is a hybrid of their desirability-based and difference-based models \textit{Max-$U$} and \textit{Min-$\delta$}. The more substantial difference is that their indecision state is decided by comparing exogenous utilities over each local option, while indecision in our model is microfounded as resulting from multiple competing priorities evidencing different choices. Our work has a similar relationship to the broader set of work extending classical pairwise comparison models like Bradley-Terry to permit ties \cite{RaoKupper1967,davidson}.

These tie-permitting extensions of Bradley-Terry have been applied in LLM alignment by \citet{liu2024reward}; as we do, they find that ignoring ties comes with learning costs. The intuition for why utilizing indecision helps in our setting is closely related to active-learning results showing that information about closeness to the decision boundary can improve learning \cite{kane2017active}. In learning our model, we apply standard tools from Bayesian Active Learning by Disagreement \cite{houlsby2011bayesian}.

\textbf{Relationship to LLM alignment.} Our model can in principle represent decision rules as simple as linear models and as complicated as large language models, where $\cX$ is the space of prompts and $\cY$ is the space of responses. Our approach, examples, and intuitions are built primarily around simpler high-stakes decision tasks because these settings make the rule-level object more concrete: one can describe the inputs, outputs, and counterfactual rule behavior, abstracting away from the separate challenge of controlling a large generative model. However, our results concern the information content of local pairwise comparisons, not the tractability of optimizing over $\rules$, so they are relevant whenever pairwise comparisons are used as evidence about human values, including in LLM alignment.

\textbf{Relationship to Pluralistic Alignment.} Our work is also related to the growing literature on pluralistic alignment, which argues that aligned AI systems should account for the diversity of values, preferences, and perspectives across people and groups \cite{conitzer2024socialchoice,sorensen2024roadmap}. This literature is similar in spirit to ours in rejecting the idea that alignment can be reduced to a single objective, but our work differs in that we study the consequences of plurality \textit{within} rather than \textit{across} individuals. This distinction matters because reconciling plurality within versus across individuals pose different challenges: across individuals, the central challenge is defining an external rule for fairly making trade-offs; within an individual, the challenge is learning how \textit{they want to} reconcile trade-offs. Our results complement the pluralistic alignment literature, showing that failures due to value pluralism can arise even before the social-choice problem of aggregating across people. We discuss this further in \Cref{sec:discussion}.

\section{Model} \label{sec:model}
Let a \textit{decision task} be defined by an input space $\cX$ and output space $\cY$. We denote a single input and output by $x \in X$ and $y \in Y$, respectively. Not all outputs may be valid for a given input $x$, so we let $\cY(x)\subseteq \cY$ denote the set of permissible outputs on input $x$. While $\cX,\cY$ can be discrete or continuous, for technical convenience we take $\cX$ and $\cY$ to be discrete and finite unless otherwise stated. A \textit{decision rule} is a function $F:\cX\to\cY$ such that $F(x)\in\cY(x)$ for every $x\in\cX$. We focus on deterministic rules, but randomized rules could also be of interest. Let $\rules$ denote the set of all decision rules for a given $\cX,\cY$ (where the input/output spaces are left implicit, as they will be clear from context).

For concreteness, we will often use the running example of \textit{allocation decision tasks}. In an allocation task, each decision allocates a set of goods or bads to a set of possible recipients. Formally, there is a universe of recipients $N$ and a universe of goods or bads $K$. Each input $x\in\cX$ consists of a set of goods or bads $K(x)\subseteq K$ and a set of recipients $N(x)\subseteq N$, so $x=(K(x),N(x))$. Each output $y\in\cY(x)$ is an assignment involving the elements of $K(x)$ and $N(x)$.

\begin{example}[Trolley problem] \label{ex:trolley}
In a simple version of the trolley problem, an input $x$ consists of two groups of individuals $N_1(x),N_2(x)$ which partition $N(x)$. The decision rule must choose which group is spared. Thus $\cY(x)=\{1,2\}$, where output $\ell\in\{1,2\}$ means that group $N_\ell(x)$ is spared.
\end{example}

Finally, we define a local pairwise comparison query in the context of decision task $\cX,\cY$. A \textit{query format} is itself a decision task, with query space $\cQ$ and response space $\cW$. A query is denoted $q\in\cQ$, and a response is denoted $w\in\cW$. The set of valid responses to query $q$ is the response alphabet $\cW(q)\subseteq\cW$.

\begin{definition}[\textbf{Local Pairwise Comparison Query Format}] \label{def:pc-query}
A \textit{local pairwise comparison query} consists of an input $x\in\cX$ and two feasible outputs $y,y'\in\cY(x)$:
\[
\cQ^{\mathrm{pc}}
=
\{(y,y';x): x\in\cX,\ y,y'\in\cY(x)\}.
\]
We write a generic query as $q=(y,y';x)$. We allow there to be \textit{four} possible responses to a local pairwise comparison query:
\[
\cW^{\mathrm{pc}}(q)=\{\succ,\prec,\sim,\bowtie\} \qquad \forall q \in \cQ^{\mathrm{pc}}.
\]
As we will formalize later, responses $\succ$ and $\prec$ indicate that $y$ is preferred to $y'$ and $y'$ is preferred to $y$, respectively, $\sim$ indicates indifference between $y$ and $y'$, and $\bowtie$ indicates \textit{conflict} or incomparability, which we will describe later.
\end{definition}

\subsection{Priority Model} \label{sec:pluralistic}
The priority model is composed of $m$ priorities and their relative weights. Formally, a \textit{priority} $j\in[m]$ is a complete and transitive preference relation $\succsim_j$ over $\rules$. As usual, $F\succsim_j F'$ means that rule $F$ is weakly preferred to rule $F'$ according to priority $j$. Because each $\succsim_j$ is complete and transitive, it admits a utility representation $u_j:\rules\to\R$, so that
\[
F\succsim_j F'
\iff
u_j(F)\ge u_j(F')
\qquad \forall F,F'\in\rules.
\]
We assume utilities are normalized so that $|u_j(F)-u_j(F')|\le 1$ for all $j\in[m],\ F,F'\in\rules.$ Let $u = (u_j)_{j \in [m]}$. When we write a natural utility function whose range is not explicitly bounded by one, we implicitly take its positive affine normalization to satisfy this convention; we suppress this normalization in the notation whenever it plays no substantive role.

We additionally let each priority $j$ have a \textit{weight} $\omega_j \in [0,1]$, interpreted as its relative importance to the individual. We assume these weights are normalized to 1; that is, letting $\Delta^{m-1}$ denote the standard $m-1$ simplex, $\omega=(\omega_j)_{j\in[m]}
\in
\Delta^{m-1}.$

Then, a \textit{priority model} is a tuple $M=\big(u,\omega\big),$
consisting of the $m$ priorities' utility representations and relative weights. When $u,\omega$ are clear we will simply write $M$; when not, we will write $M = (u,\omega)$. For a fixed $\rules$, let
\(\mathcal M\) be the class of all such priority models.
Before proceeding, we give intuition via an example illustrating how natural priorities can be formalized into utility functions.

\begin{example}[Formalization of the priorities in \Cref{fig:priorities}]
\label{ex:trolley-priorities}
Here we formalize the three priorities from \Cref{fig:priorities}, as applied to the trolley problem example (\Cref{ex:trolley}), where at $x$ the decision rule must choose between sparing group $N_1(x)$ or $N_2(x)$.
Here, $\cD \in \Delta(\cX)$ is a reference distribution over inputs, and in the Proportionality priority, $G_1,\dots,G_g$ are protected groups of potential recipients and $\alpha_\ell$ is the ideal fraction of harm borne by group $\ell \in [g]$:
\begin{align*}
    u_{\mathrm{size}}(F)
&=
\Pr_{x\sim \cD}
\left[
F(x)\in \arg\max_{\ell\in\{1,2\}} |N_\ell(x)|
\right],\\
u_{\mathrm{family}}(F)
&=
\Pr_{x\sim \cD}
\left[
F(x)\in \arg\max_{\ell\in\{1,2\}}
\mathbf{1}\{N_\ell(x)\text{ contains a family member}\}
\right],\\
u_{\mathrm{prop}}(F)&=-\sum_{\ell \in [g]}\left(\E_{x \sim \mathcal{D}}[\mathbf{1}\{F(x) \text{ fails to spare
a member of group } G_\ell\}]-\alpha_\ell\right)^2.
\end{align*}
\end{example}
Of course, a textual priority can often be formalized in more than one reasonable way.\footnote{For instance, one could instead formulate these priorities as binary constraints, sorting rules by whether or not they satisfy the priority perfectly. For example, the \textit{size} priority could alternatively be formulated as $u_{\mathrm{size}}(F)
=
\mathbf{1}\left\{
F(x)\in \argmax_{\ell\in\{1,2\}} |N_\ell(x)|
\,\forall x\in\cX
\right\}$.} Which formalization is ``true'' may be ambiguous not only to the learner, but to the individual themselves. In this paper, we do not study the problem of translating textually-articulated priorities into formal priorities, but we do discuss the possibility of doing so in \Cref{sec:discussion}.

\textbf{Aggregate Utility and Regret.}
Given a priority model \(M=(u,\omega)\), we define a rule's \textit{aggregate utility} $U_M : \rules \to \mathbb{R}$ as the weighted sum of its utilities over the priorities:
\[
U_M(F)
:=
\sum_{j\in[m]}\omega_j u_j(F).
\]
We will refer to the set of rules that maximize the aggregate utility $\arg\max_{F\in\mathcal F'} U_M(F)$ as \textit{aggregate-optimal rules}.

Importantly, we do not claim that aggregate-optimal rules are most preferred by the individual --- in fact, we explicitly abstain from defining global preferences over rules (though we discuss a natural way of doing so in \Cref{rem:rule-level-indecision}). We define the aggregate utility because we find that it corresponds to notions used in the existing models we capture. We will use it as a convenient way to measure learning losses in the form of the \textit{regret:}
\begin{definition}[Regret]\label{def:regret}
   Letting $F^*$ be any aggregate-optimal rule in $\rules$, the
\textit{regret} of rule \(\widehat{F}\) is
  \[  \mathrm{Regret}_{M}(\widehat{F})
:=
U_{M}(F^*)
-
U_{M}(\widehat{F}).\]
\end{definition}

\subsection{Query-Level Preferences: the \textit{Latent State}}
\label{sec:pc-queries}
Here, we model how rule-level priorities are used to construct beliefs about local pairwise comparison queries. The key primitive in defining this translation is a projection from rule space onto the query $(y,y';x)$:
For any rule $F\in\rules$, input $x\in\cX$, and output $y\in\cY(x)$, define the \textit{local projection} of $F$ onto $x,y$, called $F_{x\to y}$, as the rule obtained by surgically changing $F$'s output at $x$ to $y$, while leaving its behavior elsewhere fixed.
\begin{equation} \label{eq:projection}
F_{x\to y}(x')
=
\begin{cases}
y & \text{if } x'=x,\\
F(x') & \text{if } x'\neq x.
\end{cases} \quad \forall x' \in \cX.
\end{equation}
When it is used in a projection, $F$ is known as a ``background'' rule, because it is defining what is being done in the ``background'' of $x$, i.e., at all other inputs. Now, for any local pairwise comparison query $q=(y,y';x)$ and priority $j$, we apply the projection to define the marginal gain of choosing $y$ over $y'$, with respect to background rule $F$:
\begin{equation} \label{eq:local-evidence}
\Delta_j^F(q)
=
u_j(F_{x\to y})-u_j(F_{x\to y'}) \qquad \forall F \in \rules.
\end{equation}
We call this the \textit{projection gap} for query $q$ with respect to $F$. If $\Delta_j^F(q)>0$, then, given the background rule $F$, priority $j$ favors output $y$ over $y'$ at $x$. If $\Delta_j^F(q)<0$, it favors $y'$ over $y$. If $\Delta_j^F(q)=0$, it is indifferent between the two projected rules. Conceptually, this is saying: if at all $x'$ outside of $x$, the rule behaves according to $F$, would priority $j$ prefer the rule output $y'$ or $y$?

Here arises the first main complication: in answering a query $q$, the individual must give a \textit{global judgment} about whether to choose $y$ or $y'$. In contrast, the above quantity depends on the background rule $F$ and the priority $j$, meaning we need to aggregate over both background rules $F \in \rules$ and priorities $j \in [m]$ to produce a global judgment. 

We will assume aggregation happens over rules first, and then over priorities. This is natural because if priority aggregation instead came first, the individual would be implicitly tracking a quantity for every background rule in $\rules$ (enormous) rather than one for every priority $[m]$ (small). The core intuition is captured fully under this assumption; we leave the extensions of our results to other aggregation orders to future work. 

Over the next subsections, we define rule aggregation (\Cref{sec:rule-agg}), priority aggregation (\Cref{sec:prior-agg}), and the individual's resulting global query preference, known as their \textit{latent state} (\Cref{sec:latent-state}). We depict these three steps from left to right in \Cref{fig:latent}.
\begin{figure}[h!]
    \centering
    \includegraphics[width=\linewidth]{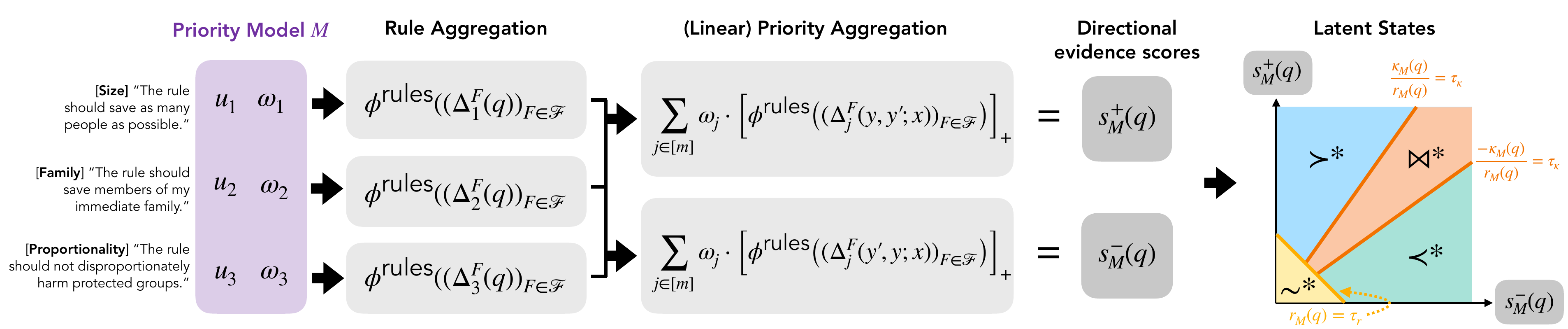}
    \caption{Example of how we go from the priority model $M$, to the directional evidence scores $s^+_M(q),s^-_M(q)$ by aggregating over rules and then priorities, to how these scores translate to latent states.}
    \label{fig:latent}
\end{figure}

\subsubsection{Rule Aggregation}
\label{sec:rule-agg}
As shown in \Cref{fig:latent}, rule aggregation happens within each priority $j$. It is done by a \textit{rule aggregator}, which aggregates the projection gaps across rules $\{\Delta_j^F(y,y';x) | F \in \rules\}$ into a single number describing how strongly priority $j$ prefers $y$ over $y'$. In the definition below, this corresponds to plugging $\Delta_j^F$ in for $z^F$.
\begin{definition}[Rule Aggregator]
    A \textit{rule aggregator} is a function $\prules:\R^{|\rules|}\to\R$.
    
    We assume two regularity conditions on rule aggregators: 
    \begin{enumerate}
        \item $\prules$ is \textit{permutation-invariant} iff the aggregator depends only on the multiset of evidence values, but not their order; that is, for every vector $z \in \mathbb{R}^{|\rules|}$ and every bijection $\sigma:\rules\to\rules$,
        \[
        \prules\big((z^F)_{F\in\rules}\big)
        =
\prules\big((z^{\sigma(F)})_{F\in\rules}\big).
        \]
        \item $\prules$ is \textit{unanimous} iff, for any constant $z\in\R$,
\[
\prules\big((z)_{F\in\rules}\big)=z.
\]
    \end{enumerate}
\end{definition}
Natural examples of rule aggregators satisfying these regularity conditions include the maximum and average:
\[
\prules_{\mathrm{avg}}\big((z^F)_{F\in\rules}\big)
=
\frac{1}{|\rules|}\sum_{F\in\rules}z_F,
\qquad
\prules_{\max}\big((z_F)_{F\in\rules}\big)
=
\max_{F\in\rules}z_F, \quad \prules_{p\text{-Pct}}\big(\{z_F\}_{F\in\rules}\big)
=
\mathrm{Pct}_p\big(\{z_F\}_{F\in\rules}\big).
\]
Except in simulations and some examples where instantiation is necessary or useful for clarity, we will reason about generic $\prules$. 

\textbf{The \textit{Separable} Special Case.} As we will see, this rule aggregation step can erase  competing directional evidence across background rules. We now define a special kind of priority --- a \textit{perfectly separable}  priority --- within which this loss does \textit{not} occur, so long as $\prules$ is unanimous.

\begin{definition}[Perfect Separability]
    A priority $j$ is \textit{perfectly separable} at query $q = (y,y';x)$ iff there exists some constant $\delta_j(q) \in \mathbb{R}$ such that 
    \[\Delta_j^F(y,y';x) = \delta_j(q) \qquad \text{for all }F \in \rules.\]
\end{definition} 
Note that when $j$ is perfectly separable, for any $q$ and any unanimous $\prules$, it holds that
\[\prules(\{\Delta_j^F(q)\}_{F \in \rules}) = \delta_j(q).\]
Conceptually, what this is saying is that according to priority $j$, what should be done $x$ is \textit{completely independent} of what decisions are made at other inputs.\footnote{For intuition, the natural weakening of this notion, which we will not formally use, is its directional (i.e., non-cardinal) analog, \textit{separability}. A priority $j$ is separable at query $q = (y,y';x)$ iff $\mathrm{Sign}(\Delta^F_j(y,y';x))$ is constant over all $F \in \rules$, i.e., the preference between $F_{x \to y}$ and $F_{x \to y'}$ dictated by $u_j$ must go in the same direction for all background rules $F$.} 
If a priority is perfectly separable at every query $q \in \cQ^{\mathrm{pc}}$, we say that priority is perfectly separable. If every priority $j \in [m]$ in a model $M$ is perfectly separable, we say that model is perfectly separable. We let $\cM_{\mathrm{sep}} \subseteq \cM$ be the class of perfectly separable pluralistic models.

\subsubsection{Priority Aggregation} \label{sec:prior-agg}
Once evidence is tallied across rules \textit{within} each priority, we must then aggregate evidence \textit{across} priorities. The key intuition we want to capture in this aggregation is that the individual may not be able to resolve evidence across priorities that conflicts, i.e., cases where certain priorities advocate for $y$ and others for $y'$. We thus define our priority aggregation process to produce two directional evidence scores $s^+_M(q)$ and $s^-_M(q)$, which respectively tally evidence in favor of $y$ over $y'$ and $y'$ over $y$ \textit{separately,}  rather than collapsing cross-priority evidence into a single score.

For underlying priority model $M$, rule aggregator $\prules$, query $q=(y,y';x)$ and notation $[a]_+=\max\{a,0\}$, let the \textit{directional evidence scores} be
\begin{equation} 
    s^+_{M}(q)
=
\sum_{j=1}^m \omega_j [\prules\big((\Delta_j^F(y,y';x))_{F\in\rules}\big)]_+,
\quad
s^-_{M}(q)
=
\sum_{j=1}^m \omega_j [\prules\big((\Delta_j^F(y',y;x))_{F\in\rules}\big)]_+.
\end{equation}
Here, we are assuming that evidence is tallied across priorities linearly in $\omega$; this will be sufficient to capture classical models in the literature, and corresponds conveniently with our notions of aggregate utility and regret. Some results will hold for more general priority aggregation methods, which we discuss where relevant.

\subsubsection{Latent State} \label{sec:latent-state}
The directional evidence scores $s^+_M(q)$ and $s^-_M(q)$ represent global evidence, i.e., evidence aggregated over all elements of the model $M$. We summarize the relative values of these scores into two key intermediates: the \textit{valence} $r_M(q)$ describes the strength of the total evidence generated across priorities, and the \textit{decisiveness} $\kappa_M(q)$ describes the extent to which there is a clear choice:
\begin{equation}\label{eq:shorthand}
    r_M(q)
:=
s_{M}^+(q)+s_{M}^-(q), \qquad \text{and} \qquad \kappa_M(q)
=
s_{M}^+(q)-s_{M}^-(q).
\end{equation}
When $M$ is clear from context, we will drop it from the notation. The latent states are derived from these quantities via thresholds $\tau_r\in[0,2]$ and $\tau_\kappa\in[0,1]$, which respectively dictate how much valence is required to avoid latent indifference, and how much decisiveness is required to avoid latent conflict:\footnote{Note that $s^+_M,s^-_M \in [0,1]$ because for all $j$, $|u_j(F) - u_j(F')|\leq 1$. Thus, these thresholds operate on the correct scale.}

\begin{definition}[\textbf{Latent State}]
\label{def:threshold-response-model}
Fix a priority model $M$, query $q \in \cQ^{\mathrm{pc}}$, valence and decisiveness $r_M(q)$ and $\kappa_M(q)$ (dropping $M$ subscripts on $r_M$ and $\kappa_M$ for readability), and thresholds $\tau_r,\tau_\kappa$. The individual's latent state $L_{M}(q)$ is
\[
 L_{M}(q)=
\begin{cases}
y \succ^* y'
& \text{if  }\ r(q) - \tau_r \geq 0  \ \ \text{ and } \ \ \ \ \  \kappa(q) - \tau_\kappa r(q) \geq 0,\\[0.4em]
y \prec^* y'
& \text{if  }\  r(q) - \tau_r \geq 0 \  \ \text{ and } \ -\kappa(q) - \tau_\kappa r(q) \geq 0,\\[0.4em]
y\bowtie^* y'
& \text{if  }\  r(q) - \tau_r \geq 0
 \ \  \text{ and } \ \ \ 
|\kappa(q)| - \tau_\kappa r(q) < 0 \\[0.4em]
y\sim^* y' & \text{if  } \ r(q) - \tau_r < 0 .
\end{cases}
\] 
Note that in the special case where $\tau_\kappa = \kappa(q) =0$, this definition is technically improper as it assigns two states: $y \succ^* y'$ and $y \prec^* y'$. In this case, arbitrarily assign the state to either $\succ^*$ or $\prec^*$.
\end{definition}
We diagram these states in \Cref{fig:latent}, which are most easily understood by considering the relative values of $s_M^+,s_M^-$ they reflect. Decisive states $\succ^*$ and $\prec^*$ occur when one score is high and the other is low, i.e., overwhelming evidence in favor of one outcome and not the other. The other two states, $\sim^*$ and $\bowtie^*$, correspond to distinct indecisive states: \textit{indifference} $\sim^*$ arises when both scores are too low, i.e., no priority offers strong evidence in favor of either response over the other. \textit{Conflict} $\bowtie^*$, in contrast, arises when both scores are high and are similar, i.e., there is strong evidence for \textit{both} outcomes. To illustrate this difference in the trolley example, indifference might occur when the individual must decide whether the trolley should hit a plastic bag or an old sandwich; conflict might occur when they must decide between their mother and 1000 strangers.

\begin{remark}[Foundations from the Behavioral Sciences]\label{rem:behavioral-sciences}
Our model is closely related to value-pluralist accounts of moral and political judgment, especially Tetlock's value pluralism model of ideological reasoning \cite{tetlock1986value}. In that literature, people are understood as reasoning from multiple core values or objectives, rather than from a single all-purpose preference ordering. Different decisions may activate different values, and morally difficult cases arise when multiple active values point in different directions. In our formalism, the priorities play the role of these distinct values, and moral difficulty is exactly captured by our conflicted latent state. Selective ``value activation'' can also easily be formalized in our model: a priority is inactive on a query $q = (y,y';x)$ iff $u_j(F_{x \to y}) - u_j(F_{x \to y'}) = 0$ for all $F \in \rules$, in which case it contributes nothing to the directional evidence scores $s_M^+$ and $s_M^-$. This definition manifests intuitively: for instance in our trolley example (\Cref{ex:trolley-priorities}), the \textit{size} priority would be inactive on a choice between two sets of 3 recipients. Additional work contributing accounts of internal pluralism include Moral Foundations Theory, which models moral judgment as drawing on multiple partially distinct concerns such as harm, fairness, loyalty, and authority \cite{graham2013moral}; and work on moral dilemmas, which shows that responses often reflect sensitivity to multiple considerations, such as consequences and moral norms \cite{gawronski2017consequences}.

The concept of conflict in the presence of difficult trade-offs is well-established (e.g., \cite{rosas2019decision}). Our distinction between indifference and conflict is closely related to the evaluative-space model of attitudes, which argues that positive and negative evaluations are not simply opposite ends of one scale, but should be decoupled \cite{CacioppoBerntson1994EvaluativeSpace}. Our model formalizes this intuition in
$s_M^+(q)$ and $s_M^-(q)$, which separately track evidence in favor of \(y\) over \(y'\) and evidence in favor of \(y'\) over \(y\); our \textit{indifference} and \textit{conflict} states correspond to their states of \textit{neutrality} and \textit{maximal conflict}. More generally, our explicit modeling of indecision as behaviorally relevant is consistent with work on choice under conflict, which shows that difficult trade-offs
can produce deferral and no-choice responses
\cite{TverskyShafir1992ChoiceUnderConflict}. We give a more thorough treatment of related work from the behavioral sciences literature in \Cref{app:behavioral-science}.
\end{remark}

\subsection{Query Response Model} \label{sec:complete-response-model}
In contrast to the latent state \(L_{M}(q)\), which is unobserved, a \textit{query response model} describes the query response that is actually observed. These two elements may come apart, e.g.,  due to noise and/or distortions due to restrictions in the permissible responses. 
Formally, a query response model is a mapping
        \[
        R : \mathcal M \to \left(\mathcal Q^{pc}\to\Delta(\mathcal W^{pc})\right).
        \]
       Then, given a query $q$ and a model $M$, the resulting object $R(q;M)$ is a distribution over $\mathcal W^{pc}$. To model randomness in responses, we consider query response models that are parameterized by a \textit{link function} $h$, which describes the functional form of the noise:
\begin{definition}[Link Function]
    A \textit{link function} $h : \mathbb R\to(0,1)$ is a continuous, strictly increasing function that satisfies
\[
\lim_{t\to-\infty}h(t)=0,
\qquad
\lim_{t\to+\infty}h(t)=1,
\qquad
h(-t)=1-h(t)
\quad \forall t\in\mathbb R.
\]
Let $h_\beta(t):=h(\beta t)$ be the link function with scalar inverse temperature parameter \(\beta > 0\), which makes the responses less noisy as $\beta$ gets larger.
\end{definition}

Now, we will define the class of local pairwise comparison query response models we will study, called \textit{baseline} query response models and denoted as $R^\circ_{h_\beta;\tau_r,\tau_\kappa}$. The $\circ$ denotes the baseline designation; the response models in this class vary over the choice of link function $h$ and $\beta$, as well as the thresholds $\tau_r,\tau_\kappa.$ This response model corresponds to the individual reporting a noisy version of their latent state.\footnote{Here, we only apply noise at the boundaries between $\succ^*,\prec^*,$ and $\bowtie^*$ because this is sufficient for our investigation, but one could noise the latent state in many ways. The only property of this noising required for our results is that when $\tau_r = \tau_\kappa = 0$, the model collapses to the zero-threshold case described in \Cref{lem:zero-thresh}.}

\begin{definition}[Baseline Query Response Model]
\label{def:baseline-response-model}
Fix thresholds \(\tau_r\in[0,2]\), \(\tau_\kappa\in[0,1]\), a link function
\(h\), and $\beta > 0$. Given a pluralistic model $M \in \cM$ and query $q=(y,y';x)$,
the baseline response model $R^\circ_{h_\beta;\tau_r,\tau_\kappa}$ gives the response distribution (dropping the $M$ subscripts for readability):
\[
R^{\circ}_{h_\beta;\tau_r,\tau_\kappa}(q;M)(y \triangleright y')
=
\begin{cases}
\mathbf{1}\{r(q) - \tau_r\geq 0 \} \cdot h_\beta(\kappa(q) - \tau_\kappa r(q))
&\triangleright=\,\succ,\\[6pt]
\mathbf{1}\{r(q) - \tau_r\geq 0 \} \cdot h_\beta(-\kappa(q) - \tau_\kappa r(q))
&\triangleright=\,\prec,\\[6pt]
\mathbf{1}\{r(q) - \tau_r\geq 0 \} \cdot \left(1-h_\beta(\kappa(q) - \tau_\kappa r(q))-h_\beta(-\kappa(q) - \tau_\kappa r(q))\right)
&\triangleright=\,\bowtie,\\[6pt]
\mathbf{1}\{r(q) - \tau_r< 0 \}
&\triangleright=\,\sim.
\end{cases}
\]
\end{definition}
Note that for any link function $h$, as $\beta \to \infty$ this response model corresponds to the individual deterministically reporting their latent state.\footnote{This is true except at points exactly on the response boundary where $\kappa(q) = \tau_\kappa r(q)$; then, by $h(t) = 1-h(-t)$, the individual must randomize 50/50 between the relevant decisive response ($\succ$ or $\prec$) and conflict ($\bowtie$).} Note also that we distinguish the latent state from the query response with an $^*$, where the $^*$ designates the latent state and its absence designates the reported relation.

\textbf{The \textit{Zero-Threshold} Special Case.} One important special case occurs when $\tau_r = \tau_\kappa = 0$. The key observation is that in this case, indecisive latent states and query responses become impossible:
\begin{lemma}[Zero-threshold case] \label{lem:zero-thresh}
   When $\tau_r = \tau_\kappa = 0$, the latent states reduce to 
   \[
     L_{M}(q)=
    \begin{cases}
    y \succ^* y'
    & \text{if  }\kappa(q) \geq 0,\\[0.4em]
    y \prec^* y'
    & \text{if  }\  -\kappa(q) \geq 0
    \end{cases}
    \] 
   and for any link function $h$ and $\beta > 0$, the baseline query model reduces to
\[R_{h_\beta;0,0}^{\circ}(q;M)(y\triangleright y')
 = \begin{cases}
    h_\beta(\kappa(q)) & \text{if }\  \triangleright = \, \succ\\
    h_\beta(-\kappa(q)) & \text{if } \ \triangleright = \,  \prec\\
    0 & \text{if } \ \triangleright \in \{\sim,\bowtie\}.\end{cases} \]
\end{lemma}
\begin{proof}
    The latent states follow by definition. In the response model the first two probabilities are by definition; then by the fact that $h(-t) = 1-h(t)$, it follows that the probability of $\succ$ and $\prec$ responses must add to 1, and thus the remaining possible responses occur with 0 probability.
\end{proof}
\subsection{Key Preliminary: S-RUMs as the \textit{Perfectly Separable, Zero-Threshold} Special Case}
\label{sec:SRUM}

We now relate our model to \textit{score-based random utility models} (S-RUMs), a popular class of choice models that assume that every feasible local output can be assigned a single latent score, and that pairwise comparisons are generated by noisily comparing these scores:

\begin{definition}
\label{def:score-rum} A score-based random utility model (S-RUM) is
defined by a local score function
\[
    V:X\times Y\to\mathbb R_{\geq 0}.
\]
Let \(\mathcal V\) be the class of feasible local score functions. For any $\mathcal V$, a link function $h$, and $\beta > 0$, the S-RUM is a query response model   $S_{h_\beta}
:
\mathcal V
\to
\left(\mathcal Q^{pc}\to\Delta(\mathcal W^{pc})\right)$ such that, for any query $q$ and any $V\in \mathcal V$, $S_{h_\beta}(q;V)$ is the following distribution over possible responses:
\begin{align*}
    S_{h_\beta}(q;V)(y\succ y')
    &=
    h_\beta\bigl(V(x,y)-V(x,y')\bigr) \quad \text{and} \quad S_{h_\beta}(q;V)(y'\succ y)
    =
    1-
    S_{h_\beta}(q;V)(y\succ y').  
\end{align*}
\end{definition}
Note that S-RUMs capture several choice models popular in alignment, most notably Bradley-Terry with the logit link function $h_\beta^{\mathrm{logit}}(t)
=
(1+\exp(-\beta t))^{-1}$, and Thurstone-Mosteller \cite{thurstone1927law,mosteller1951remarks} with the probit link function $h_\beta^{\mathrm{probit}}(t)
=
\Phi(\beta t)$. 
We now show that S-RUMs correspond exactly to a version of our model that is heavily restricted on two key dimensions: that all priorities are \textit{perfectly separable}, and that the thresholds $\tau_r,\tau_\kappa$ are zero, i.e., indifference and conflict are impossible. In this reduction, $V$ is the analog of $M$ and the response model $S_{h_\beta}$ is the analog of $R^\circ_{h_\beta;0,0}$. We show this correspondence in both the response behavior $S_{h_\beta}$ and $R^\circ_{h_\beta;0,0}$, and the natural choice of optimal decision rule under $V$ and $M$.\footnote{Technically, the utilities defined in this construction may violate the normalization convention that \(|u_j(F)-u_j(F')|\leq 1 \ \forall F,F' \in \rules\). This is not a substantive issue here because that normalization is just to ensure the thresholds $\tau_r,\tau_\kappa$ are appropriately scaled, and here they are zero. If desired, one can instead normalize these utilities and absorb the resulting rescaling into the inverse-temperature parameter \(\beta\), yielding the scale-free version of the correspondence formalized later in \Cref{def:scale-free-indistinguishability}.}

\begin{theorem}[S-RUMs as the Separable, Zero-Threshold Case]
\label{thm:score-rum-equivalence}
Fix any link function \(h\) and \(\beta>0\). 
For every local score function \(V\in\mathcal V\), there exists a
perfectly separable priority model \(M_V\in\mathcal M_{\mathrm{sep}}\)
such that
\[
R^\circ_{h_\beta;0,0}(q;M_V)
=
S_{h_\beta}(q;V)
\qquad
\forall q\in\mathcal Q^{pc},
\]
and for every \(M\in\mathcal M_{\mathrm{sep}}\), there
exists a local score function \(V_M\in\mathcal V\) such that
\[
R^\circ_{h_\beta;0,0}(q;M)
=
S_{h_\beta}(q;V_M)
\qquad
\forall q\in\mathcal Q^{pc}.
\]

Moreover, for any corresponding pair \((V',M')\in\{(V,M_V),(V_M,M)\}\), for any subset of rules $\rules' \subseteq \rules$ the learner might consider, the aggregate-optimal rule is the same:
\[
\arg\max_{F\in\mathcal \rules'}
\sum_{x\in\mathcal X} V'(x,F(x))
=
\arg\max_{F\in\mathcal \rules'}
U_{M'}(F).
\]
\end{theorem}

\begin{proof}[Proof Sketch]
The proof is deferred to \Cref{app:score-rum-equivalence}. The construction in both directions is simple: Given $V$, $M_V$ is constructed with a single priority with utility function $u_V(F):= \sum_{x \in \cX} V(x,F(x))$. Given perfectly separable $M$, $V_M$ is constructed such that $V(x,y) = \sum_{j \in [m]} \omega_j V_j(x,y)$, where $V_j(x,y) = u_j(F_{x \to y})$ for an arbitrary background rule $F$. The key in both cases is showing that for all queries $q = (y,y';x)$, $V(x,y) - V(x,y') = s^+(q) - s^-(q)$, as these are the differences on which the probability link functions depend in the respective response models. The optimal rule correspondence holds by perfect separability of the priorities in $M'$, which allows the utility impact of the rule's behavior to decompose across all $x$, making the two sums equivalent for every $F$ (up to a shift, which is irrelevant to the optimization).
\end{proof}
\section{Generalization \#1:  Inseparability}  \label{sec:insep}
Now, we ask: what happens when the standard assumption of perfect separability no longer holds --- i.e., when the individual's priority model can be in $\cM \setminus \Msep$? To isolate this generalization all else held equal, we keep the standard restriction that $\tau_r,\tau_\kappa = 0$ throughout this section. 

We begin by illustrating that this generalization from $\Msep$ to $\cM$ is practically relevant --- that realistic priorities violate separability, and can do so to the maximum possible degree. We formalize such violations as \textit{inseparability} and \textit{perfect inseparability}, defined below.
\begin{definition}[(Perfect) Inseparability] \label{def:insep}

    A priority $j$ is \textit{inseparable} at query $q = (y,y';x)$ iff there exists $F,F' \in \rules$ such that
    \[u_j(F_{x\to y}) > u_j(F_{x \to y'}) \qquad \text{and}\qquad u_j(F'_{x\to y}) < u_j(F'_{x \to y'}).\]
    A priority \(j\) is \textit{perfectly inseparable} at query \(q=(y,y';x)\) iff
    there exists \(\delta>0\) and a partition \(\mathcal F_y,\mathcal F_{y'}\)
    of \(\mathcal F\) such that \(|\mathcal F_y|=|\mathcal F_{y'}|\) and
    \[
    \Delta^F_j(y,y';x)=\delta \quad \forall F\in \mathcal F_y,
    \qquad
    \Delta^F_j(y,y';x)=-\delta \quad \forall F\in \mathcal F_{y'}.
    \]
    A priority $j$ is (perfectly) inseparable when it is (perfectly) inseparable at every query $q \in \cQ^{pc}$.

\end{definition}
 \Cref{fig:separability} illustrates the four separability-related notions we define, from most to least separable. As shown, perfect inseparability represents the strongest kind of inseparability, where the priority yields \textit{exactly} equal and opposing evidence for each response $y$ and $y'$. 
\begin{figure}[h!]
    \centering
    \includegraphics[width=0.92\linewidth]{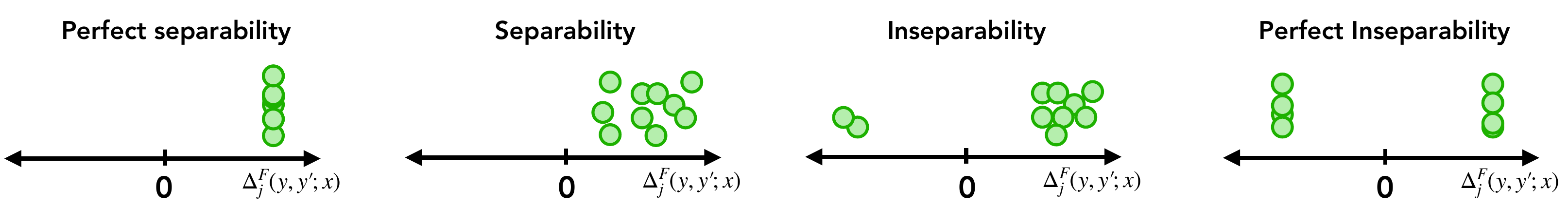}
    \caption{Illustrations of Separability/Inseparability notions. Each dot's horizontal position represents the value of $\Delta_j^F(q)$ for an $F \in \rules$. The defining characteristic of separability is that all dots appear on the same side of 0 (and for perfect separability, in exactly the same horizontal position). Oppositely, inseparability occurs when dots appear on opposite sides of 0, and perfect inseparability reflects the case where they appear in symmetric clusters equidistant from 0.}
    \label{fig:separability}.
    \vspace{-1em}
\end{figure}

We now show that in fact, many natural
priorities are inseparable, because inseparability simply requires the priority to evaluate a rule \textit{across inputs}. We point out two types of inseparable priorities recurring in the literature, though there are almost certainly others. First,
\textit{distributional priorities} evaluate the allocation of benefits or harms across
individuals or groups. This class includes
canonical priorities like \textit{egalitarianism} and \textit{proportionality}, which we formalize below. Second, \textit{axiomatic priorities} require the rule to satisfy consistency conditions across related inputs, of which
\textit{equal treatment} is our canonical example. The salience of these issues to people is supported by substantial research on distributive fairness and procedural justice (e.g., \citep{cappelen2007pluralism,fisman2007individual,colquitt2001dimensionality,leventhal1980what}).

We will use egalitarianism as the main illustrative example. To formalize ``benefit'' and ``harm'', let $v_i : \mathcal{X} \times \mathcal{Y} \to \mathbb{R}$ describe $i$'s benefit associated with the outcome, so $v_i(x,y)$ is their benefit (or if negative, harm) from the rule selecting $y$ at $x$. Because distributional priorities measure how benefits are distributed \textit{over inputs}, we let $\mathcal{D} \in \Delta(\mathcal{X})$ be a generic distribution over inputs.

\begin{definition}[\textbf{Egalitarianism}] \label{def:egal} \upshape \textit{Egalitarianism} reflects the priority that ``the rule should make sure no one receives too little benefit in the long run.'' For a generic input distribution $\mathcal{D} \in \Delta(\mathcal{X})$,\footnote{The 2-scaling is just to make the numbers more convenient. This is just one possible formulation of this priority, but conceptually, egalitarian's structural inseparability is not due to the specific formulation, but because it considers the distribution of benefits over inputs.} 
  \[
    u_{\mathrm{egal}}(F)=2 \cdot \min_{i\in N}\mathbb{E}_{x\sim\mathcal{D}}\!\left[v_i\bigl(x,F(x)\bigr)\right] \qquad \forall \, F \in \rules.
    \]
\end{definition}
To see why egalitarianism is structurally inseparable, consider an allocation
problem in which each input asks whether a good should be allocated to
recipient \(a\) or recipient \(b\). The egalitarian value of this local
choice depends on who is worst off, which depends centrally on the rule's behavior elsewhere. Concretely, for a query comparing \(a\) and \(b\), there are background rules under
which \(a\) has been maximally shortchanged at all other inputs, in which
case egalitarianism favors allocating the good to \(a\). There are also
background rules under which \(b\) has been maximally shortchanged, in
which case egalitarianism favors allocating the good to \(b\). We formalize this intuition in the proposition below; the proof is in \Cref{app:egal-insep}. 

\begin{proposition}[Egalitarianism is inseparable (in non-degenerate cases)]
\label{prop:egal-insep}
Define an allocation task with recipients $N$ and items $K$. Let $|N| \geq 3$ and fix input space $
\mathcal X=\{(a,b;k): a,b\in N,\ a\neq b,\ k\in K\}$ and output space $\mathcal Y((a,b;k))=\{a,b\}.$
Let \(\mathcal D\) have full support, i.e., \(\cD(x)>0\) for all \(x\in\cX\). Define the benefits such that they are nonnegative, and each recipient benefits only when they receive a good:
\[
v_j((a,b;k),y)>0 \iff y=j
\qquad
\text{for all } j\in N,\ a,b\in N,\ a\neq b,\ k\in K,\ y\in\{a,b\}.
\]
Then, the egalitarian priority is inseparable at every local pairwise query \(q\).
\end{proposition}
This result is proven for an allocation task over individual recipients, but similar propositions can be proven for other types of allocation tasks; in each setup, one simply needs to avoid degeneracy in $\cD$. For example, suppose we are allocating aid over \textit{sets} of recipients, and let $\cD$ be restricted such that recipient \(i\) appears rarely under \(\mathcal D\) (so \(i\) can be worst off), but whenever \(i\) does appear, they appear in both possible sets of recipients.  Then, on such queries, the egalitarian priority cannot generate evidence in favor of one side on \(i\)'s behalf, and its inseparability does not bind.

Now, we show something stronger: that in some instances, egalitarianism can be \textit{perfectly} inseparable at every query. We use an extremely simple example for illustrative purposes. We will re-use this example in later results.
    
\begin{proposition}[Egalitarianism can be perfectly inseparable] \label{prop:egal}
     For certain $\cX,\cY$, \textsc{Egalitarianism} can be perfectly 
    inseparable on all local pairwise comparison queries $q \in \cQ^{\text{pc}}$.
\end{proposition}
\begin{proof}
Define an allocation task with two recipients, $N = \{a,b\}$ and two goods $K = \{k_1,k_2\}$. Suppose the task is to allocate one good to one recipient; thus there are two possible inputs, $x_1 = (a,b;k_1)$ and $x_2 = (a,b;k_2)$ and $\cY(x) = \{a,b\}$ for $x \in \{x_1,x_2\}$. Suppose $\cD$ is uniform, so $\Pr_{x\sim \cD}[x = x_1] = \Pr_{x\sim \cD}[x = x_2] = 1/2$. Suppose both recipients have the same benefit for either good, and no benefit if they aren't given a good:
\[v_i(x,j) = \mathbf{1}\{j = i\} \qquad \text{for all } \ i,j \in \{a,b\}.\]
In this instance, there are four possible deterministic rules: $\rules = \{F^{aa}, F^{ab}, F^{ba}, F^{bb}\}$, where $F^{ij}$ is the rule that assigns $x_1 \to i$ and $x_2 \to j$ for all $i,j \in \{a,b\}$. 
As expected, the egalitarian priority prefers rules that spread out the goods over recipients:
    \[u_\text{egal}(F^{ab}) = u_\text{egal}(F^{ba}) = 1; \quad u_\text{egal}(F^{aa}) = u_\text{egal}(F^{bb}) = 0.\]
   We immediately see that this priority is perfectly inseparable at both possible queries: applying \Cref{def:insep}, at each query we have that $\delta = 1$, and the following equal-size partitions of rules:
   \begin{align*}
       q_1=(a,b;x_1):& \quad \rules_a = \{F^{bb}, F^{ab}\}, \ \  \rules_b = \{F^{aa}, F^{ba}\}, \\
       q_2=(a,b;x_2):& \quad \rules_a = \{F^{bb}, F^{ba}\}, \ \ \rules_b = \{F^{aa}, F^{ab}\}.
   \end{align*}
   In words, at each query, the set $\rules_j$ consists of the rules for which $j$ is the more advantageous projection at that query's input, always with a margin of 1.
\end{proof}
To further illustrate the space of inseparable priorities, we now give our two additional examples: \textit{proportionality}, another distributional priority, and \textit{equal treatment}, an axiomatic priority. We show in Appendix \ref{app:prop-eq} that, like Egalitarianism, both of these priorities can be perfectly inseparable.

\begin{definition}[\textbf{Proportionality}] \label{ex:prop}
\upshape Proportionality describes the intuition that ``group $G_\ell$ should receive an $\alpha_\ell$-share of the overall benefit.'' We can formalize this as follows, for a generic distribution of inputs $\mathcal{D}$ over $\mathcal{X}$. Let $G_\ell \subseteq N$ for all $\ell \in [g]$ be the set of protected groups, with ideal fraction $\alpha_\ell \in [0,1]$. Then,
\[
u_{\mathrm{prop}}(F)=-\sum_{\ell \in [g]}\left(\E_{x \sim \mathcal{D}}[\mathbf{1}\{F(x) \text{ allocates
to a member of group } G_\ell\}]-\alpha_\ell\right)^2.
\]
\end{definition}

\begin{definition}[\textbf{Equal Treatment}]
\label{def:equal-treatment-priority}
The equal
treatment priority describes the intuition that counterparts should be treated equally. Formally, let \(\mathcal C\) be a collection of ``\textit{counterpart constraints}'' of the form $(x,x',\eta)$, where each
\((x,x',\eta)\in\mathcal C\) consists of two inputs \(x,x'\in\mathcal X\)
that differ only in some protected-group identity of otherwise comparable
recipients, together with a bijection $\eta:\mathcal Y(x)\to\mathcal Y(x')$
that describes what it means for the outputs at $x$ and $x'$ to correspond appropriately. Then,
\[
u_{\mathrm{eq}}(F)
=
\frac{1}{|\mathcal C|}
\sum_{(x,x',\eta)\in\mathcal C}
\mathbf 1\{F(x')=\eta(F(x))\}.
\]
\end{definition}

\subsection{Technical Preliminaries}
In the next two subsections, we consider the consequences of assuming separability when it does not hold. In order to formalize the consequences of this misspecification, we define some key preliminaries.
First, let 
\[\cR^\circ_{h;0,0} =\{ R^\circ_{h_\beta;0,0}(\cdot;M):\beta>0\}\]
be the class of baseline response models with link function $h$, varied over all possible values of $\beta$. The core definition we will use is \textit{scale-free indistinguishability} of two underlying priority models, which means that these two models' induced baseline response models are behaviorally indistinguishable up to the inverse temperature parameter $\beta > 0$.
\begin{definition}[Scale-Free Indistinguishability]
\label{def:scale-free-indistinguishability} Two priority models \(M,M'\in\mathcal M\)
are scale-free indistinguishable with respect to (w.r.t.) link function $h$ iff
\[
\cR^\circ_{h;0,0}(q;M)
=
\cR^\circ_{h;0,0}(q;M') \quad
\forall q\in\mathcal Q^{pc}.
\]
In other words, in $h_\beta$, $\beta$ is a nuisance parameter: for every $\beta > 0$, there is a $\beta' > 0$ such that $R^\circ_{h_\beta;0,0}(q;M)=R^\circ_{h_{\beta'};0,0}(q,M')$, and likewise for all $\beta' > 0$ there exists some such $\beta$.
\end{definition}

Scale-free indistinguishability is closely related to identifiability of the priority weights.
\begin{definition}[Scale-Free Identifiability]
\label{def:scale-free-weight-id}
Fixing $M = (u,\omega)$, weight \(\omega_j\) is
scale-free identifiable from local pairwise comparison queries w.r.t.~$h$ if
there do not exist two weight vectors \(\omega,\omega'\in\Delta^{m-1}\)
such that $\omega_j\neq \omega'_j$ and the models \(M=(u,\omega)\) and \(M'=(u,\omega')\) are scale-free
indistinguishable w.r.t.~$h$.

We say that \(\omega_j\) is scale-free non-identifiable over
\(S\subseteq[0,1]\) if, for every \(\gamma \in S\), there exists a
scale-free indistinguishable model \(M'=(u,\omega')\) with
\(\omega'_j=\gamma\).
\end{definition}

To consider what might be learned under the assumption of perfect separability, we define a \textit{perfect separable rationalization} of a generic model $M \in \cM$, which is a perfectly separable model $\widetilde{M} \in \Msep$ whose resulting response behavior is indistinguishable from that produced by $M$:
\begin{definition}[Perfect Separable Rationalization]
\label{def:perfect-sep-rationalization}
\(\widetilde M\in \Msep\) is a perfect separable rationalization
of \(M\in\cM\) w.r.t.~$h$  iff \(M\) and \(\widetilde M\) are scale-free indistinguishable w.r.t.~$h$.
\end{definition}
Let the set of all perfectly separable rationalizations of $M$ w.r.t.~$h$ be defined as follows. Note that this set could be empty, i.e., an inseparable model may not have a perfectly separable rationalization.
\[
\mathrm{PSR}_{h}(M)
=
\left\{
\widetilde M\in\Msep
:
M \text{ and } \widetilde M
\text{ are scale-free indistinguishable w.r.t.~}h\right\}.
\]

Finally, we formalize a class of decoders that are standard when behavior is assumed to arise from an S-RUM. 
A \textit{transcript} is a sequence $\cT=(q_t,w_t)_{t\ge1}$ of query, response
pairs, and a \textit{decoder} $D_{h,\cM'}$ maps a transcript to a model $\widehat M\in\cM'$. Here, the subscripts respectively reflect the decoder's assumptions that the baseline response model has link function $h$ and the true model lies in $\cM' \subseteq \cM$. Then,  $D_{h,\cM'}(\cT)$ is the model  in $\cM'$ the decoder returns on $\cT$.
We study decoders in the infinite-data limit, so that identifiability rather
than sampling noise is the binding constraint: call a transcript \textit{exhaustive}
if every query $q\in\cQ^{\mathrm{pc}}$ recurs infinitely often. Then, along an
exhaustive transcript generated by true response model $R^*(\cdot,M^*)$, the empirical response
frequencies converge to $R^*(\cdot\,;M^*)$ at every query. 

\begin{definition}[Separable-consistent decoder]
\label{def:separable-consistency}
A decoder $D_{h,\Msep}$ is \textit{separable-consistent} iff for every $M^*\in\cM$ with
$\mathrm{PSR}_{h}(M^*)\neq\emptyset$ and every exhaustive transcript $\cT$ generated
under $R^\circ_{h_\beta;0,0}(\cdot;M^*)$ for any $\beta > 0$,
\[
 D_{h,\Msep}(\cT)\in\mathrm{PSR}_h(M^*).
\]
\end{definition}
Intuitively, a separable-consistent decoder ``assumes perfect separability'' in the sense that,
when the responses \textit{could have} come from a perfectly separable model, it commits
to such an explanation. This is an extremely weak requirement, and it is automatically satisfied by
standard preference-learning estimators (e.g., maximum-likelihood estimation) that assume S-RUMs: any decoder that fits an S-RUM reports a fitted score function $\widehat V$, which is equivalent to a perfectly separable model $M_{\widehat{V}}$ by Theorem~\ref{thm:score-rum-equivalence}. When $\mathrm{PSR}(M^*)\neq\emptyset$,
the fit is exact on exhaustive data, so $M_{\widehat V}\in\mathrm{PSR}(M^*)$.

In the next two subsections, we illustrate the risks of using separable-consistent decoders. In particular, we will show that both perfectly inseparable (\Cref{sec:perfinsep}) and general inseparable (\Cref{sec:gen-insep}) priority models $M^*$ can admit perfectly separable rationalizations whose aggregate-optimal rules are highly suboptimal in the true model $M^*$. Our results will also clarify when there is hope for other kinds of decoders to avoid this issue.

\subsection{Perfectly Inseparable Priorities} \label{sec:perfinsep}

Now, we first show that when perfect separability is assumed erroneously, perfectly separable priorities are erased with no trace (\Cref{cor:erase}). Then, we show that perfectly inseparable priorities cannot be identified by local pairwise comparisons at all, meaning that no learner restricted to such queries can rectify this issue (\Cref{cor:non-ident}). Both of these conclusions follow from the following key theorem below. This result and its subsequent corollaries are not actually dependent on the use of the linear priority aggregator: they hold for a more general class of priority aggregators that are \textit{scale-preserving} (\Cref{def:scale-pres}), which just requires that adding priority $j^*$ does not change how the remaining priorities trade off against each other. We give the full proof in \Cref{app:scale-free-ind}.

Here, we let $M^{-J'}$ denote the model $M \in \cM$ with priorities in $J' \subset J$ removed, and with the priority weights rescaled so that $\omega^{-J'} = \frac{(\omega_j)_{j\in[m]\setminus\{J'\}}}
{\|(\omega_j)_{j\in[m]\setminus\{J'\}}\|_1}.$

\begin{theorem}
\label{thm:scale-free-ind}
Let \(M_{\mathrm{insep}}\in\cM\) be a priority model containing a nonzero number of perfectly inseparable priorities
\(J_{\mathrm{insep}}\subset[m]\).\footnote{One can also allow $J_{\mathrm{insep}}\subseteq[m]$, and handle the case of $J_{\mathrm{insep}} = [m]$ by defining the null model consisting of one priority in which $u_j(F) = 0 \ \forall F \in \rules$.} Let $M=(M_{\mathrm{insep}})^{-J_{\mathrm{insep}}}$. Then, \(M_{\mathrm{insep}}\) and \(M\) are scale-free indistinguishable w.r.t.~any link function \(h\).
\end{theorem}
\begin{proof}[Proof sketch]
The key to the proof is that no reasonable
rule aggregator can extract directional evidence
from a perfectly inseparable priority. Fix a perfectly inseparable priority
\(j\in J_{\mathrm{insep}}\) and a query \(q=(y,y';x)\). By perfect
inseparability, there is some \(\delta>0\) and an equal-size partition
\(\mathcal F_y,\mathcal F_{y'}\) of \(\mathcal F\) such that
\begin{align*}
    \left\{\Delta^{F}_{j}(y, y';\, x)\right\}_{F \in \mathcal{F}}
    &\;=\;
    \{\underbrace{\delta, \ldots, \delta}_{|\mathcal{F}_y|},\;
      \underbrace{-\delta, \ldots, -\delta}_{|\mathcal{F}_{y'}|}\}, \qquad \left\{\Delta^{F}_{j}(y', y;\, x)\right\}_{F \in \mathcal{F}}
    \;=\;
    \{\underbrace{-\delta, \ldots, -\delta}_{|\mathcal{F}_y|},\;
      \underbrace{\delta, \ldots, \delta}_{|\mathcal{F}_{y'}|}\}.
\end{align*}
Since \(|\mathcal F_y|=|\mathcal F_{y'}|\), these two multisets are identical
up to permutation. Therefore, by the permutation invariance of $\prules$,
\[
\phi^{\mathrm{rules}}
\left(
\left\{\Delta^{F}_{j}(y,y';x)\right\}_{F\in\mathcal F}
\right)
=
\phi^{\mathrm{rules}}
\left(
\left\{\Delta^{F}_{j}(y',y;x)\right\}_{F\in\mathcal F}
\right).
\]
Thus, the output of $\prules$ affects both directional evidence scores $s^+_M$ and $s^-_M$ symmetrically, so after applying a scale-preserving priority aggregator, its contribution cancels from $\kappa(q)$, up to a constant (over queries) rescaling of the remaining priorities. It follows that removing such priorities can only rescale the total
contribution of the remaining priorities but does not affect their relative importance, and thus can be absorbed by the inverse temperature parameter $\beta$. Hence,
\(M_{\mathrm{insep}}\) and \(M\) are scale-free indistinguishable.
\end{proof}

We now consider what a separable-consistent decoder will do in the presence of perfectly inseparable priorities. We assume here what is arguably the best-case scenario: that all priorities that are not perfectly inseparable are perfectly separable. We show that such decoders erase the perfectly inseparable priorities without a trace (\Cref{cor:erase}). 
\begin{corollary}[Erasure by Separable-Consistent Decoders]
\label{cor:erase}
Fix any link function $h$. Let \(M_{\mathrm{insep}}\) be as in \Cref{thm:scale-free-ind}, and suppose
that every priority in \([m]\setminus J_{\mathrm{insep}}\) is perfectly
separable. Again, let $M=(M_{\mathrm{insep}})^{-J_{\mathrm{insep}}}$. Then $M\in \mathrm{PSR}_h(M_{\mathrm{insep}}).$ Consequently, given an exhaustive transcript generated under \(M_{\mathrm{insep}}\), any separable-consistent decoder returns a perfectly separable model that exactly rationalizes the observed response behavior.
\end{corollary}

As shown by the following example, it is not hard to construct cases where the erasure of perfectly inseparable priorities leads to highly suboptimal inferred rules (formal details in \Cref{app:bad-rec}).
\begin{example}[Erasure can lead to high-regret rules] \label{ex:bad-rec}
Assume we are in the allocation task from \Cref{prop:egal}, with $N=\{a,b\}$, two inputs $x_1,x_2$ that occur with equal probability, and any rule must assign each input to either $a$ or $b$. Suppose the true priority model $M_{\mathrm{insep}}$ contains two priorities, {\upshape{Egalitarianism}} (\Cref{def:egal}) and {\upshape{Family}} (\Cref{ex:trolley-priorities}), where the individual is primarily egalitarian: for $\epsilon \in (0,1/2)$, let these priorities have respective weights $\omega_{\mathrm{egal}} = 1-\epsilon$ and $\omega_{\mathrm{family}} = \epsilon$. In this example, we know egalitarianism is perfectly inseparable; on the other hand, it is not hard to show that the family priority is perfectly separable. 

Then, by \Cref{cor:erase}, any separable-consistent decoder will effectively\footnote{Technically, this requires one extra lemma to reason about the aggregate-optimal rule when there are multiple perfect separable rationalizations. See \Cref{app:bad-rec}.} return the priority model containing only the \textit{family} priority, call it $M_{\mathrm{family}}$. The aggregate-optimal rule according to $M_{\mathrm{family}}$ is far from egalitarian, always prioritizing the individual's family members and giving zero benefit to anyone else. As $\epsilon \to 0$ and the individual truly becomes more egalitarian, the aggregate utility of the resulting rule approaches the worst possible rule according to the true model $M_{\mathrm{insep}}$: the regret of the chosen rule is $1-\nicefrac{3\epsilon}{2}$, where the regret of the worst possible rule is $1-\epsilon$.
\end{example}

One may wonder whether a more sophisticated decoder might be able to avoid the problem illustrated above. Unfortunately, the answer is \textit{no}: \Cref{thm:scale-free-ind} implies that the weights on perfectly separable priorities are not at all identifiable when the learner only has access to local pairwise comparison queries.
\begin{corollary}[Non-identification]
\label{cor:non-ident}
Let $j$ be any perfectly inseparable priority in any priority model $M$. Then,
$\omega_j$ is scale-free non-identifiable w.r.t.~any link function $h$ over
\([0,1)\) from local pairwise comparison queries. 
\end{corollary}

\subsection{General Inseparable Priorities} \label{sec:gen-insep}
Now we consider the more general class of inseparable priorities which, unlike perfectly inseparable priorities, can generate some directional evidence in response to queries. Unfortunately, because these priorities still fall outside of any model in $\Msep$, we now illustrate how this evidence is misinterpreted by separable-consistent decoders --- again, sometimes without a trace, and with major consequences for the resulting rule quality. We fully formalize this argument in \Cref{app:insep-misinterpret}.

\begin{example}[Misinterpretation by Separable-Consistent Decoders]
\label{ex:inseparable-local-trace}
Fix any link function $h$. We again assume we are in the allocation task from \Cref{prop:egal}, with
\(N=\{a,b\}\), two inputs \(x_1,x_2\) that occur with equal probability,
and any rule assigning each input to either \(a\) or \(b\). We impose an additional, very mild restriction on $\prules$ just for this example (see \Cref{app:insep-misinterpret}).

Let \(M\) contain a single priority, {\upshape{Proportionality}}, with target
shares \(\alpha_a=1/2+\epsilon\) and \(\alpha_b=1/2-\epsilon\), where
\(0<\epsilon<1/4\). That is, this priority wants to avoid shortchanging both \(a\) and \(b\), but with a bias toward $a$.

Now, examine the evidence produced across background rules on the query
\(q_1=(a,b;x_1)\). If the other input \(x_2\) is assigned to \(b\), then
assigning \(x_1\) to \(a\) moves the rule from \(F^{bb}\) to \(F^{ab}\),
which improves proportionality. If the other input \(x_2\) is assigned to
\(a\), then assigning \(x_1\) to \(b\) moves the rule from \(F^{aa}\) to
\(F^{ba}\), which also improves proportionality. However, because the target
share for \(a\) is slightly above one half, \(F^{bb}\) is worse than
\(F^{aa}\). 

Thus the improvement from \(F^{bb}\) to \(F^{ab}\) is larger
than the improvement from \(F^{aa}\) to \(F^{ba}\), and the benefit of choosing $a$ at $x_1$ outweighs that of choosing $b$, after aggregating over these background rules. Thus, the response model favors \(a\)
on the query \(q_1\). The same argument applies to the query \(q_2=(a,b;x_2)\). Consequently, the
response model favors assigning each input to \(a\), even though the
proportionality priority itself is maximized by the balanced rules
\(F^{ab}\) and \(F^{ba}\).

Thus \(M\) admits a perfect separable rationalization
\(\widetilde M\in\mathrm{PSR}_h(M)\) consisting of the single priority
\[
u_a(F)
=
\mathbb E_{X\sim\mathcal D}
\left[
\mathbf 1\{F(X)=a\}
\right].
\]
A separable-consistent decoder can therefore fit the exhaustive transcript
exactly with \(\widetilde M\). According to \(\widetilde M\), the resulting aggregate-optimal rule is $F^{aa}$. However, this rule is highly disproportional, and accordingly, as $\epsilon \to 0$, its aggregate utility approaches that of the worst possible rule under the true model $M$, i.e.,
\[
U_M(F^{ab})=U_M(F^{ba})\to 0,
\qquad
U_M(F^{aa})\to -\nicefrac{1}{2},
\qquad
U_M(F^{bb})\to -\nicefrac{1}{2}.
\]
\end{example}
The same phenomenon can persist when the true model contains additional priorities. An inseparable priority can mask, distort, or even reverse the local signal generated by perfectly separable priorities. In some cases, this means that no perfect separable rationalization exists at all.\footnote{For example, inseparable priorities can generate score differences whose magnitudes are inconsistent with any separable model, or can even induce cyclic pairwise comparison patterns.} Then the assumption of perfect separability leads to irreducible population misfit: even with unlimited data from every local query, the best separable model cannot perfectly explain the response distribution. This offers one possible source of observed inconsistency, i.e., the failure of any preference model within the assumed class to fit people’s responses.

Finally, we return to whether a different decoder could avoid this problem. Unlike the perfect-inseparability case, where the obstacle was non-identifiability, the obstacle here is \textit{misinterpretation}: general inseparable priorities may be identifiable from local pairwise comparisons, but a separable-consistent decoder interprets their signal through the wrong structural lens. This points to a possible alternative: \textit{priority-aware learning}, where the decoder is given the individual’s priority utility functions up front and interprets query responses while accounting for both separable and inseparable priority structures. We expand on this proposal in \Cref{sec:discussion}, but formalizing this idea constitutes rich future work.
\section{Generalization \#2: Latent Indecision} \label{sec:conflict}
In the previous section, we considered the generalization from $\Msep$ to $\cM$, holding $\tau_r,\tau_\kappa = 0$. Now, we consider the opposite generalization: restricting to $\Msep$ but allowing $\tau_r,\tau_\kappa > 0$ and thus permitting latent indecision. This generalization allows us to go beyond the mechanical justification of forced comparisons under zero-thresholds, where only decisive latent states $\succ^*$ and $\prec^*$ can arise.

The practical relevance of this generalization is demonstrated by mounting evidence that individuals struggle to answer forced comparisons, acting inconsistently, hesitating to choose, expressing difficulty or anguish, or expressing low confidence in their answer. 

For example, in interviews conducted by \citet{keswani2025can}, a participant was asked to decide whether to give a kidney to a recipient who would live 10 more years with one dependent, versus a recipient who would live 20 more years with no dependents. In their response, they said:
\begin{quote}
\textit{``...If this person has a child or something, 10 years is going to be the difference between leaving a child and leaving a young adult [\dots] I wouldn't sleep well after making this decision.''}
\end{quote}
In the language of our model, this quote seems to express that the decision is of high moral valence with no clear decision, in line with our notion of conflict. This is not an isolated incident in this dataset: we hand- and LLM-coded the 20 interview transcripts underlying the study, each of which contains three narrated pairwise comparisons. We coded for language markers of indecision, including back-and-forth reasoning, explicit statements of difficulty, hedging, self-correction, and discomfort with choosing, details are in Appendix \ref{app:criteria}. We find these indecision markers to be common: 18 of 20 participants displayed at least one such marker, and 8 of 20 displayed them on at least two of the three queries.

This interview data --- along with other studies documenting participants voluntarily reporting indecision and conflict \cite{mcelfresh2021indecision,rosas2019decision}\footnote{In pairwise kidney-allocation experiments, participants frequently used an explicit indecision option when it was available \cite{mcelfresh2021indecision}, and in sacrificial moral-dilemma experiments, many participants directly reported conflict while making each judgment \cite{rosas2019decision}.} --- supports the possibility that people (a) experience indecision when answering pairwise comparisons, and (b) can report it. Our model allows us to investigate the technical consequences of these possibilities by giving a formal notion of ``true'' indecision: the latent states $\sim^*$ and $\bowtie^*$, which arise exactly in the generalization of S-RUMs where $\tau_r,\tau_\kappa > 0$. We use this generalization to ask two questions. First, in \Cref{sec:accuracy}, we ask: To what extent can forced decisive responses compromise learning accuracy when individuals \textit{distort their behavior} on queries where they are latently indecisive? Second, in \Cref{sec:speed}, we ask: Can allowing individuals to report conflict and indifference, or even just general indecision, not only avoid this distortion but improve learning speed by providing additional information about their underlying beliefs?

Both of these questions examine how indecision can potentially compromise or aid the technical task of learning $\omega^*$. Before investigating these reasons for considering indecision, we formalize another intuition about why indecision at the query level is useful: that indecision at the level of queries can tell us something about indecision at the level of \textit{rules}.

\begin{remark}[Toward a richer rule-level preference relation]
\label{rem:rule-level-indecision} 
So far, we have explicitly avoided defining a global preference relation over rules based on $M$. However, our model of latent states at the query level can easily be defined analogously at the rule level: for two rules \(F,F'\in\rules\), define the rule-level
directional evidence scores as
\[
s^+_M(F,F')
=
\sum_{j\in[m]} \omega_j [u_j(F)-u_j(F')]_+,
\qquad
s^-_M(F,F')
=
\sum_{j\in[m]} \omega_j [u_j(F')-u_j(F)]_+.
\]
Then, one can define $r_M(F,F')$ and $\kappa_M(F,F')$ and apply the thresholds $\tau_r,\tau_\kappa$ analogously, leading to latent rule level relations of $\succ,\prec,\sim,$ and $\bowtie$ between rules. 
The key observation is that when $M$ is perfectly separable, indecision between queries is formally exchangeable with indecision between \textit{entire rules}, as defined by the relation above. Formally, for any local query $q=(y,y';x)$ and background rule $F$, the query-level latent state $L_M(q)$ corresponds exactly to the latent rule-level relation between the two projected rules $F_{x\to y}$ and $F_{x\to y'}$. We prove this equivalence in \Cref{app:rule-level-indecision}.

This relationship formalizes the importance of a natural intuition: that forced comparisons \textit{misrepresent the authority of the response}, treating a choice made under conflict or indifference as if it expressed the same endorsement as a genuinely decisive judgment. By the above definition, this misrepresentation carries up to the rule level, where judgments are consequential to the decision. On the other hand, maintaining this information is arguably normatively desirable, producing a more faithful representation that captures where the individual is indifferent or conflicted between rules. This richness can also be practically useful when trying to find consensus between people and being transparent about what trade-offs are being made: indifference identifies regions of flexibility in the decision-rule space, and conflict identifies regions of moral difficulty.
\end{remark}

Now, we turn our attention to the two questions above. In our simulations, we will specifically analyze \textit{linear} priority models, in which each priority advocates for the importance of a single feature in a latent feature space; we will show these are perfectly separable, as needed. While this is a substantial restriction of our model, it is of independent interest: it captures the popular approach of assuming that S-RUM scores have linear structure over a fixed feature space, as in the emerging paradigm of linear social choice \cite{ge2024axioms,ge_linear_2026} and many other papers in preference learning \cite{lee2019webuildai,freedman2020adapting,ge2024learning,noothigattu2018voting,boerstler2024stability}. By allowing $\tau_r,\tau_\kappa > 0$, we obtain a strict generalization of these linear settings in which the individual can experience conflict when salient features give evidence for different responses, or indifference when no salient feature is relevant to a given comparison. This setting also gives a simple first test case for learning within our model. We introduce this model, along with additional technical preliminaries, in \Cref{sec:tech}.

\subsection{Technical Setup} \label{sec:tech}

\subsubsection{Linear Priority Models}
We restrict to $M$ whose priorities are linear in a known feature space, defined by feature map
\[
\psi : \{(x,y): x\in\mathcal X,\ y\in\mathcal Y(x)\}\to [0,1]^d.
\]
For example, in the kidney allocation task, $\psi(x,y)$ could describe the features of the chosen recipient $y$, e.g., their age, gender, and whether they have dependents. It could also encompass features that span $x$ and $y$, such as a feature capturing the match quality of kidney $x$ with recipient $y$.

Accordingly, the priority model $M$ consists of $d$ priorities, where each priority $j \in [d]$ advocates for the importance of a single feature. This produces utility functions as follows, where $\psi_j(x,y)$ is the $j$-th feature in the feature vector:
\begin{definition}[Linear Priority Model]
\label{def:linear-perfectly-separable-model}
Fix a feature map \(\psi\) with dimension $d$. A linear priority model $M_\omega = (u,\omega)$ is any model with $d$ priorities and utilities
\[
u_j(F)
:=
\frac{1}{|\cX|}
\sum_{x \in \cX}\psi_j(x,F(x)) \qquad \forall j \in [d].
\]
\end{definition}
Because the utility functions are fixed by $\psi$, given a $\psi$ the only varying element of a priority model is $\omega$; accordingly, let $M_\omega$ be the priority model with vector $\omega$ under (implicit) feature map $\psi$. Let $\mathcal M_{\psi}
:=
\{M_\omega:\omega\in \Omega\}$ be the class of all linear priority models over feature map $\psi$.
Every \(M_\omega\in\mathcal M_{\psi}\) is perfectly separable: for any query \(q=(y,y';x)\) and priority \(j\), the projected rules $F_{x\to y}$ and $F_{x \to y'}$ differ only at $x$, so all other feature terms cancel and $\Delta^F_j(q)$ is independent of $F$:
\[
\Delta^F_j(q)
=
u_j(F_{x\to y})-u_j(F_{x\to y'})
=\frac{1}{|\cX|}\big(\psi_j(x,y)-\psi_j(x,y')\big) \qquad \forall F \in \rules.
\]

Now, let the \textit{linear rule} $F_\omega \in \rules$ be the rule that at any given $x$ outputs the $y$ with the highest linear-weighted score:
\[F_\omega(x) = \arg\max_{y \in \cY(x)} \langle \omega,\psi(x,y)\rangle.\]
We now show that, given any linear priority model $M_{\omega^*}$, its aggregate-optimal rule is the linear rule $F_{\omega^*}$, i.e., the linear rule defined by $\omega^*$:
\begin{proposition}
\label{prop:linear-rule-is-aggregate-optimal}
For every feature map $\psi$ and every linear priority model $M_{\omega^*} \in \cM_\psi$, 
\[F_{\omega^\ast}
\in
\arg\max_{F\in\mathcal{F}} U_{M_{\omega^\ast}}(F).\]
\end{proposition}
We prove this in \Cref{app:linear-rule-is-aggregate-optimal}; the intuition is that the linear rule maximizes the weighted score pointwise at each $x$, and because aggregate utility decomposes additively over inputs under perfect separability, a pointwise-optimal rule is globally optimal.

Now, we formalize the claim that in the special case where $\tau_r=\tau_\kappa=0$, linear priority models exactly capture the linear models used in the literature, as cited above. We call these models \textit{Linear S-RUMs}, because they are S-RUMs whose local score function is structured according to a linear weight vector $\theta$. We defer the proof to \Cref{app:linear-srum-relationship}.

\begin{theorem}[Capture of Linear S-RUMs]
\label{thm:linear-srum-relationship}
Fix any link function $h$. Let $\theta\in\R^d_{\ge 0}$ satisfy $\|\theta\|_1>0$, and define the linear S-RUM score function
\[
V_\theta(x,y)=\langle \theta,\psi(x,y)\rangle.
\]
Let $\omega=\theta/\|\theta\|_1$, and let $M_\omega$ be the linear priority model from \Cref{def:linear-perfectly-separable-model}. Then $V_\theta$ under the S-RUM response model is scale-free indistinguishable from $M_\omega$ under the zero threshold response model: that is, for every $\beta>0$, there exists $\beta' =\beta\, |\cX|\, \|\theta\|_1$ such that
\[
S_{h_\beta}(\cdot;V_\theta)
\equiv
R^\circ_{h_{\beta'};0,0}(\cdot;M_\omega).
\]
Moreover, they induce the same learning target: for every $\rules'\subseteq\rules$,
\[
\arg\max_{F\in\rules'} U_{M_\omega}(F)
=
\arg\max_{F\in\rules'}
\sum_{x \in \cX}
V_\theta(x,F(x)).
\]
\end{theorem}

\subsubsection{Simulation and Learning Setup}

\textit{Decision Task.} In our simulation experiments, each decision instance consists of five candidate recipients.
Each candidate is represented by \(d=5\) normalized features: we write an input as
\(x=(z^{(1)},\ldots,z^{(5)})\in([0,1]^d)^5\), where \(z^{(\ell)}\in[0,1]^d\)
is the feature vector of candidate \(\ell\).\footnote{Technically, inputs are being drawn from a continuous region
\([0,1]^d\), while the sums above treat the input space as discrete. Because each simulation uses finite samples, we are effectively in a finite and discrete $\cX$ regime; one could simply discretize the sampling space to an arbitrarily fine degree to make $\cX$ formally discrete, if desired.} On input \(x\), the rule must select one of
the five candidates, so the feasible output set is \(\mathcal{Y}(x)=\{1,\ldots,5\}\).
The feature map returns the features of the selected candidate, so
\(\psi(x,\ell)=z^{(\ell)}\). Thus, a local pairwise query \(q=(\ell,\ell';x)\)
asks whether, in decision instance \(x\), the rule should select candidate \(\ell\)
or candidate \(\ell'\). We assume that inputs are drawn uniformly at random; this is a natural and non-degenerate choice, giving an unbiased sample of the array of trade-offs that can occur.

\textit{True priority weights.} The underlying linear priority model must then have $m = 5$ priorities, one per feature, and the priority utilities $u_j^*$ are thus fixed. We repeat all tests for 40 ground-truth choices of $\omega^*$, where the true weights are drawn randomly as $\omega^*\sim \text{Dirichlet}(0.2)$. We report means and standard errors over random choices of $\omega^*.$ As is standard in linear S-RUMs, we assume the feature map, and hence the priority utilities, are known to the learner, so the learner's task is then to recover these priorities' true importance $\omega^*.$ 

\paragraph{Bayesian active learning.} 
For each ground-truth weight vector \(\omega^*\), we learn \(\omega^*\) by Bayesian active learning. We use active learning to be fair to all response conditions; if we in contrast fixed a query sequence across all algorithms, it could happen to ask queries that are far more useful for one response condition over another, confounding our ability to compare the actual information value of difference response conditions. With active learning, the query algorithm explicitly seeks the most useful queries given the assumed response condition. At a high level, our active learning methods work as follows.

At round \(t\), the learner maintains a posterior \(\pi_t\) over plausible weights \(\omega\), selects a query \(q_t\), observes a response drawn from the true response model \(R^*\), and updates its posterior using its assumed response model \(R\). Upon reaching a stopping condition,\footnote{In \Cref{sec:speed}, we run all algorithms for a fixed $T = 100$ rounds to compare learning speed. In \Cref{sec:accuracy}, we run to convergence: letting $\widehat{\omega}_t$ be the posterior mean at $t$, the convergence condition is that $|\widehat{\omega}_{t+1} -\widehat \omega_{t}|<0.01$ for 5 consecutive iterations.} the algorithm outputs the posterior mean \(\widehat{\omega}_T=\mathbb{E}_{\omega\sim\pi_T}[\omega]\). This general algorithmic template is formally specified in Algorithm \ref{alg:bald-template}, with its variants formally specified in the appendix (Appendix \ref{app:algo-details}). Queries are chosen using Bayesian Active Learning by Disagreement (BALD), which scores a query by how much its possible answers would reduce uncertainty about the true value of \(\omega\). 

The ideal implementation of this algorithmic method would calculate all objects exactly. However, as is standard, for tractability we use sampling to approximate the objects above. First, we would ideally compute the BALD score for every possible pairwise query in \(\mathcal Q^{pc}\) and select the query with the largest score. We approximate this by computing the BALD score for each of (C=50) randomly-sampled queries from \(\mathcal Q^{pc}\), and then asking the highest-scoring query in this sample. Each score estimate is computed using an estimate of the posterior, based on \(N_{\mathrm{BALD}}=50\) samples of \(\omega\) drawn from the learner's maintained posterior sample set of size \(N_{\mathrm{post}}=200\). \Cref{rem:consistency} gives the standard consistency statement: as the candidate-pool size and number of posterior samples grow, this approximation returns a near-optimal BALD query with high probability. 

\textit{Response models.}
Across the simulations, we distinguish between the true response model \(R^*\), which generates the individual's observed answers, and the learner's assumed response model \(R\), which is used for posterior updates. We vary these models across experiments starting from the baseline family \(R^{\circ}_{h_\beta;\tau_r,\tau_\kappa}\); we specify the details in the respective sections. Throughout we use the logistic link \(h_\beta(t)=(1+\exp(-\beta t))^{-1}\) with \(\beta=10\). Thus, when \(\tau_r=\tau_\kappa=0\), the baseline model reduces to the standard Bradley-Terry model. Unless otherwise stated, the thresholds \(\tau_r,\tau_\kappa\) are known to the learner; in one variant, we instead learn \(\tau_r,\tau_\kappa\) jointly with \(\omega^*\), using an extension of our active learning procedure described in \Cref{app:learn-tau}. Because the priority models in this section are perfectly separable, the choice
of rule aggregator \(\prules\) does not affect the directional evidence scores, provided \(\prules\) is unanimous. To avoid arbitrary scaling due to the discretization of $\cX$, in a minor departure from our model we drop the $1/|\cX|$ from the utility gap, so 
for \(q=(y,y';x)\), we let
\(\prules((\Delta^F_j(q))_{F\in\mathcal F})=\psi_j(x,y)-\psi_j(x,y')\).

\subsubsection{Evaluation Metrics}
We evaluate learning at two levels: whether the learner recovers the individual's
priority weights, and whether the learned weights induce low-regret decisions. For
weight recovery, we report the $\ell_1$ loss $ \|\widehat{\omega}-\omega^*\|_1.$
Such errors are an issue if we are interpreting $\omega^*$ as substantively meaningful importance weights on features, or if we are measuring the welfare of the resulting weights according to their cardinal distance from the true weights. This distance also has significance for the regret: by \Cref{lem:weight-recovery-rule-recovery}, if
\(\|\widehat{\omega}-\omega^*\|_1\leq \epsilon\), then the rule induced by
\(\widehat{\omega}\) has aggregate-utility regret at most \(2\epsilon\) under the true
model.

To evaluate the quality of the actual choices made by the learned rule $F_{\widehat{\omega}}$, we use the regret (\Cref{def:regret}). 
For a learned vector \(\widehat{\omega}\), define the regret at input \(x\) by 
\[
\Reg_x(\widehat{\omega};\omega^*)
:=
\langle \omega^*,\psi(x,F_{\omega^*}(x))\rangle
-
\langle \omega^*,\psi(x,F_{\widehat{\omega}}(x))\rangle ,
\]
and define the utility range at \(x\) as the utility gap between the best and worst choice at $x$:
\[
\Range_x(\omega^*)
:=
\langle \omega^*,\psi(x,F_{\omega^*}(x))\rangle
-
\min_{y\in\cY(x)}\langle \omega^*,\psi(x,y)\rangle .
\]
We then define the normalized \textit{average regret} as the expected regret normalized by the range, where this normalization accounts for the fact that certain settings will produce lower-stakes choices. The average regret is then interpretable as the fraction of the typical available utility range lost by the learned rule:\footnote{Note that the numerator of $\operatorname{Avg-Regret}$ is equivalent to the regret in \Cref{def:regret}, reduced to the linear case, where the aggregate-optimal rule under model $M_{\omega^*}$ is the linear rule $F_{\omega^*}$ (\Cref{prop:linear-rule-is-aggregate-optimal}).}
\[
\operatorname{Avg-Regret}(\widehat{\omega};\omega^*)
:=
\frac{\nicefrac{1}{|\cX|} \sum_{x \in \cX}\left[\Reg_x(\widehat{\omega};\omega^*)\right]
}{
    \nicefrac{1}{|\cX|} \sum_{x \in \cX}\left[\Range_x(\omega^*)\right]
}.
\]
We define the \textit{worst-case} regret analogously, the fraction of the maximum utility range lost on the input where the learned rule makes its lossiest mistake:
\[\operatorname{WC-Regret}(\widehat{\omega};\omega^*)
:=
\frac{
    \sup_{x\in\cX}\Reg_x(\widehat{\omega};\omega^*)
}{
    \sup_{x\in\cX}\Range_x(\omega^*)
}.\]
For any fixed $\omega^*$ and learned $\widehat{\omega}$, we estimate the average regret using a fixed set of $300$ i.i.d.\ uniformly sampled inputs, shared
across all response conditions and random choices of $\omega^*$. We compute the worst-case regret exactly via linear programming, as described in
\Cref{app:regret-computation}.

\subsection{What if the individual's response behavior changes when they are indecisive, but forced to decide?} \label{sec:accuracy}
To formally test this question, we let the learner assume the standard response model $R_{h_\beta;0,0}^{\circ}$, as this is the assumption under which forced comparisons are justified. In contrast, we allow the true response model $R^*$ to vary such that when the individual is truly conflicted or indifferent (i.e., when $L_{M_{\omega^*}(q)} \in \{\sim^*,\bowtie^*\}$), they may deviate from the assumed behavior in order to coerce their response into the permitted response alphabet $\{\succ,\prec\}$. To be maximally friendly to the learner, we assume that outside indecisive latent states (i.e., when $L_{M_{\omega^*}(q)} \in \{\succ^*,\prec^*\}$), the individual responds according to $R_{h_\beta;0,0}^{\circ}$. We describe the four behavioral conditions we test here, and formally define them in Appendix \ref{app:formal-indecision}.
\begin{enumerate}
    \item \textsc{Correct}: The individual resolves latent indecision according to the baseline model --- i.e., regardless of $L_{M_{\omega^*}}(q)$, the individual responds according to $R^\circ_{h_\beta;0,0}$. This is the idealized benchmark. 
    \item \textsc{50/50}: When $L_{M_{\omega^*}}(q) \in \{\sim^*,\bowtie^*\}$, the individual decides between $y$ and $y'$ by flipping an unbiased coin.
    \item \textsc{Lexicographic}: When $L_{M_{\omega^*}}(q) \in \{\sim^*,\bowtie^*\}$, the individual 
    decides between $y$ and $y'$ lexicographically: they choose deterministically based on the highest-weight priority (per $\omega^*$) that, in isolation, produces a decisive response between the two options.
    \item \textsc{Self-similarity}: 
    Here, we suppose the individual themselves has features that belong in $\psi$ (this would make sense, e.g., in any allocation problem where goods/bads are being allocated to people). In each run, the individual's feature vector $v$ is drawn uniformly randomly from $[0,1]^5$. When $L_{M_{\omega^*}}(q) \in \{\sim^*,\bowtie^*\}$, the individual deterministically chooses the alternative that is most similar to them in feature space $\psi$, reflecting self-similarity bias.
\end{enumerate}
In the following results, we set $\tau_r=\tau_\kappa=0.25$ so that latent indecision occurs at nontrivial
frequency but does not overwhelm directional signal. 

\begin{figure}[!]
    \centering
    \includegraphics[width=\linewidth]{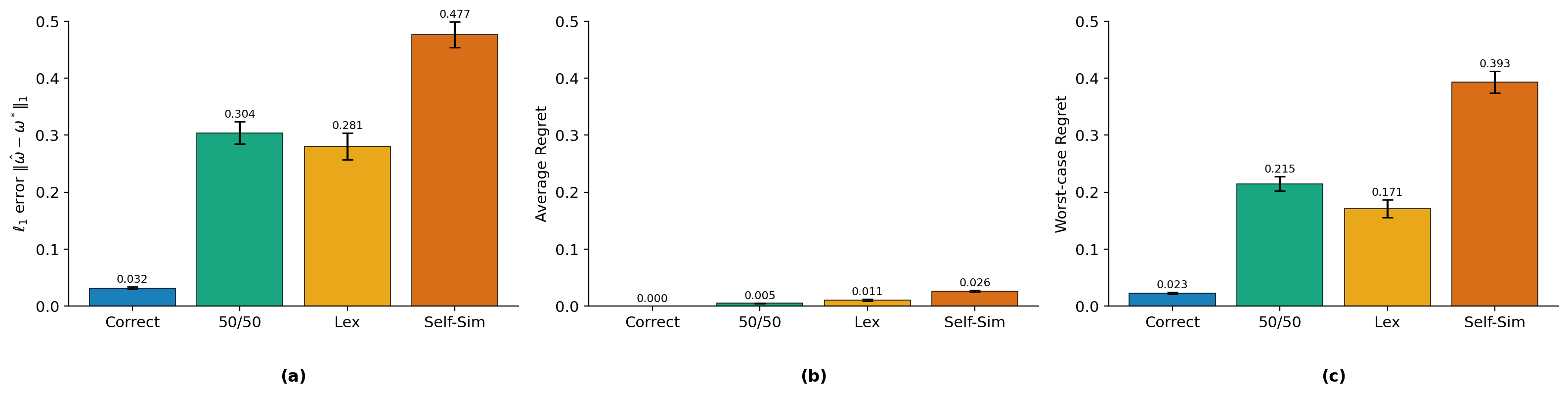}
    \caption{Learner performance under the four models of resolving indecision. Error bars are $\pm 1$ standard errors, reflecting randomness over 40 random choices of $\omega^*$. \textbf{(a)} $\ell_1$ error of $\widehat\omega$. \textbf{(b)} Average regret, estimated via 300 sampled $x\in \cX$. \textbf{(c)} Worst-case regret.}
    \label{fig:threepanel}
\end{figure}

\textbf{$\ell_1$ errors.} In \Cref{fig:threepanel}(a), we examine the $\ell_1$ error in the estimated $\widehat{\omega}$ relative to $\omega^*$. Noting that the maximum possible value of this $\ell_1$ error is 2 (by the triangle inequality), the $\ell_1$ error across the three deviating response models is nontrivial, reaching between 14\%-24\% of the worst case. As expected, the idealized baseline \textsc{Correct} has negligible error. 

Under each behavioral deviation, the driver of the $\ell_1$ error in the learned $\widehat{\omega}$ reflects the corresponding structure of the forced-response rule. We illustrate these findings in Figure \ref{fig:biasdecomp}, Appendix \ref{app:moreacc}. Under \textsc{Lexicographic} (\Cref{fig:biasdecomp}(a)), $\widehat{\omega}$ is overly concentrated on the highest-weight priority: across runs it assigns an extra $0.13\pm0.01$ mass to the true top feature relative to $\omega^*$, exceeding the true top-feature mass in 38 of 40 runs. Under \textsc{Self-Similarity} (\Cref{fig:biasdecomp}(b)), the distortion is instead directed toward the individual's own feature vector $v$: the projection $\langle \widehat{\omega}-\omega^*,v\rangle$ is positive in 35 of 40 runs, with a cross-run mean of $0.10\pm0.01$. In contrast, these sub-figures together show that \textsc{50/50} distorts $\widehat{\omega}$ in neither
direction, because its forced labels are an unbiased coin flip.

\textbf{Average regret.}
If one simply cares about making the optimal choice over \(\cY(x)\) at each \(x\),
the story is more positive, at least in the average case. \Cref{fig:threepanel}(b), displaying $\operatorname{Avg-Regret}(\widehat{\omega},\omega^*)$, shows that over random $x$, the expected choice loss from forced-response behavior is small across
deviations. Normalized average regret is highest for \textsc{Self-Similarity},
which loses \(2.6\%\) of the typical utility range available across decision
instances; \textsc{50/50} and \textsc{Lexicographic} lose \(0.5\%\) and
\(1.1\%\), respectively. Thus, on average, the choices induced by the learned weights nearly optimally recover utility, even when the weights themselves are
substantially misestimated.

These regret levels are largely explained by whether the learned vector
\(\widehat{\omega}\) points in the right direction. \textsc{Self-Similarity}
systematically biases \(\widehat{\omega}\) toward the individual's own feature
vector \(v\), which can point in a substantially different direction from
\(\omega^*\), producing the largest average regret. By contrast, under
\textsc{Lexicographic}, behavior remains anchored to the direction of
\(\omega^*\), even though the largest priorities are overweighted; under
\textsc{50/50}, deviations around \(\omega^*\) arise from noise
rather than from a systematic directional bias. We
illustrate this by showing the cosine similarity $\cos(\widehat{\omega},\omega^*)$ in Figure \ref{fig:cosine} (Appendix \ref{app:moreacc}). The fact that
cosine similarity closely tracks average regret is consistent with prior work
on linear decision rules showing that angular alignment is theoretically
related to good downstream choice behavior
\cite{feffer2023moral,baharav2026end}.

\textbf{Worst-case regret.}
The picture is less reassuring in the worst case. \Cref{fig:threepanel}(c) shows worst-case regret \(\operatorname{WC-Regret}(\widehat{\omega},\omega^*)\), i.e.,
largest choice loss the learned rule can incur on any input. On this metric,
\textsc{50/50}, \textsc{Lexicographic}, and \textsc{Self-Similarity} respectively reach
\(21\%\), \(17\%\), and \(39\%\) of the maximum utility range available over choices, on average across runs. 

The changed ordering in relative performance of \textsc{Lexicographic} and \textsc{50/50} from average to worst-case regret is not spurious --- it occurs because \textsc{Lexicographic} has less directional distortion, overweighting
priorities that are already highly weighted under \(\omega^*\). As a result,
when it makes mistakes, it often does so by leaning too heavily on features
that \(\omega^*\) also regards as important. By contrast, the noisy behavior produced under \textsc{50/50} produces a $\widehat \omega$ that is less directionally tied to $\omega^*$, and this can be exploited by a worst-case adversary.

\subsection{If the learner allows and \textit{utilizes} self-reported indecision, can they learn faster?} \label{sec:speed}
One natural solution to the problem described above is to allow the individual to report their indecision. Using a learner that can \textit{utilize} this indecision, we now investigate the extent to which this also allows faster learning. To test this question, we compare how quickly four different learning approaches converge to the true $\omega^*$. In all cases, the learner assumes the correct underlying response model; what varies is the extent to which they receive or can utilize indecision information. 

The first two response conditions below represent forced comparison benchmarks. The latter two represent cases where the learner actively solicits reports of conflict and indecision (\textsc{Utilize-4}) or generic indecision (\textsc{Utilize-3}). Here, the individual can report a richer alphabet of responses.

\begin{enumerate}
     \item \textsc{Correct}: \ The true response model is \(R^\circ_{h_\beta;0,0}\), so the individual only reports responses in $\{\succ,\prec\}$. The learner assumes the same response model. This is standard Bradley-Terry, and serves as the baseline testing how fast we can learn from forced comparisons under perfect conditions.

    \item \textsc{Ignore}: \ The true response model is \(R^\circ_{h_\beta;\tau_r,\tau_\kappa}\), the learner assumes \(R^\circ_{h_\beta;0,0}\), and any query response in $\{\sim,\bowtie\}$ is dropped from the transcript. This reflects a learner who forces decisive comparisons by dropping queries in which the individual was unable to respond due to indecision.

    \item \textsc{Utilize-3}: \ Here, the true response model is a version of 
    \(R^\circ_{h_\beta;\tau_r,\tau_\kappa}\) where indecisive responses $\{\sim,\bowtie\}$ are both reported as generic indecision \(\oslash\) (see \ref{def:generic-indecision} for the formal specification). The learner assumes the same. This case lets us examine the relative benefit
    of distinguishing \(\sim\) and \(\bowtie\) versus letting the
    individual report general indecision.
    
    \item \textsc{Utilize-4}: \ Here, the true response model is 
    \(R^\circ_{h_\beta;\tau_r,\tau_\kappa}\) and the learner assumes the same. This represents
    the case where the individual responds
    in a way that reflects (a potentially noisy version of) their latent state, and the learner can utilize all possible responses.

\end{enumerate}

We also include a version of (4) where we learn $\tau_r,\tau_\kappa$ alongside $\omega$, to demonstrate the possibility of doing this efficiently; we denote this variant with a $\diamond$ and detail the methodology in Algorithm \ref{alg:learn-tau}, Appendix \ref{app:learn-tau}. 

We now compare the convergence rates of all five learning setups (\Cref{fig:bald_6}). To give underlying intuition about what types of queries are useful, we additionally show which query types were sought by the active learning algorithm in each setting (Figure \ref{fig:respdist-main}). We discuss our key findings below.

\begin{figure}[!]
    \centering
    \includegraphics[width=\linewidth]{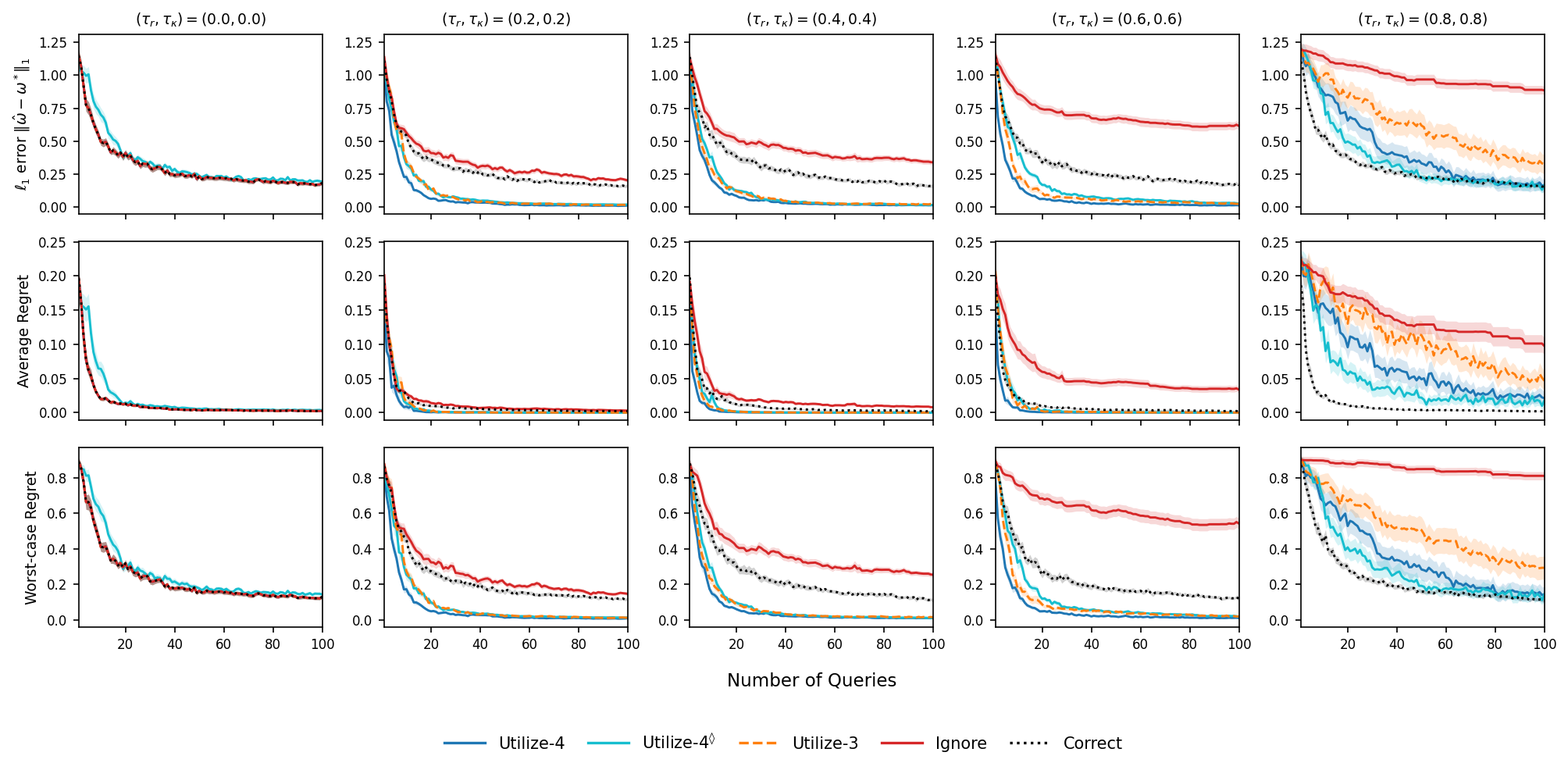}
    \caption{Performance of five methods versus query number across the diagonal threshold
    regimes.
    \textbf{Top:} $\ell_1$ error $\|\widehat\omega-\omega^\ast\|_1$. \textbf{Middle:} Average regret on the
    uniform-$[0,1]^5$ distribution, all features independent. \textbf{Bottom:} Worst-case
    single-decision regret. See performance for more $\tau$ regimes and extensions in Appendix \ref{app:big-grid}.}
    \label{fig:bald_6}
\end{figure}

\begin{figure}[!]
    \centering
    \includegraphics[width=0.9\linewidth]{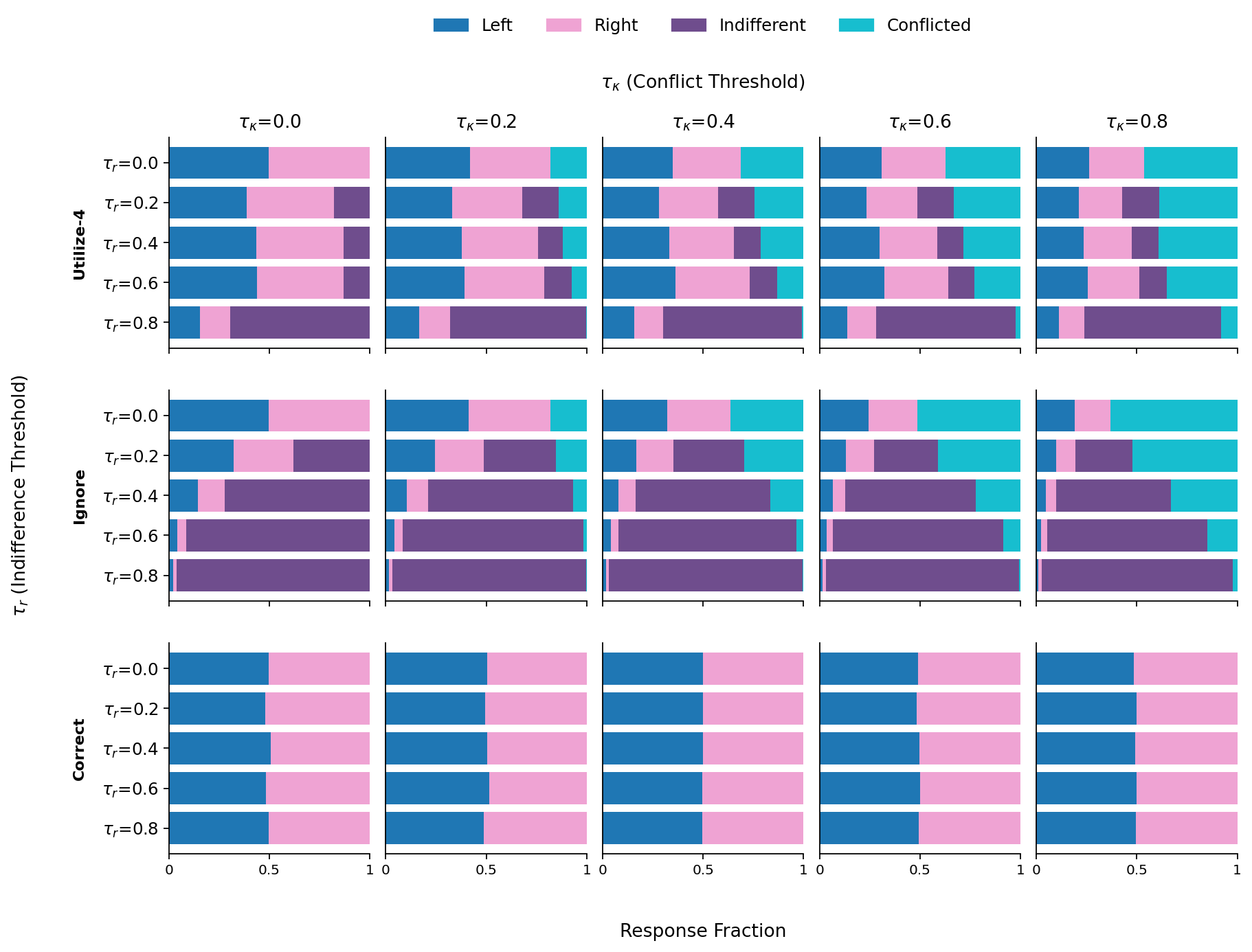}
    \caption{Fraction of each response type (Left $\succ$, Right $\prec$,
    Indifferent $\sim$, Conflict $\bowtie$) for three representative methods ---
    \textsc{Utilize-4}, \textsc{Ignore}, and \textsc{Correct} --- across the full
    $5\times5$ threshold grid. }
    \label{fig:respdist-main}
\end{figure}

\textbf{Utilizing indecision allows faster learning.} First, comparing \textsc{Correct} and \textsc{Ignore} to the \textsc{Utilize} conditions in \Cref{fig:bald_6}, we see that all three indecision-aware learners converge much faster to low \(\ell_1\) error and worst-case regret. As the indecision region grows, the gap also widens in the average regret.  Only in the most severe indecision regime does \textsc{Correct} outperform \textsc{Utilize} methods, where it is simply benefiting from its independence of the latent state in a regime where the latent states are saturated with indecision. However, this is perhaps the least behaviorally plausible regime for the forced benchmark: it assumes that at maximal true indecision, individuals can still resolve all indecision via unbiased, well-structured noise. Making finer comparisons, we see that asking people to distinguish between indifference and conflict makes almost no difference, with \textsc{Utilize-3} and \textsc{Utilize-4} performing almost identically relative to the forced-comparison gap. 

\textsc{Utilize} methods outperforming \textsc{Ignore} is partly mechanical: \textsc{Ignore} is dropping data, and it is doing so in a non-random way, meaning it may never converge to the correct weight vector.\footnote{When \(\tau_r>0\), dropping \(\sim\) responses removes low-valence queries and can eliminate identifying variation. When \(\tau_\kappa>0\), dropping \(\bowtie\) responses conditions on being outside the conflict band, so the remaining \(\succ,\prec\) responses follow a conditional response distribution rather than the original link model.}
 This comparison is practically interesting, because dropping queries when the individual cannot answer them is \textit{a priori} a natural solution to indecision. It is much more striking that the \textsc{Utilize} improves so substantially over \textsc{Correct}, which is the response condition under which decisive responses are behaviorally justified. 

This strongly supports the idea that there is distinct and useful information content in indecision. This claim is bolstered by balance of the queries actively sought by the \textsc{Utilize} methods: as shown in \Cref{fig:respdist-main} (and Figure \ref{fig:respdist-grid} in Appendix \ref{app:big-grid} for \textsc{Utilize-3, Utilize-4}$^\diamond$), \textsc{Utilize} learners do not simply avoid indecisive regions, nor do they simply inherit the latent state distribution. Compared with \textsc{Ignore}, whose queried responses become quickly dominated by \(\sim\) as \(\tau_r\) grows and $\bowtie$ as $\tau_\kappa$ grows, \textsc{Utilize} seems to ``fight against" the latent state distribution and seeks a mix of decisive and indecisive queries through most of the threshold grid. This is consistent with the mechanism suggested by the model: when $\tau_r$ and $\tau_\kappa$ are small or intermediate, $\sim$ and $\bowtie$ responses are highly informative because they impose tight constraints on $\omega$. As the thresholds get higher, these queries may become less informative but more unavoidable.

\textbf{Extensions.} An important possibility we need to rule out is that \textsc{Utilize-4}'s advantage over forced comparison settings is coming from our assumption that it knows $\tau_r,\tau_\kappa$ up front. Comparing \textsc{Utilize-4} and \textsc{Utilize-4}$^\diamond$, we see that this concern is unwarranted --- \textsc{Utilize-4}$^\diamond$ learns almost exactly as fast as \textsc{Utilize-4} despite having to learn the thresholds alongside $\widehat{\omega}$. We note that while \textsc{Utilize-4}$^\diamond$ should not generally outperform \textsc{Utilize-4} (as it is solving a harder problem), our results at $\tau_r = 0.8$ appear to contradict this at low $t$. This early edge likely comes from the fact that by virtue of not knowing the thresholds, \textsc{Utilize-4}$^\diamond$ does not update too strongly based on responses that contradict the latent state due to noise, because early on, such responses can be partly attributed to uncertainty in $\tau_r,\tau_\kappa$.

Additionally, in \Cref{fig:diag-summary} (\Cref{app:big-grid}), we investigate whether the richer response alphabet is powerful enough to retain gains in learning speed under weaker assumptions about noise model. To test this, we compare \textsc{Utilize-4}, \textsc{Ignore}, and \textsc{Correct} learners that know the logistic link function to otherwise identical learners that instead fit a flexible mixture-of-Gaussians noise model, denoted by $\dagger$. We find that generalizing the assumed noise model has essentially no effect on the learning speed under \textsc{Utilize-4}, while further compromising learning rates for \textsc{Ignore} and \textsc{Correct}. This suggests that permitting indecisive responses can reduce risk of model misspecification in an additional way, allowing weaker assumptions about the functional form of noise. 

\textbf{Magnitude of gains are practically significant.} The magnitude of the gains described above matter because the relevant query budgets are realistic: a few tens of comparisons per individual are plausible in elicitation
studies, while many tens or hundreds quickly become burdensome. In the intermediate
threshold regimes, \textsc{Utilize} moves accurate learning into this plausible
range. For example, at \((\tau_r,\tau_\kappa)=(0.4,0.4)\),
\textsc{Utilize-4} drops below \(\ell_1=0.05\) after roughly \(25\) queries and
reaches \(\ell_1=0.014\) by \(T=100\), compared with \(0.16\) for
\textsc{Correct} and \(0.34\) for \textsc{Ignore}. At
\((0.6,0.6)\), the contrast with \textsc{Ignore} is even larger:
\textsc{Utilize-4} reaches \(\ell_1=0.012\), while \textsc{Ignore} remains at
\(0.62\). As shown in \Cref{fig:bald_6}, the regret advantage of \textsc{Utilize} methods is also decisive within a realistic budget: by $T=25$ queries, \textsc{Utilize-4}'s average regret is $0.03\%$ of the per-decision utility range at thresholds $(0.4,0.4)$ and $(0.6,0.6)$, versus $1.2\%$/$1.8\%$ (\textsc{Correct}/\textsc{Ignore}) at $(0.4,0.4)$ and $0.8\%$/$5.4\%$ at $(0.6,0.6)$ — a multiplicative gap of $30$–$180$. Gaps are even more pronounced for worst-case regret. By way of explanation, we again demonstrate that these regret gains correspond to more quickly learning the correct direction of $\omega^*$, as quantified by the cosine similarity (Figure \ref{fig:grid-cos}).
\section{Discussion}
\label{sec:discussion}

Taken together, our results in \Cref{sec:conflict} --- despite their limitations as stylized simulations\footnote{First, we assume that the individual reports indecision accurately with respect to their latent states, up to some noise. This assumption could be unrealistically favorable to richer response alphabets, e.g., under the realistic possibility that people will instead overuse the indecision response to avoid exerting effort to decide. Second, the benefits of richer responses may be much greater when queries are chosen via active learning, given our results showing that under richer response alphabets, the learner seems to be seeking a certain ideal balance of query types. Under passive or randomly sampled queries, indecision reports may arise less often, or in less informative regions of the query space.} 
--- suggest that permitting individuals to report indecision in local pairwise comparisons (even without distinguishing between indifference and conflict) can sidestep issues with deviating response behavior due to forcing comparisons, and dramatically improve learning speed, with gains occurring exactly within the practical query complexity regime. These technical reasons for eliciting indecision add to the reasons articulated in \Cref{rem:rule-level-indecision}: learning indecision over queries can help us learn about the individual's indecision over \textit{rules}, and this information has both normative value for understanding the moral authority of the individual's judgments within the decision space, and instrumental value for finding consensus.

While our learning methods from \Cref{sec:conflict} illustrate how a learner can utilize indecision and learn $\tau_r,\tau_\kappa$ in \textit{linear} priority models, \Cref{sec:insep} demonstrates that this is actually a highly restricted case relative to the universe of priorities the individual might care about. A complete solution is out of scope, because \Cref{sec:perfinsep} shows that perfectly inseparable priorities are not identifiable from local pairwise comparison queries alone, so richer query design is required. However, our results do suggest a learning approach under the restriction that $M$'s weights that are identifiable. We call it \textit{priority-aware learning}: if the learner is \textit{aware of the underlying priority structure}, they can interpret responses according to this information and consequently learn the weights. We discuss this approach below.

\subsection{Toward learning general $M$: \textit{Priority-Aware Learning}} \label{sec:priority-aware}
Fix a ground-truth model $M = (u^*, \omega^*)$  in which every $\omega^*_j$ is scale-free identifiable
with respect to $h$. For now, let's assume the learner knows
$u^*$ in advance, so the goal is to learn is $\omega^*$ --- we will address how to elicit $u^*$ below. We also assume, as usual, that the learner knows the response model $R^{\circ}_{h_\beta;\tau_r,\tau_\kappa}$, for generic $\tau_r,\tau_\kappa$. For intuition, we describe the decoder in the deterministic limit
$\beta \to \infty$, where an individual acting according to
$R^{\circ}_{h_\beta;\tau_r,\tau_\kappa}$ reports their latent state exactly; the noisy
case is the Bayesian analog.

Given a query $q = (y, y'; x)$, each possible response imposes a constraint on $\omega$:
\begin{align*}
w = {\succ}   &\implies r_{M^*}(q) \ge \tau_r \ \text{ and }\ \kappa_{M^*}(q) \ge \tau_\kappa\, r_{M^*}(q),\\
w = {\prec}   &\implies r_{M^*}(q) \ge \tau_r \ \text{ and }\ -\kappa_{M^*}(q) \ge \tau_\kappa\, r_{M^*}(q),\\
w = {\bowtie} &\implies r_{M^*}(q) \ge \tau_r \ \text{ and }\ |\kappa_{M^*}(q)| < \tau_\kappa\, r_{M^*}(q),\\
w = {\sim}    &\implies r_{M^*}(q) < \tau_r,
\end{align*}
where both $r_{M^*}(q)$ and $\kappa_{M^*}(q)$ are linear in $\omega$:
\[
r_{M^*}(q) := \sum_{j \in [m]} \omega_j \,\bigl|\phi^{\mathrm{rules}}\bigl((\Delta^F_j(q))_{F \in \mathcal{F}}\bigr)\bigr|,
\qquad
\kappa_{M^*}(q) := \sum_{j \in [m]} \omega_j \,\phi^{\mathrm{rules}}\bigl((\Delta^F_j(q))_{F \in \mathcal{F}}\bigr).
\]
Given the priority utilities $u_j^*$, the learner can compute the coefficients $\prules\big((\Delta_j^F(q))_{F\in\rules}\big)$ above,\footnote{One may wonder whether this is computationally feasible when $\rules$ is large. This depends on the priority structure: if $\rules$ is parametrized and the utility gaps $F\mapsto u_j(F_{x\to y})-u_j(F_{x\to y'})$ are tractable, computing the scores becomes an optimization problem over that parameter space. If exact optimization is difficult, many rule aggregators can be approximated by sampling over background rules.} making these linear constraints on the unknown $\omega$.\footnote{If the thresholds $\tau_r, \tau_\kappa$ are also unknown, learning them can be folded into the same procedure, just as in Section~\ref{sec:speed}. Since $r_{M^*}(q)$ and $\kappa_{M^*}(q)$ are affine in $\omega$,
the threshold $\tau_r$ enters the constraints linearly and can be learned jointly with
$\omega$; the threshold $\tau_\kappa$ enters bilinearly, through the product
$\tau_\kappa\, r_{M^*}(q)$, so it can be handled by grid search.} Given that the $\omega^*$ are scale-free identifiable, the constraints on $\omega$
accumulated along an exhaustive query sequence eventually pin the feasible set down to
$\omega^*$. The query sequence can also be actively derived: due to the linearity of the constraints, active query choice reduces to the standard problem of choosing halfspace constraints that maximally shrink the remaining feasible set of \(\omega\), and by Lemma C.1, approximating $\omega^*$ ensures that we closely approximate the aggregate-optimal rule.\footnote{One may instead want to directly shrink uncertainty over the rule space, possibly aiming for a rule other than the aggregate-optimal rule; this is more complex, requiring a tractable connection between the remaining space of $\omega$ and the remaining rule space.}

\textbf{Eliciting priorities.} Above we assumed the priorities $u^*$ are known before learning their weights. How precisely to elicit these priorities requires empirical study, but one possible procedure could be interactive and text-based: First, the learner could ask the individual to describe textually how they think the rule should and should not behave, to get an initial set of priorities. Then, as the individual is asked comparison queries, the learner could ask \textit{why} they answered as they did, using these explanations to surface additional priorities. Each time, the prior query responses could be reinterpreted over the expanded priority list. While this approach may not recover all priorities, it is reasonable to assume that the most important priorities surface more readily, in which case the recovered priorities will approximately recover the total weight in $\omega^*$.

Of course, a key point of slippage in this approach is the process of translating textually-articulated priorities into utility functions. However, we argue that knowing the exact cardinal utilities attached to each priority is not the point --- and in fact, such precise values over such a large decision space probably don't even exist within the individual. Rather, we suggest that even an approximate priority basis may be more faithful than other approaches to making learning tractable---e.g., forcing preferences into a highly restricted separable score model, or reducing to highly restricted rule classes that structurally cannot serve certain priorities well. The risk of mistranslating priorities could also be managed by the learner maintaining uncertainty over these representations, and seeking further information when that uncertainty is decision-relevant. 

At a higher level, tractable learning over a massive rule space $\rules' \subseteq \rules$ generally requires some kind of simplification or down-projection of preferences into a lower-dimensional representation. The approach proposed here --- of eliciting priorities textually then using their known structure to learn the weights --- can be seen as a version of learning a lower-dimensional representation in which \textit{the dimensions are chosen by the individual themself}. This gets the learning-speed benefits of a lower dimensional representation while using text for what it is especially good at: eliciting the qualitative, non-numeric structure of a person’s judgments. As an added bonus, the resulting representation is also interpretable to the individual, since its dimensions correspond to considerations they can recognize and revise.

\subsection{Applications of the priority model}
In this paper, we used our model as a theoretical tool to formally investigate intuitions about why forced, local pairwise comparisons --- despite being behaviorally justified by standard models (S-RUMs) --- feel as if they might be missing important elements of our values. However, as alluded to by \Cref{sec:priority-aware}, the fact that our model is potentially tractable to learn means that it can also be useful as a preference-learning tool. In thinking about its applications, it is important to note that while this model is formalized for the decision space of decision rules, its fundamental elements can be applied to many decision tasks, and its key insights apply especially when the decision space is large, e.g., combinatorial (as in committee selection or fair division) or continuous (as with decision rules, prediction algorithms, or policies).

One immediate application is deliberation. In AI-mediated deliberative processes, it is tempting to let an LLM implicitly infer what participants value and use that latent representation to summarize disagreement or guide the conversation through complex policy spaces \cite{fish2026generative,tessler2024ai}. Our model suggests a different architecture: learn an explicit model $M_i$ for each participant \(i\), and use the LLM only as an interface to help elicit the priorities within the model. Instead of using the LLM as a black-box interpreter of opinion, this keeps the preference representation outside the LLM as an interpretable object, allowing the individual to inspect whether their priorities, importance weights, and regions of conflict and indifference have been learned correctly.

Such individual models could also serve as richer inputs in decentralized public input processes, e.g., in maxipublics: instead of simply submitting a local vote, people could submit their priorities --- possibly along with a few additional questions to get at relative importance --- leading to much richer information about what trade-offs they would prefer than is captured by a single vote. This suggests a different way to think about platforms such as Polis \cite{small2021polis} or Remesh \cite{konya2023democratic}, which accept short, free-form comments and then cluster them. Using similar text snippets, one could instead elicit priorities, which could be used as richer inputs to downstream deliberative processes, particularly those designed as described in the paragraph above.

Although we study plurality \textit{within} a single individual, our priority model can also be applied to capture plurality \textit{across} individuals. This could be done, e.g., by learning a single shared model across individuals in a deliberative context, or by learning a model for each individual and then aggregating them. Doing this would make clear which priorities are fundamentally at odds across people, allowing people to more directly consent to compromises. In the likely case where people's preferences contain large equivalence classes, this approach could also help identify where people are broadly indifferent, potentially leading to greater opportunities for consensus. 

\newpage
\bibliographystyle{ACM-Reference-Format}
\bibliography{00_arxiv-new/bibliography-old,0_arxiv/bibliography}

@article{gatto2026medical,
  title={Medical Triage as Pairwise Ranking: A Benchmark for Urgency in Patient Portal Messages},
  author={Gatto, Joseph and Seegmiller, Parker and Burdick, Timothy and Resnik, Philip and Rahat, Roshnik and DeLozier, Sarah and Preum, Sarah M},
  journal={arXiv preprint arXiv:2601.13178},
  year={2026}
}

@inproceedings{feffer2023moral,
  author    = {Feffer, Michael and Heidari, Hoda and Lipton, Zachary C.},
  title     = {Moral Machine or Tyranny of the Majority?},
  booktitle = {Proceedings of the 37th AAAI Conference on Artificial Intelligence (AAAI)},
  pages     = {5974--5982},
  year      = {2023}
}

@article{baharav2026end,
  title={The End Justifies the Mean: A Linear Ranking Rule for Proportional Sequential Decisions},
  author={Baharav, Carmel and Boehmer, Niclas and Flanigan, Bailey and Wittmann, Maximilian T},
  journal={arXiv preprint arXiv:2605.12717},
  year={2026}
}

@incollection{lang_xia_combinatorial_2016,
  author    = {Lang, J{\'e}r{\^o}me and Xia, Lirong},
  title     = {Voting in Combinatorial Domains},
  booktitle = {Handbook of Computational Social Choice},
  editor    = {Brandt, Felix and Conitzer, Vincent and Endriss, Ulle and Lang, J{\'e}r{\^o}me and Procaccia, Ariel D.},
  chapter   = {9},
  pages     = {197--222},
  publisher = {Cambridge University Press},
  year      = {2016},
  doi       = {10.1017/CBO9781107446984.010}
}

@inproceedings{barrot_lang_conditional_2016,
  author    = {Barrot, Nathana{\"e}l and Lang, J{\'e}r{\^o}me},
  title     = {Conditional and Sequential Approval Voting on Combinatorial Domains},
  booktitle = {Proceedings of the Twenty-Fifth International Joint Conference on Artificial Intelligence},
  series    = {IJCAI 2016},
  pages     = {88--94},
  year      = {2016}
}

@article{ratliff_selecting_2006,
  author  = {Ratliff, Thomas C.},
  title   = {Selecting Committees},
  journal = {Public Choice},
  volume  = {126},
  number  = {3--4},
  pages   = {343--355},
  year    = {2006},
  doi     = {10.1007/s11127-006-1747-5}
}

@article{ratliff_saari_diverse_2014,
  author  = {Ratliff, Thomas C. and Saari, Donald G.},
  title   = {Complexities of Electing Diverse Committees},
  journal = {Social Choice and Welfare},
  volume  = {43},
  number  = {1},
  pages   = {55--71},
  year    = {2014},
  doi     = {10.1007/s00355-013-0773-8}
}

@inproceedings{uckelman_alice_2010,
  author    = {Uckelman, Joel},
  title     = {Alice and Bob Will Fight: The Problem of Electing a Committee in the Presence of Candidate Interdependence},
  booktitle = {ECAI 2010: 19th European Conference on Artificial Intelligence},
  editor    = {Coelho, Helder and Studer, Rudi and Wooldridge, Michael},
  series    = {Frontiers in Artificial Intelligence and Applications},
  volume    = {215},
  pages     = {1023--1024},
  publisher = {IOS Press},
  address   = {Amsterdam},
  year      = {2010},
  doi       = {10.3233/978-1-60750-606-5-1023}
}

@misc{ge_linear_2026,
  author        = {Ge, Luise and Halpern, Daniel and Kehne, Gregory and Vorobeychik, Yevgeniy},
  title         = {Linear Social Choice with Few Queries: A Moment-Based Approach},
  year          = {2026},
  eprint        = {2603.19510},
  archivePrefix = {arXiv},
  primaryClass  = {cs.GT}
}

@article{cappelen2007pluralism,
  author  = {Cappelen, Alexander W. and Hole, Astri Drange and S{\o}rensen, Erik {\O}. and Tungodden, Bertil},
  title   = {The Pluralism of Fairness Ideals: An Experimental Approach},
  journal = {American Economic Review},
  volume  = {97}, number = {3}, pages = {818--827}, year = {2007}
}

@incollection{leventhal1980what,
  author    = {Leventhal, Gerald S.},
  title     = {What Should Be Done with Equity Theory?},
  booktitle = {Social Exchange: Advances in Theory and Research},
  editor    = {Gergen, Kenneth J. and Greenberg, Martin S. and Willis, Richard H.},
  pages     = {27--55}, publisher = {Plenum Press}, address = {New York}, year = {1980}
}

@book{ben1985discrete,
  title={Discrete choice analysis: theory and application to travel demand},
  author={Ben-Akiva, Moshe E and Lerman, Steven R},
  volume={9},
  year={1985},
  publisher={MIT press}
}

@article{graham2011mapping,
  author  = {Graham, Jesse and Nosek, Brian A. and Haidt, Jonathan and Iyer, Ravi and Koleva, Spassena and Ditto, Peter H.},
  title   = {Mapping the Moral Domain},
  journal = {Journal of Personality and Social Psychology},
  year    = {2011},
  volume  = {101},
  pages   = {366--385},
  doi     = {10.1037/a0021847}
}

@article{mosteller1951remarks,
  title={Remarks on the method of paired comparisons: I. The least squares solution assuming equal standard deviations and equal correlations},
  author={Mosteller, Frederick},
  journal={Psychometrika},
  volume={16},
  number={1},
  pages={3--9},
  year={1951},
  publisher={Springer-Verlag}
}

@article{tetlock1987accountability,
  author  = {Tetlock, Philip E. and Kim, Jae Il},
  title   = {Accountability and Judgment Processes in a Personality Prediction Task},
  journal = {Journal of Personality and Social Psychology},
  year    = {1987},
  volume  = {52},
  number  = {4},
  pages   = {700--709},
  doi     = {10.1037/0022-3514.52.4.700}
}

@article{haidt2001emotional,
  author  = {Haidt, Jonathan},
  title   = {The Emotional Dog and Its Rational Tail: A Social Intuitionist Approach to Moral Judgment},
  journal = {Psychological Review},
  year    = {2001},
  volume  = {108},
  number  = {4},
  pages   = {814--834},
  doi     = {10.1037/0033-295X.108.4.814}
}

@techreport{haidt2000dumbfounding,
  author      = {Haidt, Jonathan and Bj{\"o}rklund, Fredrik and Murphy, Scott},
  title       = {Moral Dumbfounding: When Intuition Finds No Reason},
  institution = {University of Virginia},
  year        = {2000},
  note        = {Unpublished manuscript / working paper (dated August 10, 2000)}
}

@article{gawronski2017consequences, title = {Consequences, Norms, and Generalized Inaction in Moral Dilemmas: The {CNI} Model of Moral Decision-Making}, author = {Gawronski, Bertram and Armstrong, Joel and Conway, Paul and Friesdorf, Rebecca and H{\"u}tter, Mandy}, journal = {Journal of Personality and Social Psychology}, volume = {113}, number = {3}, pages = {343--376}, year = {2017}, doi = {10.1037/pspa0000086} }

@article{tetlock2003thinking,
  title={Thinking the unthinkable: Sacred values and taboo cognitions},
  author={Tetlock, Philip E},
  journal={Trends in cognitive sciences},
  volume={7},
  number={7},
  pages={320--324},
  year={2003},
  publisher={Elsevier}
}

@article{guzman2022moral,
  title={A moral trade-off system produces intuitive judgments that are rational and coherent and strike a balance between conflicting moral values},
  author={Guzm{\'a}n, Ricardo Andr{\'e}s and Barbato, Mar{\'\i}a Teresa and Sznycer, Daniel and Cosmides, Leda},
  journal={Proceedings of the National Academy of Sciences},
  volume={119},
  number={42},
  pages={e2214005119},
  year={2022},
  publisher={National Academy of Sciences}
}

@book{keeneyraiffa1976,
  author    = {Keeney, Ralph L. and Raiffa, Howard},
  title     = {Decisions with Multiple Objectives: Preferences and Value Tradeoffs},
  publisher = {John Wiley \& Sons},
  address   = {New York},
  year      = {1976}
}

@article{baronspranca1997,
  author  = {Baron, Jonathan and Spranca, Mark},
  title   = {Protected Values},
  journal = {Organizational Behavior and Human Decision Processes},
  year    = {1997},
  volume  = {70},
  number  = {1},
  pages   = {1--16},
  doi     = {10.1006/obhd.1997.2690}
}

@book{pettycacioppo1986,
  author    = {Petty, Richard E. and Cacioppo, John T.},
  title     = {Communication and Persuasion: Central and Peripheral Routes to Attitude Change},
  publisher = {Springer-Verlag},
  address   = {New York},
  year      = {1986},
  doi       = {10.1007/978-1-4612-4964-1}
}

@article{fishkinluskin2005,
  author  = {Fishkin, James S. and Luskin, Robert C.},
  title   = {Experimenting with a Democratic Ideal: Deliberative Polling and Public Opinion},
  journal = {Acta Politica},
  year    = {2005},
  volume  = {40},
  number  = {3},
  pages   = {284--298}
}

@article{lancsar2008conducting,
  title={Conducting discrete choice experiments to inform healthcare decision making: a user’s guide},
  author={Lancsar, Emily and Louviere, Jordan},
  journal={Pharmacoeconomics},
  volume={26},
  number={8},
  pages={661--677},
  year={2008},
  publisher={Springer}
}

@article{green1978conjoint,
  title={Conjoint Analysis in Consumer Research: Issues and Outlook},
  author={Green, Paul E. and Srinivasan, V.},
  journal={Journal of Consumer Research},
  volume={5},
  number={2},
  pages={103--123},
  year={1978}
}

@article{ge2024axioms,
  title={Axioms for AI Alignment from Human Feedback},
  author={Ge, Luise and Halpern, Daniel and Micha, Evi and Procaccia, Ariel D. and Shapira, Itai and Vorobeychik, Yevgeniy and Wu, Junlin},
  journal={arXiv preprint arXiv:2405.14758},
  year={2024}
}

@article{hainmueller2014conjoint,
  title={Causal Inference in Conjoint Analysis: Understanding Multidimensional Choices via Stated Preference Experiments},
  author={Hainmueller, Jens and Hopkins, Daniel J. and Yamamoto, Teppei},
  journal={Political Analysis},
  volume={22},
  number={1},
  pages={1--30},
  year={2014}
}

@article{RaoKupper1967,
  author  = {Rao, P. V. and Kupper, L. L.},
  title   = {Ties in Paired-Comparison Experiments: A Generalization of the Bradley-Terry Model},
  journal = {Journal of the American Statistical Association},
  volume  = {62},
  number  = {317},
  pages   = {194--204},
  year    = {1967},
  doi     = {10.1080/01621459.1967.10482901}
}

@article{davidson,
 ISSN = {01621459, 1537274X},
 URL = {http://www.jstor.org/stable/2283595},
 author = {Roger R. Davidson},
 journal = {Journal of the American Statistical Association},
 number = {329},
 pages = {317--328},
 publisher = {[American Statistical Association, Taylor & Francis, Ltd.]},
 title = {On Extending the Bradley-Terry Model to Accommodate Ties in Paired Comparison Experiments},
 urldate = {2026-02-23},
 volume = {65},
 year = {1970}
}

@inproceedings{noothigattu2018voting,
  title={A voting-based system for ethical decision making},
  author={Noothigattu, Ritesh and Gaikwad, Snehalkumar and Awad, Edmond and Dsouza, Sohan and Rahwan, Iyad and Ravikumar, Pradeep and Procaccia, Ariel},
  booktitle={Proceedings of the AAAI Conference on Artificial Intelligence},
  volume={32},
  number={1},
  year={2018}
}

@article{awad2018moral,
  title={The moral machine experiment},
  author={Awad, Edmond and Dsouza, Sohan and Kim, Richard and Schulz, Jonathan and Henrich, Joseph and Shariff, Azim and Bonnefon, Jean-Fran{\c{c}}ois and Rahwan, Iyad},
  journal={Nature},
  volume={563},
  number={7729},
  pages={59--64},
  year={2018},
  publisher={Nature Publishing Group UK London}
}

@inproceedings{shirali2024participatory,
  title={Participatory Objective Design via Preference Elicitation},
  author={Shirali, Ali and Finocchiaro, Jessie and Abebe, Rediet},
  booktitle={Proceedings of the 2024 ACM Conference on Fairness, Accountability, and Transparency},
  pages={1637--1662},
  year={2024}
}

@article{freedman2020adapting,
  title={Adapting a kidney exchange algorithm to align with human values},
  author={Freedman, Rachel and Borg, Jana Schaich and Sinnott-Armstrong, Walter and Dickerson, John P and Conitzer, Vincent},
  journal={Artificial Intelligence},
  volume={283},
  pages={103261},
  year={2020},
  publisher={Elsevier}
}

@article{tversky1972elimination,
  title={Elimination by aspects: A theory of choice.},
  author={Tversky, Amos},
  journal={Psychological review},
  volume={79},
  number={4},
  pages={281},
  year={1972},
  publisher={American Psychological Association}
}

@article{tetlock1986value,
  title={A value pluralism model of ideological reasoning},
  author={Tetlock, Philip E},
  journal={Journal of Personality and Social Psychology},
  volume={50},
  number={4},
  pages={819},
  year={1986},
  publisher={American Psychological Association}
}

@book{payne1993adaptive,
  title={The Adaptive Decision Maker},
  author={Payne, John W and Bettman, James R and Johnson, Eric J},
  year={1993},
  publisher={Cambridge University Press}
}

@article{tversky1990causes,
  title={The causes of preference reversal},
  author={Tversky, Amos and Slovic, Paul and Kahneman, Daniel},
  journal={The American Economic Review},
  volume={80},
  number={1},
  pages={204--217},
  year={1990},
  publisher={JSTOR}
}

@article{slovic1995construction,
  title={The construction of preference},
  author={Slovic, Paul},
  journal={American Psychologist},
  volume={50},
  number={5},
  pages={364},
  year={1995},
  publisher={American Psychological Association}
}

@article{deutsch1975equity,
  author  = {Deutsch, Morton},
  title   = {Equity, Equality, and Need: What Determines Which Value Will Be Used as the Basis of Distributive Justice?},
  journal = {Journal of Social Issues},
  year    = {1975},
  volume  = {31},
  number  = {3},
  pages   = {137--149},
  doi     = {10.1111/j.1540-4560.1975.tb01000.x}
}

@incollection{lind2001fairness,
  author    = {Lind, E. Allan},
  title     = {Fairness Heuristic Theory: Justice Judgments as Pivotal Cognitions in Organizational Relations},
  booktitle = {Advances in Organizational Justice},
  editor    = {Greenberg, Jerald and Cropanzano, Russell},
  pages     = {56--88},
  publisher = {Stanford University Press},
  year      = {2001}
}

@article{jiang2024survey,
  title={A survey on human preference learning for large language models},
  author={Jiang, Ruili and Chen, Kehai and Bai, Xuefeng and He, Zhixuan and Li, Juntao and Yang, Muyun and Zhao, Tiejun and Nie, Liqiang and Zhang, Min},
  journal={arXiv preprint arXiv:2406.11191},
  year={2024}
}

@inproceedings{huang2024collective,
  title={Collective constitutional ai: Aligning a language model with public input},
  author={Huang, Saffron and Siddarth, Divya and Lovitt, Liane and Liao, Thomas I and Durmus, Esin and Tamkin, Alex and Ganguli, Deep},
  booktitle={Proceedings of the 2024 ACM Conference on Fairness, Accountability, and Transparency},
  pages={1395--1417},
  year={2024}
}

@article{boerstler2024stability, title = {On The Stability of Moral Preferences: A Problem with Computational Elicitation Methods}, author = {Boerstler, Kyle and Keswani, Vijay and Chan, Lok and Schaich Borg, Jana and Conitzer, Vincent and Heidari, Hoda and Sinnott-Armstrong, Walter}, journal = {Proceedings of the AAAI/ACM Conference on AI, Ethics, and Society}, volume = {7}, number = {1}, pages = {156--167}, year = {2024}, doi = {10.1609/aies.v7i1.31626} }

@article{sorensen2024roadmap,
  title={A roadmap to pluralistic alignment},
  author={Sorensen, Taylor and Moore, Jared and Fisher, Jillian and Gordon, Mitchell and Mireshghallah, Niloofar and Rytting, Christopher Michael and Ye, Andre and Jiang, Liwei and Lu, Ximing and Dziri, Nouha and others},
  journal={arXiv preprint arXiv:2402.05070},
  year={2024}
}

@article{konya2023democratic,
  title={Democratic policy development using collective dialogues and AI},
  author={Konya, Andrew and Schirch, Lisa and Irwin, Colin and Ovadya, Aviv},
  journal={arXiv preprint arXiv:2311.02242},
  year={2023}
}

@article{small2021polis,
  title={Polis: Scaling deliberation by mapping high dimensional opinion spaces},
  author={Small, Christopher and Bjorkegren, Michael and Erkkil{\"a}, Timo and Shaw, Lynette and Megill, Colin},
  journal={Recerca: revista de pensament i an{\`a}lisi},
  volume={26},
  number={2},
  year={2021},
  publisher={Universitat Jaume I Servei de Comunicacio i Publicacions}
}

@article{tessler2024ai,
  title={AI can help humans find common ground in democratic deliberation},
  author={Tessler, Michael Henry and Bakker, Michiel A and Jarrett, Daniel and Sheahan, Hannah and Chadwick, Martin J and Koster, Raphael and Evans, Georgina and Campbell-Gillingham, Lucy and Collins, Tantum and Parkes, David C and others},
  journal={Science},
  volume={386},
  number={6719},
  pages={eadq2852},
  year={2024},
  publisher={American Association for the Advancement of Science}
}

@article{fish2026generative,
  title={Generative social choice},
  author={Fish, Sara and G{\"o}lz, Paul and Parkes, David and Procaccia, Ariel and Rusak, Gili and Shapira, Itai and Wuthrich, Manuel},
  journal={Journal of the ACM},
  volume={73},
  number={2},
  pages={1--52},
  year={2026},
  publisher={ACM New York, NY}
}

@article{rosas2019decision,
  title={Decision conflict drives reaction times and utilitarian responses in sacrificial dilemmas},
  author={Rosas, Alejandro and Berm{\'u}dez, Juan Pablo and Aguilar-Pardo, David},
  journal={Judgment and Decision Making},
  volume={14},
  number={5},
  pages={555--564},
  year={2019},
  publisher={Cambridge University Press}
}

@article{metropolis1953equation,
  title = {Equation of State Calculations by Fast Computing Machines},
  author = {Metropolis, Nicholas and Rosenbluth, Arianna W. and Rosenbluth, Marshall N. and Teller, Augusta H. and Teller, Edward},
  journal = {The Journal of Chemical Physics},
  volume = {21},
  number = {6},
  pages = {1087--1092},
  year = {1953},
  doi = {10.1063/1.1699114}
}

@article{hastings1970monte,
  title = {Monte Carlo Sampling Methods Using Markov Chains and Their Applications},
  author = {Hastings, W. K.},
  journal = {Biometrika},
  volume = {57},
  number = {1},
  pages = {97--109},
  year = {1970},
  doi = {10.1093/biomet/57.1.97}
}

@article{gelfand1990sampling,
  title = {Sampling-Based Approaches to Calculating Marginal Densities},
  author = {Gelfand, Alan E. and Smith, Adrian F. M.},
  journal = {Journal of the American Statistical Association},
  volume = {85},
  number = {410},
  pages = {398--409},
  year = {1990},
  doi = {10.1080/01621459.1990.10476213}
}

@incollection{kim2015guide,
  title = {A Guide to Sample Average Approximation},
  author = {Kim, Sujin and Pasupathy, Raghu and Henderson, Shane G.},
  editor = {Fu, Michael C.},
  booktitle = {Handbook of Simulation Optimization},
  series = {International Series in Operations Research \& Management Science},
  volume = {216},
  pages = {207--243},
  publisher = {Springer},
  address = {New York, NY},
  year = {2015},
  doi = {10.1007/978-1-4939-1384-8_8}
}

@article{tierney1994markov,
  title = {Markov Chains for Exploring Posterior Distributions},
  author = {Tierney, Luke},
  journal = {The Annals of Statistics},
  volume = {22},
  number = {4},
  pages = {1701--1762},
  year = {1994},
  doi = {10.1214/aos/1176325750}
}

@article{chib1995understanding,
  title = {Understanding the Metropolis-Hastings Algorithm},
  author = {Chib, Siddhartha and Greenberg, Edward},
  journal = {The American Statistician},
  volume = {49},
  number = {4},
  pages = {327--335},
  year = {1995},
  doi = {10.1080/00031305.1995.10476177}
}

@article{smith1984efficient,
  title = {Efficient Monte Carlo Procedures for Generating Points Uniformly Distributed over Bounded Regions},
  author = {Smith, Robert L.},
  journal = {Operations Research},
  volume = {32},
  number = {6},
  pages = {1296--1308},
  year = {1984},
  doi = {10.1287/opre.32.6.1296}
}

@inproceedings{wilson2018maximizing,
  title = {Maximizing Acquisition Functions for Bayesian Optimization},
  author = {Wilson, James T. and Hutter, Frank and Deisenroth, Marc Peter},
  booktitle = {Advances in Neural Information Processing Systems 31},
  pages = {9884--9895},
  year = {2018},
  publisher = {Curran Associates, Inc.},
  url = {https://proceedings.neurips.cc/paper/2018/hash/498f2c21688f6451d9f5fd09d53edda7-Abstract.html}
}

@article{ge2024learning,
  title={Learning linear utility functions from pairwise comparison queries},
  author={Ge, Luise and Juba, Brendan and Vorobeychik, Yevgeniy},
  journal={arXiv preprint arXiv:2405.02612},
  year={2024}
}

@inproceedings{kane2017active,
  author    = {Kane, Daniel M. and Lovett, Shachar and Moran, Shay and Zhang, Jiapeng},
  title     = {Active Classification with Comparison Queries},
  booktitle = {2017 IEEE 58th Annual Symposium on Foundations of Computer Science (FOCS)},
  pages     = {355--366},
  year      = {2017},
  doi       = {10.1109/FOCS.2017.40}
}

@article{keswani2025moralchange,
  title   = {Moral Change or Noise? On Problems of Aligning AI With Temporally Unstable Human Feedback},
  author  = {Keswani, Vijay and Cousins, Cyrus and Nguyen, Breanna and Conitzer, Vincent and Heidari, Hoda and Schaich Borg, Jana and Sinnott-Armstrong, Walter},
  journal = {arXiv preprint arXiv:2511.10032},
  year    = {2025},
  doi     = {10.48550/arXiv.2511.10032},
  note    = {To appear in the AAAI 2026 Alignment Track},
  url     = {https://arxiv.org/abs/2511.10032}
}

@article{cousins2025cognitivelyfaithful,
  title   = {Towards Cognitively-Faithful Decision-Making Models to Improve AI Alignment},
  author  = {Cousins, Cyrus and Keswani, Vijay and Conitzer, Vincent and Heidari, Hoda and Schaich Borg, Jana and Sinnott-Armstrong, Walter},
  journal = {arXiv preprint arXiv:2509.04445},
  year    = {2025},
  url     = {https://arxiv.org/abs/2509.04445},
  note    = {OpenReview: https://openreview.net/forum?id=ziP9zetlLp}
}

@article{liscio2025value, title = {Value Preferences Estimation and Disambiguation in Hybrid Participatory Systems}, author = {Liscio, Enrico and Siebert, Luciano C. and Jonker, Catholijn M. and Murukannaiah, Pradeep K.}, journal = {Journal of Artificial Intelligence Research}, volume = {82}, pages = {819--850}, year = {2025}, doi = {10.1613/jair.1.14958} }

@inproceedings{SiebertEtAl2022HHAI,
  title     = {Estimating Value Preferences in a Hybrid Participatory System},
  author    = {Siebert, Luciano C. and Liscio, Enrico and Murukannaiah, Pradeep K. and Kaptein, Lionel and Spruit, Shannon and van den Hoven, Jeroen and Jonker, Catholijn M.},
  booktitle = {HHAI 2022: Augmenting Human Intellect},
  series    = {Frontiers in Artificial Intelligence and Applications},
  volume    = {354},
  pages     = {114--127},
  publisher = {IOS Press},
  year      = {2022},
  doi       = {10.3233/FAIA220193},
  url       = {https://pure.tudelft.nl/ws/files/139879518/FAIA_354_FAIA220193.pdf}
}

@inproceedings{mohsin2022learning,
  title={Learning Individual and Collective Priorities over Moral Dilemmas with the Life Jacket Dataset},
  author={Mohsin, Farhad and Kang, Inwon and Chen, Pin-Yu and Rossi, Francesca and Xia, Lirong},
  booktitle={13th Multidisciplinary Workshop on Adv. in Preference Handling, Vienna, Austria},
  year={2022}
}

@inproceedings{keswani2025can,
  title={Can ai model the complexities of human moral decision-making? a qualitative study of kidney allocation decisions},
  author={Keswani, Vijay and Conitzer, Vincent and Sinnott-Armstrong, Walter and Nguyen, Breanna K and Heidari, Hoda and Schaich Borg, Jana},
  booktitle={Proceedings of the 2025 CHI Conference on Human Factors in Computing Systems},
  pages={1--17},
  year={2025}
}

@inproceedings{johnston_deploying_2023,
	title = {Deploying a {Robust} {Active} {Preference} {Elicitation} {Algorithm} on {MTurk}: {Experiment} {Design}, {Interface}, and {Evaluation} for {COVID}-19 {Patient} {Prioritization}},
	shorttitle = {Deploying a {Robust} {Active} {Preference} {Elicitation} {Algorithm} on {MTurk}},
	url = {http://arxiv.org/abs/2306.04061},
	doi = {10.1145/3617694.3623254},
	urldate = {2025-12-08},
	booktitle = {Equity and {Access} in {Algorithms}, {Mechanisms}, and {Optimization}},
	author = {Johnston, Caroline M. and Vossler, Patrick and Blessenohl, Simon and Vayanos, Phebe},
	month = oct,
	year = {2023},
	note = {arXiv:2306.04061 [cs]},
	pages = {1--10},
}

@misc{keswani_pros_2024,
	title = {On the {Pros} and {Cons} of {Active} {Learning} for {Moral} {Preference} {Elicitation}},
	url = {https://arxiv.org/abs/2407.18889v1},
	language = {en},
	urldate = {2025-12-08},
	journal = {arXiv.org},
	author = {Keswani, Vijay and Conitzer, Vincent and Heidari, Hoda and Borg, Jana Schaich and Sinnott-Armstrong, Walter},
	month = jul,
	year = {2024},
}

@article{colquitt2001dimensionality, title = {On the Dimensionality of Organizational Justice: A Construct Validation of a Measure}, author = {Colquitt, Jason A.}, journal = {Journal of Applied Psychology}, volume = {86}, number = {3}, pages = {386--400}, year = {2001}, doi = {10.1037/0021-9010.86.3.386} }

@article{fisman2007individual,
  title={Individual preferences for giving},
  author={Fisman, Raymond and Kariv, Shachar and Markovits, Daniel},
  journal={American Economic Review},
  volume={97},
  number={5},
  pages={1858--1876},
  year={2007},
  publisher={American Economic Association}
}

@article{lee2019webuildai, title = {{WeBuildAI}: Participatory Framework for Algorithmic Governance}, author = {Lee, Min Kyung and Kusbit, Daniel and Kahng, Anson and Kim, Ji Tae and Yuan, Xinran and Chan, Allissa and See, Daniel and Noothigattu, Ritesh and Lee, Siheon and Psomas, Alexandros and Procaccia, Ariel D.}, journal = {Proceedings of the ACM on Human-Computer Interaction}, volume = {3}, number = {CSCW}, articleno = {181}, pages = {1--35}, year = {2019}, publisher = {Association for Computing Machinery}, doi = {10.1145/3359283} }

@article{TverskyKahneman1981Framing,
  author  = {Tversky, Amos and Kahneman, Daniel},
  title   = {The Framing of Decisions and the Psychology of Choice},
  journal = {Science},
  year    = {1981},
  volume  = {211},
  number  = {4481},
  pages   = {453--458},
  doi     = {10.1126/science.7455683}
}

@article{CacioppoBerntson1994EvaluativeSpace,
  author  = {Cacioppo, John T. and Berntson, Gary G.},
  title   = {Relationship Between Attitudes and Evaluative Space: A Critical Review, with Emphasis on the Separability of Positive and Negative Substrates},
  journal = {Psychological Bulletin},
  year    = {1994},
  volume  = {115},
  number  = {3},
  pages   = {401--423},
  doi     = {10.1037/0033-2909.115.3.401}
}

@article{VanHarreveldVanDerPligtDeLiver2009MAID,
  author  = {van Harreveld, Frenk and van der Pligt, Joop and de Liver, Yael N.},
  title   = {The Agony of Ambivalence and Ways to Resolve It: Introducing the MAID Model},
  journal = {Personality and Social Psychology Review},
  year    = {2009},
  volume  = {13},
  number  = {1},
  pages   = {45--61},
  doi     = {10.1177/1088868308324518}
}

@article{SchneiderSchwarz2017MixedFeelings,
  author  = {Schneider, Iris K. and Schwarz, Norbert},
  title   = {Mixed Feelings: The Case of Ambivalence},
  journal = {Current Opinion in Behavioral Sciences},
  year    = {2017},
  volume  = {15},
  pages   = {39--45},
  doi     = {10.1016/j.cobeha.2017.05.012}
}

@article{Anderson2003DecisionAvoidance,
  author  = {Anderson, Christopher J.},
  title   = {The Psychology of Doing Nothing: Forms of Decision Avoidance Result from Reason and Emotion},
  journal = {Psychological Bulletin},
  year    = {2003},
  volume  = {129},
  number  = {1},
  pages   = {139--167},
  doi     = {10.1037/0033-2909.129.1.139}
}

@article{Dhar1997NoChoiceOption,
  author  = {Dhar, Ravi},
  title   = {Consumer Preference for a No-Choice Option},
  journal = {Journal of Consumer Research},
  year    = {1997},
  volume  = {24},
  number  = {2},
  pages   = {215--231}
}

@article{TverskyShafir1992ChoiceUnderConflict,
  author  = {Tversky, Amos and Shafir, Eldar},
  title   = {Choice under Conflict: The Dynamics of Deferred Decision},
  journal = {Psychological Science},
  year    = {1992},
  volume  = {3},
  number  = {6},
  pages   = {358--361},
  doi     = {10.1111/j.1467-9280.1992.tb00047.x}
}

@misc{liu2024reward, title = {Reward Learning From Preference With Ties}, author = {Liu, Jinsong and Ge, Dongdong and Zhu, Ruihao}, year = {2024}, eprint = {2410.05328}, archivePrefix = {arXiv}, primaryClass = {cs.LG}, url = {https://arxiv.org/abs/2410.05328} }

@article{conitzer2024socialchoice,
  title={Social choice should guide AI alignment in dealing with diverse human feedback},
  author={Conitzer, Vincent and others},
  journal={arXiv preprint arXiv:2404.10271},
  year={2024}
}

@article{houlsby2011bayesian,
  title={Bayesian active learning for classification and preference learning},
  author={Houlsby, Neil and Husz{\'a}r, Ferenc and Ghahramani, Zoubin and Lengyel, M{\'a}t{\'e}},
  journal={arXiv preprint arXiv:1112.5745},
  year={2011}
}

@article{dieteren2022should,
  title={How should ICU beds be allocated during a crisis? Evidence from the COVID-19 pandemic},
  author={Dieteren, Charlotte M and Van Hulsen, Merel AJ and Rohde, Kirsten IM and Van Exel, Job},
  journal={Plos one},
  volume={17},
  number={8},
  pages={e0270996},
  year={2022},
  publisher={Public Library of Science San Francisco, CA USA}
}

@inproceedings{lee2017human,
  title={A human-centered approach to algorithmic services: Considerations for fair and motivating smart community service management that allocates donations to non-profit organizations},
  author={Lee, Min Kyung and Kim, Ji Tae and Lizarondo, Leah},
  booktitle={Proceedings of the 2017 CHI conference on human factors in computing systems},
  pages={3365--3376},
  year={2017}
}

@inproceedings{mcelfresh2021indecision,
  title={Indecision modeling},
  author={McElfresh, Duncan C and Chan, Lok and Doyle, Kenzie and Sinnott-Armstrong, Walter and Conitzer, Vincent and Borg, Jana Schaich and Dickerson, John P},
  booktitle={Proceedings of the AAAI Conference on Artificial Intelligence},
  volume={35},
  number={7},
  pages={5975--5983},
  year={2021}
}

@book{fiske1984social,
  author    = {Fiske, Susan T. and Taylor, Shelley E.},
  title     = {Social Cognition},
  year      = {1984},
  publisher = {Addison-Wesley},
  address   = {Reading, MA}
}

@incollection{greene2008secret,
  author    = {Greene, Joshua},
  title     = {The Secret Joke of Kant's Soul},
  booktitle = {Moral Psychology},
  editor    = {Sinnott-Armstrong, Walter},
  volume    = {3},
  pages     = {35--80},
  year      = {2008},
  publisher = {MIT Press}
}

@article{cushman2012finding,
  author  = {Cushman, Fiery and Greene, Joshua D.},
  title   = {Finding Faults: How Moral Dilemmas Illuminate Cognitive Structure},
  journal = {Social Neuroscience},
  volume  = {7},
  number  = {3},
  pages   = {269--279},
  year    = {2012},
  doi     = {10.1080/17470919.2011.614000}
}

@inproceedings{kim2018computational,
  title={A computational model of commonsense moral decision making},
  author={Kim, Richard and Kleiman-Weiner, Max and Abeliuk, Andr{\'e}s and Awad, Edmond and Dsouza, Sohan and Tenenbaum, Joshua B and Rahwan, Iyad},
  booktitle={Proceedings of the 2018 AAAI/ACM Conference on AI, Ethics, and Society},
  pages={197--203},
  year={2018}
}

@article{tversky1974judgment,
  title={Judgment under uncertainty: Heuristics and biases},
  author={Tversky, Amos and Kahneman, Daniel},
  journal={Science},
  volume={185},
  number={4157},
  pages={1124--1131},
  year={1974},
  publisher={American Association for the Advancement of Science}
}

@book{kahneman2011thinking,
  title={Thinking, fast and slow},
  author={Kahneman, Daniel},
  year={2011},
  publisher={Macmillan}
}

@article{graham2013moral,
  title   = {Moral Foundations Theory: The Pragmatic Validity of Moral Pluralism},
  author  = {Graham, Jesse and Waytz, Adam and Meindl, Peter and Iyer, Ravi and Koleva, Spassena and Haidt, Jonathan},
  journal = {Advances in Experimental Social Psychology},
  volume  = {47},
  pages   = {55--130},
  year    = {2013},
  doi     = {10.1016/B978-0-12-407236-7.00002-4}
}

@book{lind1988social,
  title     = {The Social Psychology of Procedural Justice},
  author    = {Lind, E. Allan and Tyler, Tom R.},
  year      = {1988},
  publisher = {Plenum Press},
  address   = {New York},
  doi       = {10.1007/978-1-4899-2115-4}
}

@article{levi2009conceptualizing,
  title   = {Conceptualizing Legitimacy, Measuring Legitimating Beliefs},
  author  = {Levi, Margaret and Sacks, Audrey and Tyler, Tom R.},
  journal = {American Behavioral Scientist},
  volume  = {53},
  number  = {3},
  pages   = {354--375},
  year    = {2009},
  doi     = {10.1177/0002764209338797}
}

@article{bradley1952rank,
  title={Rank analysis of incomplete block designs: I. The method of paired comparisons},
  author={Bradley, Ralph Allan and Terry, Milton E.},
  journal={Biometrika},
  volume={39},
  number={3/4},
  pages={324--345},
  year={1952}
}

@article{thurstone1927law,
  title={A law of comparative judgment},
  author={Thurstone, L. L.},
  journal={Psychological Review},
  volume={34},
  number={4},
  pages={273--286},
  year={1927}
}

\newpage

\appendix
\crefalias{section}{appendix}
\crefalias{subsection}{appendix}
\crefalias{subsubsection}{appendix}
\section{Supplemental Materials from \Cref{sec:model}}

\subsection{Foundations in Behavioral and Decision Sciences}\label{sec:moralpsych} \label{app:behavioral-science}

\subsubsection{Preferences over rules}
One key difference between our model and S-RUMs is that we define preferences at the level of \textit{rules}, from which preferences over outcomes at individual inputs are constructed. This approach lets the model encompass the possibility that the individual evaluates not only realized outcomes but also the rules and procedures that generate them. Research on procedural justice emphasizes that individuals care about the fairness of processes, and that process judgments shape attitudes such as trust and willingness to comply \cite{lind1988social,lind2001fairness,levi2009conceptualizing}. Distributive justice work likewise highlights that judgments about what is fair invoke multiple competing values (e.g., equity, equality, need) that reflect how a rule behaves \textit{across outcomes}, making rules the natural locus for studying value conflict \cite{deutsch1975equity}.

\subsubsection{Internal pluralism}
The second core feature of our model is internal pluralism, as captured in the multiple priorities defining $M$. At a high level, our priority model is closest in spirit to Tetlock's influential \textit{value pluralism model of ideological reasoning}, which was originally proposed to explain individual differences in political reasoning \cite{tetlock1986value}. Tetlock’s proposal was articulated textually and operationalized as a measurement tool, rather than specified as a formal mathematical model of judgment, as we do. As described in \Cref{rem:behavioral-sciences}, our priority model captures many of the intuitions described by Tetlock, including decisions selectively activating certain values and/or bringing values into conflict. Tetlock also alludes to the possibility that some values are more important to an individual than others, a concept modeled by our $\omega$ vector. 

Given that our model so closely imitates Tetlock's setup, it is natural that our model must also engage closely with Tetlock's two core elements of moral reasoning: \textit{differentiation} and \textit{integration.} \textit{Differentiation} is the extent to which a person recognizes multiple relevant considerations in a decision, captured by the multiplicity of priorities in $M$. \textit{Integration} refers to how people actually adjudicate trade-offs \textit{between} values when faced with dilemmas. Our response model assumes people can linearly make trade-offs between priorities when the trade-offs are sufficiently weak (i.e., when the decisiveness is high), but also allows certain trade-offs to be irreducible, when both sides are strongly-held or there is no evidence in either direction.

There is considerable support in the literature for these modeling decisions, as well as robust challenges. We discuss this competing evidence in the next two subsections, and where relevant we discuss how countervailing theories can be captured via alternative instantiations of our model.

\subsubsection{Differentiation} 
Another grounding example of value differentiation is in Moral Foundations Theory, which conjectures that people are influenced by multiple moral ``foundations'' (roughly, concerns -- like harm, fairness, loyalty, respect, etc.). There is empirical evidence applying this theory, finding that moral judgments cluster into multiple dissociable dimensions rather than collapsing to a single factor
\cite{graham2011mapping}. Likewise, in the literature exploring how people respond to moral dilemmas (e.g., the trolley problem), experiments show that variation in responses is explained by participants drawing upon multiple priorities, including sensitivity to both consequences and moral norms \cite{gawronski2017consequences}.

On the other hand, even if many values can matter to an individual in principle, the individual may not be able to \textit{retrieve} all of them when making any given judgment, especially in settings where there is low accountability or low deliberative demand. Tetlock (building on prior work) explicitly frames people as ``cognitive misers,'' becoming more complex only when the decision context induces deeper tradeoff reasoning, e.g., due to accountability \cite{tetlock1986value, fiske1984social, tetlock1987accountability}.
Complementarily, intuitionist and dual-process traditions emphasize that judgments are frequently driven by quick intuitions, with reasons supplied post hoc; this predicts limited explicit differentiation unless prompted or socially required \cite{haidt2001emotional,haidt2000dumbfounding, kahneman2011thinking}.

Our model $M$ already captures selective value retrieval to some degree: if a priority is not active on a given query, it does not contribute to the comparison scores $s^+(q)$ and $s^-(q)$, and is effectively ignored in the latent state $L$ and response model $R$. Our model can be extended naturally to capture richer limitations on value retrieval, too. For example, to capture the recognition of only the frames most important to a decision, one could slightly modify our instantiation of $M$ and $L$ and define $s^+(q)$ and $s^-(q)$ to account only for active priorities $j$ \textit{with $\omega_j$ above a certain threshold}, or the $k$ priorities with the largest values of $\omega_j$.

\subsubsection{Integration}
Our model of integration appears primarily in two places, our latent states, and our query response model. Secondarily, our notion of regret --- formulated primarily for convenience of evaluating rules rather than as a behavioral assumption --- does also encode a form of integration, taking the ``best rule'' to be the rule with the highest value of the $\omega$-weighted sum of priority utilities.

\textbf{Conflict and indifference when facing dilemmas.} Our model's permission of $\bowtie$ (conflict) versus $\sim$ (indifference) in both latent states and responses is closely aligned with the work of \citet{CacioppoBerntson1994EvaluativeSpace} on \textit{evaluative space}, as we discuss in \Cref{rem:behavioral-sciences}. In short, they distinguish conflict and indifference in a conceptually similar way --- the distinction resulting from decoupling positive and negative evidence.

Our explicit modeling of conflict --- rather than just indifference, which would indicate no strong reaction --- is also consistent with other research showing that conflict is often experienced as ambivalence and yields discomfort, deferral, and decision avoidance \cite{VanHarreveldVanDerPligtDeLiver2009MAID,SchneiderSchwarz2017MixedFeelings,Anderson2003DecisionAvoidance}. Classic experiments show that conflict is behaviorally relevant, increasing the tendency to defer choice, to choose a default/no-choice option, or to search for additional alternatives or information (i.e., to expand the option set rather than force a premature commitment) \cite{TverskyShafir1992ChoiceUnderConflict,Dhar1997NoChoiceOption}.

\textbf{Compensatory value integration.}
Although our model permits indecision when $s^+(q)$ and $s^-(q)$ are both very high or both very low, over the remaining regimes of these scores, trade-offs \textit{can} be adjudicated by compensatory aggregation: in such cases, the winning alternative is simply the one with the higher directional evidence score. This captures the idea that individuals can \textit{to some degree} integrate considerations in a compensatory way. In the special case of zero thresholds and perfect separability, this reduces exactly to compensatory aggregation --- simply deciding based on the weighted utility-gap comparison. 

The assumption that people can compensatorily aggregate is standard in normative and empirical decision theory: multiattribute utility theory explicitly represents preferences over alternatives with multiple conflicting objectives via additive (or additively separable) value functions with attribute weights \cite{keeneyraiffa1976}. In moral psychology, Guzm{\'a}n et al.'s ``moral trade-off system'' similarly models judgments as selecting the most ``right'' option among feasible alternatives, where rightness is determined by a weighted function over moral values; they argue this framework captures coherent compromise in dilemma-like settings \cite{guzman2022moral}. 

\textbf{Non-compensatory value integration.}
At the same time, influential alternative theories emphasize that moral integration is often \textit{non-compensatory}: rather than continuously trading off all considerations via weights, agents may treat some constraints as deontological prohibitions, ``protected values,'' or hard side-constraints that behave like vetoes or lexicographic priorities \cite{baronspranca1997, greene2008secret}.
Tetlock's work on \textit{sacred values} and \textit{taboo trade-offs} also presents a version of this idea by arguing that, in some domains, even reasoning about certain trade-offs is socially and psychologically prohibited, because certain values are inviolable \cite{tetlock2003thinking}. Additionally, \citet{cushman2012finding} make the case that the moral domain is fundamentally different than other motivational domains due to the non-negotiability over certain values, like not harming others -- leading to compromise intractability. 

While we allow non-compensatory value integration through conflict, our model does not represent taboo/non-negotiable regions as categorically different objects, instead treating all priorities as commensurable objects of identical structure. While high weights are insufficient to guarantee constraint behavior, one can approximate lexicographic reasoning by setting the priority weights to be $\omega = (1-\sum_{j \in [m-1]}\epsilon^j,\epsilon,\epsilon^2,...,\epsilon^{m-1})$ (when the relevant nonzero utility gaps between rules are bounded away from zero, and with \(\epsilon\) chosen sufficiently small relative to that gap), decreasing in the lexicographic order one wants to achieve. Of course, one can always instead capture lexicographic reasoning or taboo trade-offs in the response model $R$, the former of which we do in \Cref{def:lex}; however, this is a non-neutral choice, encoding the assumption that  this non-negotiability is just behavioral, and not part of one's fundamental commitments as modeled by $M$. 

\subsubsection{Separating underlying priorities and contextual response behavior}

The preceding discussion examines how agents construct beliefs from priorities. An additional challenge, documented across survey methodology and behavioral decision research, is that the judgments we elicit are often the product of interactions between such internal reasoning \textit{and the elicitation environment}. By defining $M$ and $R$ separately, our model explicitly decouples these two elements, making clear the distinction between the individual's underlying commitments, as represented by $M$, and the context- or procedure-dependent process through which those commitments are expressed in observed responses, as represented by $R$. This also decouples the goals of learning a behavioral emulator versus learning the individual's underlying beliefs, \textit{despite} behavioral distortions.

There are some documented sources of contextual variation that could be particularly interesting to capture in extensions of our model using the distinction between $M$ and $R$. First, how agents integrate values can vary across elicitation settings: work on constructed preferences suggests that effective trade-off weights may not be invariant across elicitation formats \cite{slovic1995construction,tversky1990causes}. In our terms, such variation could reflect changes in the commitments represented by $M$, changes in the response process $R$, or both. Second, differentiation --- i.e., \textit{which} considerations enter into a judgment --- can also be context-sensitive: framing and task demands can shift which attributes are attended to or treated as relevant \cite{TverskyKahneman1981Framing,payne1993adaptive,tversky1972elimination,tversky1974judgment}. Finally, underlying values and priorities can evolve through learning, persuasion, or deliberation \cite{pettycacioppo1986,fishkinluskin2005}.

\subsection{Proof of \Cref{thm:score-rum-equivalence}} \label{app:score-rum-equivalence}

We first derive a simplification that will be used in both directions. 
\begin{lemma} \label{lem:lin-agg-gap}
    Fix a
perfectly separable pluralistic model $M \in \Msep$ with constant utility gap $\Delta^F_j(q) = \delta_j(q) \ \forall F \in \rules$ for any priority $j$ and query $q$. Then (dropping $M$ from the notation for clarity),
    \[s^+(q)-s^-(q)=
    \sum_{j\in[m]}\omega_j\delta_j(q).\]
\end{lemma}
\begin{proof}
By perfect separability, it holds that
\[
    \Delta_j^F(y,y';x)
    =
    u_j(F_{x\to y})-u_j(F_{x\to y'})
    =
    \delta_j(q)
    \qquad \forall F\in\mathcal F.
\]
Likewise,
\[
    \Delta_j^F(y',y;x)
    =
    u_j(F_{x\to y'})-u_j(F_{x\to y})
    =
    -\delta_j(q)
    \qquad \forall F\in\mathcal F.
\]
Therefore, by unanimity of $\phi^{\mathrm{rules}}$, defining shorthand $s_j^+,s_j^-$,
\[
    s_j^+(q)
    =
    \phi^{\mathrm{rules}}
    \left(
        \{\Delta_j^F(y,y';x)\}_{F\in\mathcal F}
    \right)
    =
    \delta_j(q),
\]
and similarly,
\[
    s_j^-(q)
    =
    \phi^{\mathrm{rules}}
    \left(
        \{\Delta_j^F(y',y;x)\}_{F\in\mathcal F}
    \right)
    =
    -\delta_j(q).
\]
Then, aggregating over priorities,
\[
    s^+(q)
    =
    \sum_{j\in[m]}\omega_j
    \bigl(s_j^+(q)\bigr)_+
    =
    \sum_{j\in[m]}\omega_j
    \bigl(\delta_j(q)\bigr)_+,
\]
while
\[
    s^-(q)
    =
    \sum_{j\in[m]}\omega_j
    \bigl(s_j^-(q)\bigr)_+
    =
    \sum_{j\in[m]}\omega_j
    \bigl(-\delta_j(q)\bigr)_+.
\]
Hence,
\label{eq:simplified-scores}
\begin{align*}
    s^+(q)-s^-(q)=
    \sum_{j\in[m]}\omega_j
    \left[
        \bigl(\delta_j(q)\bigr)_+
        -
        \bigl(-\delta_j(q)\bigr)_+
    \right]  =
    \sum_{j\in[m]}\omega_j\delta_j(q),
\end{align*}
where the final equality uses $(a)_+-(-a)_+=a$.
\end{proof}

\begin{proof}[Proof of \Cref{thm:score-rum-equivalence}]
\textbf{Forward direction:} Fix any local score function
$V:X\times Y\to\mathbb R$. Construct a priority model $\cM_V$
with a single priority, weight $1$, and utility function
\[
    u_V(F):=\sum_{x\in X} V(x,F(x)).
\]
For any query $q=(y,y';x)$ and any background rule $F\in\mathcal F$, the
projected rules $F_{x\to y}$ and $F_{x\to y'}$ agree at every input
$x'\neq x$. Therefore,
\[
\begin{aligned}
    \Delta_V^F(y,y';x)=
    u_V(F_{x\to y})-u_V(F_{x\to y'})  =
    V(x,y)-V(x,y').
\end{aligned}
\]
This expression is independent of $F$, so $\cM_V$ is perfectly separable with constant gap
\[
    \delta_V(q)=V(x,y)-V(x,y').
\]
Since $\cM_V$ has one priority with weight $1$, \Cref{lem:lin-agg-gap}
implies
\[
    s^+(q)-s^-(q)
    =
    \delta_V(q)
    =
    V(x,y)-V(x,y').
\]
Note that when $\tau_r = \tau_\kappa = 0$, it is guaranteed that $r(q) \geq \tau_r$. By the above, $\kappa(q) = s^+(q) - s^-(q) = V(x,y) - V(x,y')$. Then, by \Cref{lem:zero-thresh}, it holds that
\[
\begin{aligned}
    \Pr[R_{h_\beta;\,0,0}(q;M_V)=y\succ y']
    &=
    h_\beta\!\left(
        s^+(q)-s^-(q)
    \right)  \\
    &=
     h_\beta\!\left(V(x,y)-V(x,y')\right)  \\
    &=
    \Pr[S_{h_\beta}(q;V)=y\succ y'].
\end{aligned}
\]
The complementary response probability is therefore also identical.
The learning target correspondence holds by the fact that the following (letting $\rules' \subseteq \rules$ be any subset of rules):
\begin{align*}
    \sum_{x \in \cX} V(x,F(x)) = u_V(F) = \sum_{j \in [m]} \omega_j u_j \quad \forall F \in \rules \implies \arg\max_{F \in \rules'}  \sum_{x \in \cX} V(x,F(x)) = \arg\max_{F \in \rules'} U_M(F).
\end{align*}
 
\textbf{Reverse direction:} A perfectly separable priority model $M$, and an arbitrary background rule $F^0 \in \rules$. For each priority $j$ and each
input $x$, define
\[
    V_j(x,y):=u_j(F^0_{x\to y}).
\]
Then, define the local score function as
\[
    V_M(x,y):=\sum_{j\in[m]}\omega_j V_j(x,y).
\]
We will show that this Score-RUM induces the same score gap as $M$ on
every query. Fix any query $q=(y,y';x)$. Since priority $j$ is perfectly separable at $q$,
there exists a constant $\delta_j(q)$ such that
\[
    u_j(F_{x\to y})-u_j(F_{x\to y'})
    =
    \delta_j(q)
    \qquad \forall F\in\mathcal F.
\]
Applying this equality to the particular background rule $F^0$ gives
\[
\begin{aligned}
    V_j(x,y)-V_j(x,y') 
    &=
    u_j(F^0_{x\to y})-u_j(F^0_{x\to y'}) = \delta_j(q).
\end{aligned}
\]
Therefore,
\[
\begin{aligned}
    V_M(x,y)-V_M(x,y') = 
    \sum_{j\in[m]}\omega_j
    \left(V_j(x,y)-V_j(x,y')\right) =
    \sum_{j\in[m]}\omega_j\delta_j(q) =  s^+(q)-s^-(q).
\end{aligned}
\]
where the last step is by \Cref{lem:lin-agg-gap}.
Hence,
\[
    V_M(x,y)-V_M(x,y')
    =
    s^+(q)-s^-(q).
\]
This equality implies behavioral equivalence exactly as it did in the forward direction. 

\textit{Learning target correspondence.} We will show the claim that for every $j$, there exists some $c_j \in \mathbb{R}$ such that the following is true.
\begin{align}
    \sum_{x \in \cX} V_{M} (x,F(x)) = \sum_{j\in[m]}  \sum_{x \in \cX} \omega_j u_j(F^0_{x\to F(x)})  = \sum_{j \in [m]} c_j + \omega_ju_j(F) \qquad \forall F \in \rules. 
\end{align}
Which will implied the desired claim because the constant shift of $\sum_j c_j$ doesn't affect the optimization.

We will construct the $c_j$ so that for all $j$ and for any $F \in \rules$,
\[ \sum_{x \in \cX} u_j(F^0_{x \to F(x)}) = u_j(F) + c_j/\omega_j.\]
In words, we are saying that the utility of $F$ can be constructed by starting from rule $F^0$, and then adding up all the utility gaps for each individual output replacement at the $x$'s where $F$ and $F^0$ differ, plus some constant shift. Naturally, this will be a telescoping sum argument.

Fix any $F \in \rules$, and let $\{x_1,\dots,x_\ell\} \subseteq \cX$ be the subset of inputs on which $F^0$ and $F$ differ, enumerated arbitrarily (if this set is empty, we are done). Construct a sequence of rules $F^0,F^1,\dots,F^\ell$, where \(F^\ell=F\), and \(F^t\) agrees with \(F\) on
\(\{x_1,\ldots,x_t\}\) and $F^0$ on $\{x_{t+1}, \dots, x_\ell\}$. 
Thus \(F^t\) is obtained from \(F^{t-1}\) by
changing only the value at input \(x_t\) from \(F^0(x_t)\) to \(F(x_t)\).

By telescoping,
\[
    u_j(F)-u_j(F^0)
    =
    \sum_{t=1}^\ell
    \left(
        u_j(F^t)-u_j(F^{t-1})
    \right).
\]
For each \(t\), perfect separability implies that the effect of changing the
value at \(x_t\) from \(F^0(x_t)\) to \(F(x_t)\) is independent of the
background rule. Hence, we can equivalently write this utility gap as the one obtained by modifying $F^0$ at that individual input:
\[
    u_j(F^t)-u_j(F^{t-1})
    =
    u_j(F^0_{x_t\to F(x_t)})-u_j(F^0).
\]
Therefore,
\[
\begin{aligned}
    u_j(F)-u_j(F^0)
    =
    \sum_{t=1}^\ell
    \left(
        u_j(F^0_{x_t\to F(x_t)})-u_j(F^0)
    \right) =
    \sum_{x\in\mathcal X}u_j(F^0_{x\to F(x)})
    -
    |\cX| u_j(F^0).
\end{aligned}
\]
Rearranging gives
\[
    \sum_{x\in\mathcal X}u_j(F^0_{x\to F(x)})
    =
    u_j(F)+(|\cX|-1)u_j(F^0).
\]
Thus, we set $c_j = (|\cX|-1)u_j(F^0)\omega_j$, which is a constant in $F$ as desired, and we conclude the desired equality.
\[
    \sum_{x\in\mathcal X} V_j(x,F(x))
    =
    u_j(F)+c_j/\omega_j.
\]
Putting it all together, we get that
\[
\begin{aligned}
    \sum_{x\in\mathcal X} V_{M}(x,F(x)) = \sum_{j\in[m]}\omega_j u_j(F)
    +
    \sum_{j\in[m]}\omega_j c_j,
\end{aligned}
\]
where the final term is constant in $F$, and thus does not affect the optimization. We conclude that for any $\rules' \subseteq \rules$,
\[
    \arg\max_{F \in \mathcal{F}'} \sum_{x\in\mathcal X} V_{M}(x,F(x)) =  \arg\max_{F \in \mathcal{F}'} \sum_{j\in[m]}\omega_j u_j(F). \qedhere
\]
\end{proof}

\newpage
\section{Supplemental Materials from \Cref{sec:insep}}
\subsection{Proof of \Cref{prop:egal-insep}} \label{app:egal-insep}
\begin{proof}
Fix any two \(a,b \in N\) where \(a \neq b\), and fix an arbitrary query
\(q = (a,b;(a,b;k))\), where \(x^*= (a,b;k)\). To prove inseparability,
we must simply construct two background rules \(F^a\), \(F^b\) such that
\[
F^a_{x^* \to a} \succ_{\mathrm{egal}} F^a_{x^* \to b}
\qquad \text{and} \qquad
F^b_{x^* \to b} \succ_{\mathrm{egal}} F^b_{x^* \to a},
\]
i.e., who we allocate to at $x^*$ according to the egalitarian priority depends on the background rule.
We will prove that we can construct such an \(F^a\); by symmetry, the same
proof demonstrates that we can construct such an \(F^b\).

Let \(F^a\) be any rule that never allocates to \(a\) at any other
\(x \in \cX \setminus \{x^*\}\), but does allocate to all
\(c \in N \setminus \{a\}\) at least once. This is possible because
\(|N| \geq 3\): in all other queries, either \(a\) does not appear, or
\(a\) does appear but \(F^a\) can allocate to another recipient.

By the assumption that \(\cD\) has full support, i.e., every \(x\) occurs
with positive probability, it must be that under either projection
\(F^a_{x^* \to a}\) or \(F^a_{x^* \to b}\), all other recipients receive
nonzero benefit:
\[
\mathbb{E}_{x \sim \cD}
\left[
v_c\left(x,F^a_{x^* \to j}(x)\right)
\right] > 0
\qquad
\text{for all } c \in N \setminus \{a\}, \ j \in \{a,b\}.
\]
However, because \(a\) never receives a good at any \(x \neq x^*\), whether
\(a\)'s benefit is \(0\) is entirely dictated by \(F^a\)'s behavior at
\(x^*\):
\[
\mathbb{E}_{x \sim \cD}
\left[
v_a\left(x,F^a_{x^* \to b}(x)\right)
\right] = 0
\qquad \text{and} \qquad
\mathbb{E}_{x \sim \cD}
\left[
v_a\left(x,F^a_{x^* \to a}(x)\right)
\right] > 0.
\]
This means that \(F^a\) makes \(a\) the most shortchanged recipient at the
query, and thus whether they receive the good at \(x^*\) is binding for the
egalitarian priority:
\[
u_{\mathrm{egal}}\left(F^a_{x^* \to b}\right) = 0
\qquad \text{and} \qquad
u_{\mathrm{egal}}\left(F^a_{x^* \to a}\right) > 0.
\]
This implies the following, as needed:
\[
F^a_{x^* \to a} \succ_{\mathrm{egal}} F^a_{x^* \to b}.
\qedhere
\]
\end{proof}

\subsection{Perfect Inseparability of Proportionality and Equal Treatment} \label{app:prop-eq}
\begin{proposition}\label{prop:prop} For certain $\cX,\cY$, \textbf{Proportionality} can be perfectly 
    inseparable on all local pairwise comparison queries $q \in \cQ^{\text{pc}}$.
\end{proposition}
\begin{proof}
    Consider the example from the proof of \Cref{prop:egal}, where there are two groups: group 1 is $\{a\}$, group 2 is $\{b\}$ and $\alpha_1 = \alpha_2 = 1/2$. Then,
    \[u_\text{prop}(F^{ab}) = u_\text{prop}(F^{ba}) = 0, \quad u_\text{prop}(F^{aa}) = u_\text{prop}(F^{bb}) = -1/2.\]
    These utilities have exactly the same structure as the utilities in \Cref{prop:egal}, creating exactly the same implementation of perfect inseparability (\Cref{def:insep}) with $\delta = 1/2$.
\end{proof}

\begin{proposition} \label{prop:eq}
    For certain $\cX,\cY$, \textbf{Equal Treatment} can be perfectly 
    inseparable on all local pairwise comparison queries $q \in \cQ^{\text{pc}}$.
\end{proposition}

\begin{proof}
Fix three recipients \(N=\{a,a',b\}\), where \(a\in G_1\), \(a'\in G_2\),
and \(a,a'\) are otherwise identical. Recipient \(b\) is not the protected
counterpart of either \(a\) or \(a'\). Consider an allocation task with two
inputs,
\[
\cX=\{x^{ab},x^{a'b}\},
\]
where
\[
\cY(x^{ab})=\{a,b\}
\qquad\text{and}\qquad
\cY(x^{a'b})=\{a',b\}.
\]
The inputs \(x^{ab}\) and \(x^{a'b}\) are counterparts: the former contains
recipient \(a\), while the latter replaces \(a\) by its counterpart \(a'\),
leaving \(b\) unchanged.

Given $G_1,G_2$, the equal-treatment priority is then defined by
\[
u_{\mathrm{eq}}(F)
=
\mathbf 1
\left\{
\mathbf 1\{F(x^{ab})=b\}
=
\mathbf 1\{F(x^{a'b})=b\}
\right\}.
\]
Thus the priority is satisfied exactly when the rule either chooses \(a\)
at \(x^{ab}\) and \(a'\) at \(x^{a'b}\), or chooses \(b\) at both inputs.

There are four deterministic rules, where the superscript records the outputs at \(x^{ab}\) and \(x^{a'b}\),
respectively:
\[
\rules=\{F^{aa'},F^{ab},F^{ba'},F^{bb}\},
\]
Hence,
\[
u_{\mathrm{eq}}(F^{aa'})
=
u_{\mathrm{eq}}(F^{bb})
=
1,
\qquad
u_{\mathrm{eq}}(F^{ab})
=
u_{\mathrm{eq}}(F^{ba'})
=
0.
\]

We first consider the query
\[
q_1=(a,b;x^{ab}).
\]
For any background rule \(F\), the value of the projection at \(x^{ab}\)
depends only on the rule's output at the counterpart input \(x^{a'b}\). If
\(F(x^{a'b})=a'\), then choosing \(a\) at \(x^{ab}\) satisfies equal
treatment, while choosing \(b\) violates it. Thus
\[
u_{\mathrm{eq}}(F_{x^{ab}\to a})
-
u_{\mathrm{eq}}(F_{x^{ab}\to b})
=
1.
\]
If \(F(x^{a'b})=b\), then choosing \(b\) at \(x^{ab}\) satisfies equal
treatment, while choosing \(a\) violates it. Thus
\[
u_{\mathrm{eq}}(F_{x^{ab}\to a})
-
u_{\mathrm{eq}}(F_{x^{ab}\to b})
=
-1.
\]
Therefore, applying \Cref{def:insep}, this priority is perfectly inseparable at \(q_1\) with \(\delta=1\) and the
equal-size partition
\[
\rules_a=\{F^{aa'},F^{ba'}\},
\qquad
\rules_b=\{F^{ab},F^{bb}\}.
\]

The argument for the query
\[
q_2=(a',b;x^{a'b})
\]
is symmetric. If \(F(x^{ab})=a\), then choosing \(a'\) at \(x^{a'b}\)
satisfies equal treatment, while choosing \(b\) violates it. If
\(F(x^{ab})=b\), then choosing \(b\) satisfies equal treatment, while
choosing \(a'\) violates it. Hence \(q_2\) has \(\delta=1\) and the
equal-size partition
\[
\rules_{a'}=\{F^{aa'},F^{ab}\},
\qquad
\rules_b=\{F^{ba'},F^{bb}\}.
\]
Thus equal treatment is perfectly inseparable at every local pairwise
comparison query in this instance.
\end{proof}

\subsection{Proof of \Cref{thm:scale-free-ind}} \label{app:scale-free-ind}
We begin by stating a useful lemma:
\begin{lemma}
\label{lem:insep-silent}
Suppose priority $j$ is perfectly inseparable at $q = (y, y';\, x)$.  Let
$\phi^{\text{rules}}$ be permutation-invariant; then $j$ contributes  equal evidence in both directions:
\[
    \prules(\{\Delta_j^F(y,y';x)\}_{F \in \mathcal{F}}) = \prules(\{\Delta_j^F(y',y;x)\}_{F \in \mathcal{F}}).
\]
\end{lemma}
\begin{proof}
By perfect inseparability, the evidence multisets in the two directions are
\begin{align*}
    \left\{\Delta^{F}_{j}(y, y';\, x)\right\}_{F \in \mathcal{F}}
    &\;=\;
    \{\underbrace{\delta, \ldots, \delta}_{|\mathcal{F}_y|},\;
      \underbrace{-\delta, \ldots, -\delta}_{|\mathcal{F}_{y'}|}\}, \qquad \left\{\Delta^{F}_{j}(y', y;\, x)\right\}_{F \in \mathcal{F}}
    \;=\;
    \{\underbrace{-\delta, \ldots, -\delta}_{|\mathcal{F}_y|},\;
      \underbrace{\delta, \ldots, \delta}_{|\mathcal{F}_{y'}|}\}.
\end{align*}
Since $|\mathcal{F}_y| = |\mathcal{F}_{y'}|$, these two multisets are identical, so the permutation-invariance of $\prules$ implies the claim.
\end{proof}

\textbf{General priority aggregators.} We will prove the claim for more generic priority aggregators here, so now we define them formally. 
A \textit{priority aggregator} is a map
\[
\ppriors:\ \mathbb{R}^{m}\times\Delta^{m-1}\ \longrightarrow\ \mathbb{R},
\]
which takes in a vector of $m$ per-priority  evidence values (the output of $\prules$ on every priority) and a weight
vector and outputs a single nonnegative directional score. Formally, given a rule aggregator
$\phi^{\mathrm{rules}}$, a model $M=(u,\omega)$, and a query $q=(y,y';x)$, define
the per-priority rule-aggregated evidence in each direction as
\[
s^{+}_j(q)\ :=\ \phi^{\mathrm{rules}}\!\big((\Delta^{F}_j(y,y';x))_{F\in\rules}\big),
\qquad
s^{-}_j(q)\ :=\ \phi^{\mathrm{rules}}\!\big((\Delta^{F}_j(y',y;x))_{F\in\rules}\big),
\]
and collect $\mathbf{s}^{\pm}(q):=(s^{\pm}_j(q))_{j\in[m]}$. The two directional
evidence scores are obtained by applying $\ppriors$ in each direction:
\[
s^{+}_M(q)\ :=\ \ppriors\!\big(\mathbf{s}^{+}(q),\,\omega\big),
\qquad
s^{-}_M(q)\ :=\ \ppriors\!\big(\mathbf{s}^{-}(q),\,\omega\big).
\]
Note that this construction is a generalization of the linear aggregator, in which $\ppriors$ is linear in $\omega$.

Now we define the more general class of priority aggregators for which the theorem  will hold:
\begin{definition}[Scale-preservation] \label{def:scale-pres}
A priority aggregator $\ppriors$ satisfies \textit{scale-preservation} iff, for every
priority $j^*$ and every weight vector $\omega\in\Delta^{m-1}$, there exist a constant
$c(\omega,j^*)>0$ and a function $A_{\omega,j^*}:\mathbb R\to\mathbb R$ such that, for every
evidence vector $\mathbf{z}\in\mathbb R^m$,
\[
\ppriors(z,\omega)
=
c(\omega,j^*)\,
\ppriors(\mathbf{z}_{-j^*},\omega_{-j^*})
+
A_{\omega,j^*}(z_{j^*}).
\]
In words, the effect of a priority can be separated from the aggregate contribution
of the remaining priorities: adding priority $j^*$ may add its own evidence term and may even
rescale the contribution of all other priorities, but it cannot change how the remaining priorities
trade off against one another. Note that by linearity, the linear priority aggregator satisfies this criterion.
\end{definition}

Now, we prove the theorem for scale-preserving priority aggregators.

\begin{proof}[Proof of \Cref{thm:scale-free-ind}]
We first prove the claim for the removal of a single perfectly inseparable
priority \(j^*\in J_{\mathrm{insep}}\). The result for the full set
\(J_{\mathrm{insep}}\) follows by iterating the same argument.

Fix any local pairwise query \(q=(y,y';x)\in \cQ^{pc}\). For each priority
\(j\), define shorthand for the output of $\prules$:
\[
s^+_j(q)
:=
\prules\!\left(
\{\Delta^F_j(y,y';x)\}_{F\in\mathcal F}
\right),
\qquad
s^-_j(q)
:=
\prules\!\left(
\{\Delta^F_j(y',y;x)\}_{F\in\mathcal F}
\right).
\]
Summarize these values in
\[
\mathbf s^+(q):=(s^+_j(q))_{j\in J},
\qquad
\mathbf s^-(q):=(s^-_j(q))_{j\in J}.
\]
Then the two directional evidence scores induced by \(M_{\mathrm{insep}}\) are
\[
s_{M_{\mathrm{insep}}}(y,y';x)
=
\ppriors(\mathbf s^+(q),\omega),
\qquad
s_{M_{\mathrm{insep}}}(y',y;x)
=
\ppriors(\mathbf s^-(q),\omega).
\]

By scale preservation of \(\ppriors\), there exist a constant
\(c(\omega,j^*)>0\) and a function \(A_{\omega,j^*}:\mathbb R\to\mathbb R\)
such that, writing \(c=c(\omega,j^*)\) and \(A=A_{\omega,j^*}\),
\[
\ppriors(\mathbf s^+(q),\omega)
=
c\,
\ppriors(\mathbf s^+_{-j^*}(q),\omega_{-j^*})
+
A(s^+_{j^*}(q)),
\]
and
\[
\ppriors(\mathbf s^-(q),\omega)
=
c\,
\ppriors(\mathbf s^-_{-j^*}(q),\omega_{-j^*})
+
A(s^-_{j^*}(q)).
\]
By \Cref{lem:insep-silent}, since \(j^*\) is perfectly inseparable at \(q\),
\[
s^+_{j^*}(q)=s^-_{j^*}(q).
\]
Thus the two \(A(\cdot)\) terms cancel from the score gap, giving
\begin{align*}
&s_{M_{\mathrm{insep}}}(y,y';x)
-
s_{M_{\mathrm{insep}}}(y',y;x)
=
c\left(
\ppriors(\mathbf s^+_{-j^*}(q),\omega_{-j^*})
-
\ppriors(\mathbf s^-_{-j^*}(q),\omega_{-j^*})
\right).
\end{align*}
The expression in parentheses is exactly the score gap induced by the model
\(M^{-j^*}\). Hence, for every local pairwise query \(q\),
\[
s_{M_{\mathrm{insep}}}(y,y';x)
-
s_{M_{\mathrm{insep}}}(y',y;x)
=
c\left(
s_{M^{-j^*}}(y,y';x)
-
s_{M^{-j^*}}(y',y;x)
\right).
\]

Now fix any inverse-temperature parameter \(\beta>0\) for
\(M_{\mathrm{insep}}\), and define
\[
\beta' := \beta \, c.
\]
Since \(c >0\), we have \(\beta'>0\). Therefore, for any $h$ and every local
pairwise query \(q=(y,y';x)\),
\begin{align*}
R^\circ_{h_\beta;0,0}(q;M_{\mathrm{insep}}) (y\succ y')
&=
h_\beta\!\left(
s_{M_{\mathrm{insep}}}(y,y';x)
-
s_{M_{\mathrm{insep}}}(y',y;x)
\right)
\\
&=
h\!\left(
\beta c
\left[
s_{M^{-j^*}}(y,y';x)
-
s_{M^{-j^*}}(y',y;x)
\right]
\right)
\\
&=
h_{\beta'}\!\left(
s_{M^{-j^*}}(y,y';x)
-
s_{M^{-j^*}}(y',y;x)
\right)
\\
&=
R^\circ_{h_{\beta'};0,0}(q;M^{-j^*})(y\succ y').
\end{align*}
Under the zero-threshold response model, the remaining probability mass is
assigned to the complementary response \(y'\succ y\), so the full response
distributions are identical at every local pairwise query. Hence
\(M_{\mathrm{insep}}\) and \(M^{-j^*}\) are scale-free indistinguishable.

Finally, repeat the argument for each priority in \(J_{\mathrm{insep}}\).
This yields a positive constant \(C>0\) such that, for every local pairwise
query \(q=(y,y';x)\),
\[
s_{M_{\mathrm{insep}}}(y,y';x)
-
s_{M_{\mathrm{insep}}}(y',y;x)
=
C\left(
s_{M}(y,y';x)
-
s_{M}(y',y;x)
\right),
\]
where \(M=M^{-J_{\mathrm{insep}}}\). Taking \(\beta''=\beta C\), we obtain
\[
R^\circ_{h_\beta;0,0}(q; M_{\mathrm{insep}})
\equiv
R^\circ_{h_{\beta''};0,0}(q;M)
\qquad
\forall q\in\cQ^{pc}.
\]
Thus \(M_{\mathrm{insep}}\) and \(M\) are scale-free indistinguishable.
\end{proof}

\subsection{Formalization of \Cref{ex:bad-rec}} \label{app:bad-rec}
Consider again
the two-input allocation task from \Cref{prop:egal}: there are two recipients
\(N=\{a,b\}\), two inputs \(x_1=(a,b;k_1)\) and \(x_2=(a,b;k_2)\), and
\(\cY(x_1)=\cY(x_2)=\{a,b\}\). $\cD$ is such that both inputs occur with equal probability, and suppose both recipients have the same benefit for either good, and no benefit if they aren't given a good:
\[v_i(x,j) = \mathbf{1}\{j = i\} \qquad \text{for all } \ i,j \in \{a,b\}.\] As usual, write $\rules=\{F^{aa},F^{ab},F^{ba},F^{bb}\},$ where \(F^{ij}(x_1)=i\) and \(F^{ij}(x_2)=j\). 

Now, we let \(M_{\mathrm{insep}}\) have two priorities. The first is
\textit{Egalitarianism} (\Cref{def:egal}), which we know is perfectly inseparable in this example. The second is the
\textit{Family} priority (\Cref{ex:trolley-priorities}), which rewards giving the good to family members. Letting \(a\) be a family member and $b$ be a non-family member, the utility function for the family priority becomes 
\[
u_{\mathrm{family}}(F)
=
\mathbb E_{x\sim\cD}
\left[
\mathbf 1\{F(x)=a\}
\right].
\]
Note that the family priority is perfectly separable.\footnote{For any query \(q=(y,y';x)\), $\Delta^F_{\mathrm{family}}(q)=
\cD(x)\left(\mathbf 1\{y=a\}-\mathbf 1\{y'=a\}\right)$
is independent of the background rule \(F\).} Then, by \Cref{cor:erase}, the reduced model $M_{\mathrm{family}}=(M_{\mathrm{insep}})^{-\mathrm{egal}}$, which consists only of the family priority with weight \(1\), is a perfect
separable rationalization of \(M_{\mathrm{insep}}\). Thus, the set of perfectly separable rationalizations is non-empty, and on an exhaustive
transcript generated by \(R^\circ_{h_\beta;0,0}(\cdot;M_{\mathrm{insep}})\), any separable-consistent
decoder returns some model $\widetilde M\in \mathrm{PSR}(M_{\mathrm{insep}}).$

Finally, suppose that the system outputs an aggregate-optimal rule
according to \(\widetilde M\), which must coincide with the aggregate-optimal rules for $M_{\mathrm{family}}$ (\Cref{lem:psr-same-optimal-rules}).\footnote{By \Cref{lem:psr-same-optimal-rules} (below), every
perfect separable rationalization of \(M_{\mathrm{insep}}\) induces the same
aggregate-optimal rules.}  Since \(M_{\mathrm{family}}\)
contains only the family priority,
\[
u_{\mathrm{family}}(F^{aa})=1,\qquad
u_{\mathrm{family}}(F^{ab})=u_{\mathrm{family}}(F^{ba})=\frac12,\qquad
u_{\mathrm{family}}(F^{bb})=0.
\]
Therefore the system uniquely selects $F^{aa}$, the rule that always allocates to the family member \(a\). 

When $\omega_{\mathrm{egal}}$ is dominant this rule is suboptimal, and it approaches the worst rule as $\omega_{\mathrm{egal}}$ becomes more dominant. Formally, let $\epsilon \in (0,1/2)$ and let $\omega_{\mathrm{egal}} = 1-\epsilon$ and $\omega_{\mathrm{family}} = \epsilon$; then,
\begin{align*}
    U_{M_{\mathrm{insep}}}(F^{ab})
=
U_{M_{\mathrm{insep}}}(F^{ba})
=
(1-\epsilon)+\epsilon/2
=
1 - \epsilon/2,\ \ \ \ \ \ U_{M_{\mathrm{insep}}}(F^{aa})
=
\epsilon, \ \ \ \ \ \ U_{M_{\mathrm{insep}}}(F^{bb}) = 0.
\end{align*}
Then, the regret is
\[U_{M_{\mathrm{insep}}}(F^{ab}) - U_{M_{\mathrm{insep}}}(F^{aa}) = 1-\epsilon/2 - \epsilon = 1 -3\epsilon/2.\]
As $\epsilon \to 0$, this gets arbitrarily close to the regret of the worst rule, which is
\[U_{M_{\mathrm{insep}}}(F^{ab}) - U_{M_{\mathrm{insep}}}(F^{bb}) = 1-\epsilon/2.\]

\subsection{\Cref{lem:psr-same-optimal-rules}}

\begin{lemma}
\label{lem:psr-same-optimal-rules}
Fix a link function \(h\) and let \(M,M'\in\cM_{\mathrm{sep}}\) be
scale-free indistinguishable w.r.t.~$h$. Then, for any subset of rules $\rules' \subseteq \rules$,
\(\mathcal \rules\subseteq\rules\),
\[
\arg\max_{F\in \rules'} U_M(F)
=
\arg\max_{F\in \rules'} U_{M'}(F).
\]
\end{lemma}

\begin{proof}
Because \(M\) and \(M'\) are perfectly separable,
\Cref{thm:score-rum-equivalence} gives local score functions \(V_M\) and
\(V_{M'}\) such that, for every query \(q=(y,y';x)\),
\[
\kappa_M(q)=V_M(x,y)-V_M(x,y')
\]
and
\[
\kappa_{M'}(q)=V_{M'}(x,y)-V_{M'}(x,y').
\]
Moreover, \Cref{thm:score-rum-equivalence} shows that maximizing \(U_M\) over any subset of rules is equivalent to
maximizing
\[
\sum_{x\in X} V_M(x,F(x)),
\]
and analogously for \(M'\).

Since \(M\) and \(M'\) are scale-free indistinguishable, there exist
\(\beta,\beta'>0\) such that, for every query \(q\),
\[
h(\beta \kappa_M(q))=h(\beta'\kappa_{M'}(q)).
\]
Because \(h\) is strictly increasing, it is injective. Therefore,
\[
\beta \kappa_M(q)=\beta'\kappa_{M'}(q)
\qquad \forall q\in Q^{\mathrm{pc}}.
\]
Let \(c:=\beta/\beta'>0\). Then, for every \(x\in X\) and every
\(y,y'\in \cY(x)\),
\[
V_{M'}(x,y)-V_{M'}(x,y')
=
c\bigl(V_M(x,y)-V_M(x,y')\bigr).
\]
Fix \(x\in X\), and choose an arbitrary reference output \(y_0\in \cY(x)\).
The above equality implies that, for every \(y\in \cY(x)\),
\[
V_{M'}(x,y)-V_{M'}(x,y_0)
=
c\bigl(V_M(x,y)-V_M(x,y_0)\bigr).
\]
Equivalently,
\[
V_{M'}(x,y)-cV_M(x,y)
=
V_{M'}(x,y_0)-cV_M(x,y_0).
\]
Thus there exists a constant \(a_x = V_{M'}(x,y_0)-cV_M(x,y_0)\), depending on \(x\) but not on \(y\),
such that
\[
V_{M'}(x,y)=cV_M(x,y)+a_x
\qquad \forall y\in \cY(x).
\]
Therefore, for every rule \(F\in \mathcal F\),
\[
\sum_{x\in X} V_{M'}(x,F(x))
=
c\sum_{x\in X} V_M(x,F(x))
+
\sum_{x\in X} a_x.
\]
The additive term \(\sum_{x\in X} a_x\) is independent of \(F\), and
\(c>0\). Hence the two local-score objectives have the same maximizers
over every \(\rules'\subseteq \mathcal F\). Translating back through
the perfect-separability correspondence gives
\[
\arg\max_{F\in \rules'} U_M(F)
=
\arg\max_{F\in \rules'} U_{M'}(F).\qedhere
\]
\end{proof}

\subsection{Formalization of \Cref{ex:inseparable-local-trace}} \label{app:insep-misinterpret}
\textit{Setup.} Consider again
the two-input allocation task from \Cref{prop:egal}: there are two recipients
\(N=\{a,b\}\), two inputs \(x_1=(a,b;k_1)\) and \(x_2=(a,b;k_2)\), and
\(\cY(x_1)=\cY(x_2)=\{a,b\}\). $\cD$ is such that both inputs occur with equal probability, and suppose both recipients have the same benefit for either good, and no benefit if they aren't given a good:
\[v_i(x,j) = \mathbf{1}\{j = i\} \qquad \text{for all } \ i,j \in \{a,b\}.\] As usual, write $\rules=\{F^{aa},F^{ab},F^{ba},F^{bb}\},$ where \(F^{ij}(x_1)=i\) and \(F^{ij}(x_2)=j\).

\textit{Additional (mild) restriction on $\prules$.} Assume that \(\prules\) is \textit{balanced sign-responsive}: that is, for every evidence profile \(z\) that is exactly balanced between the two responses, \[ \{z_F:F\in\rules\}=\{-z_F:F\in\rules\}, \] any positive uniform shift at every index of $z$ makes the aggregate evidence positive: for every \(c>0\), \[ \prules(z+c\mathbf 1)>0. \] This condition is satisfied by the average and maximum aggregators, but not by lower-percentile aggregators such as the minimum.

\textit{Model.} Let \(M\) contain the single proportionality priority
\[
u_{\mathrm{prop}}(F)
=
-
\sum_{i\in\{a,b\}}
\left(
\Pr_{X\sim\cD}[F(X)=i]-\alpha_i
\right)^2,
\]
with
\[
\alpha_a=\frac12+\epsilon,
\qquad
\alpha_b=\frac12-\epsilon,
\qquad
0<\epsilon<\frac14.
\]
In words, this priority reflects the intuition that the individual wants the goods to be roughly evenly split, with a slight bias toward individual $a$ (or, conceptually,  individuals of type $a$).

\textit{True rule utilities.} Define the shorthand $p_a(F)=\Pr_{X\sim\cD}[F(X)=a],$ for the probability that $a$ receives a good under $F$. Noting that $p_b(F)=1-p_a(F)$, it follows that
\[
u_{\mathrm{prop}}(F) = -\bigg(\big(p_a(F) - (1/2+\epsilon)\big)^2 + \big(1-p_a(F) - (1/2-\epsilon)\big)^2\bigg) 
=
-2\left(p_a(F)-\left(1/2+\epsilon\right)\right)^2.
\]
Therefore,
\[
u_{\mathrm{prop}}(F^{aa})
=
-2\left(1/2-\epsilon\right)^2
=
-1/2+2\epsilon-2\epsilon^2,
\]
\[
u_{\mathrm{prop}}(F^{bb})
=
-2\left(1/2+\epsilon\right)^2
=
-1/2-2\epsilon-2\epsilon^2,
\]
and
\[
u_{\mathrm{prop}}(F^{ab})
=
u_{\mathrm{prop}}(F^{ba})
=
-2\epsilon^2.
\]

Thus the proportionality priority is maximized by the balanced rules
\(F^{ab}\) and \(F^{ba}\).

\textit{Observed Query Behavior.}
Consider the query \(q_1=(a,b;x_1)\). For each background rule
\(F\in\rules\), define
\[
\Delta^F_{\mathrm{prop}}(q_1)
:=
u_{\mathrm{prop}}(F_{x_1\to a})
-
u_{\mathrm{prop}}(F_{x_1\to b}).
\]
If \(F(x_2)=b\), then \(F_{x_1\to a}=F^{ab}\) and
\(F_{x_1\to b}=F^{bb}\), so
\[
\Delta^F_{\mathrm{prop}}(q_1)
=
u_{\mathrm{prop}}(F^{ab})-u_{\mathrm{prop}}(F^{bb})
=
-2\epsilon^2
-
\left(-\frac12-2\epsilon-2\epsilon^2\right)
=
\frac12+2\epsilon.
\]
If \(F(x_2)=a\), then \(F_{x_1\to a}=F^{aa}\) and
\(F_{x_1\to b}=F^{ba}\), so
\[
\Delta^F_{\mathrm{prop}}(q_1)
=
u_{\mathrm{prop}}(F^{aa})-u_{\mathrm{prop}}(F^{ba})
=
\left(-\frac12+2\epsilon-2\epsilon^2\right)
-
\left(-2\epsilon^2\right)
=
-\frac12+2\epsilon.
\]
Thus the rule-indexed evidence profile at \(q_1\) is, up to permutation,
\[
z^+
=
\left(
\frac12+2\epsilon,\frac12+2\epsilon,
-\frac12+2\epsilon,-\frac12+2\epsilon
\right)
=
z^0+2\epsilon\mathbf 1,
\]
where
\[
z^0=
\left(
\frac12,\frac12,-\frac12,-\frac12
\right)
\]
is balanced around zero. Then, by balanced sign-responsiveness of \(\prules\), $\prules(z^+)>0.$ Since the model has only the proportionality priority with weight \(1\),
the linear priority aggregator gives
\[
s_M^+(q_1)=[\prules(z^+)]_+,
\qquad
s_M^-(q_1)=[-\prules(z^+)]_+.
\]
Therefore,
\[
\kappa_M(q_1)
=
s_M^+(q_1)-s_M^-(q_1)
=
\prules(z^+)
>
0.
\]
The same calculation applies to \(q_2=(a,b;x_2)\). Hence, letting
\[
\gamma:=\prules(z^+)>0,
\]
we have
\[
\kappa_M(a,b;x_i)=\gamma
\qquad
\text{for } i\in\{1,2\}.
\]

Thus, under the zero-threshold response model,
\[
R^\circ_{h_\beta;0,0}\bigl((a,b;x_i);M\bigr)(a\succ b)
=
h(\beta\gamma)
>
h(-\beta\gamma)
=
R^\circ_{h_\beta;0,0}\bigl((a,b;x_i);M\bigr)(b\succ a),
\]
where the strict inequality follows because \(h\) is strictly increasing.
Hence the response model induced by \(M\) favors \(a\) over \(b\) at both
inputs.

\textit{Construction of the perfectly separable rationalization.}
Now define the perfectly separable priority
\[
u_a(F)=\Pr_{X\sim\cD}[F(X)=a].
\]
For every background rule \(F\) and every input \(x_i\),
\[
u_a(F_{x_i\to a})-u_a(F_{x_i\to b})
=
\frac12,
\]
which is independent of \(F\). Thus \(u_a\) is perfectly separable. Let
\(\widetilde M\) be the model consisting of this single priority. By unanimity
of \(\prules\),
\[
\kappa_{\widetilde M}(a,b;x_i)=\frac12
\qquad
\text{and}
\qquad
\kappa_{\widetilde M}(b,a;x_i)=-\frac12
\]
for each \(i\in\{1,2\}\).

We now show that \(M\) and \(\widetilde M\) are scale-free indistinguishable.
Fix any \(\beta>0\), and set
\[
\widetilde\beta:=2\beta\gamma>0.
\]
Then, for each \(i\in\{1,2\}\),
\[
\beta\kappa_M(a,b;x_i)
=
\beta\gamma
=
\widetilde\beta\cdot \frac12
=
\widetilde\beta\kappa_{\widetilde M}(a,b;x_i).
\]
Thus the scaled directional gaps agree on every local pairwise query, and
therefore
\[
R^\circ_{h_\beta;0,0}(\cdot;M)
=
R^\circ_{h_{\widetilde\beta};0,0}(\cdot;\widetilde M).
\]
Hence
\[
\widetilde M\in \mathrm{PSR}_h(M).
\]

\textit{Chosen rule under $\widetilde M$.} Under \(\widetilde M\),
\[
u_a(F^{aa})=1,
\qquad
u_a(F^{ab})=u_a(F^{ba})=\frac12,
\qquad
u_a(F^{bb})=0.
\]
Therefore,
\[
\arg\max_{F\in\rules} U_{\widetilde M}(F)=\{F^{aa}\}.
\]
By \Cref{lem:psr-same-optimal-rules}, under the linear priority aggregator
all elements of \(\mathrm{PSR}_h(M)\) induce the same aggregate-optimal rules.
Hence every perfect separable rationalization of \(M\) selects \(F^{aa}\).

\textit{Regret.} The selected rule $F^{aa}$ is highly suboptimal. Its regret is
\[
U_M(F^{ab}) - U_M(F^{aa})
= -2\epsilon^2 - \left(-\frac12+2\epsilon-2\epsilon^2\right)
= \frac12 - 2\epsilon .
\]
As $\epsilon \to 0$, this becomes arbitrarily close to the regret of the worst possible rule $F^{bb}$:
\[
U_M(F^{ab}) - U_M(F^{bb})
= -2\epsilon^2 - \left(-\frac12-2\epsilon-2\epsilon^2\right)
= \frac12 + 2\epsilon .
\]

\newpage
\section{Supplemental Materials for \Cref{sec:conflict}} \label{app:conflict}

\subsection{Interview Coding Methods and Results}
\label{app:criteria}
\Cref{fig:kidney} summarizes the frequency of various signals of indecision detected in each of
$20$ interview transcripts from \citet{keswani2025can}, in which each participant made exactly $3$
pairwise kidney-allocation comparisons ($60$ comparisons total). Below, we describe our coding methods.
\begin{figure}[h!]
    \centering
    \includegraphics[width=0.8\linewidth]{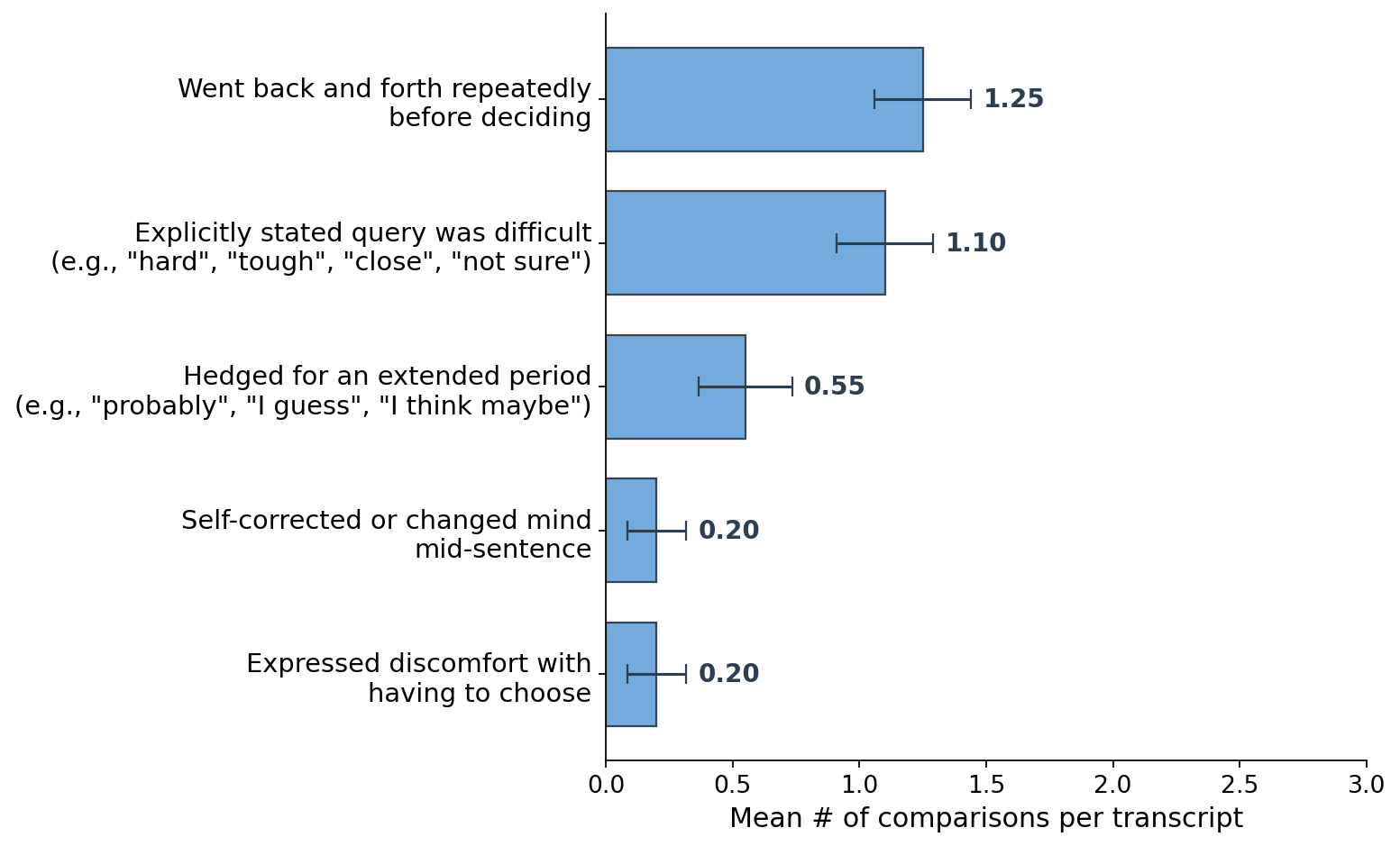}
    \caption{Frequency of language markers of  indecision across $20$ interview transcripts, each
    containing $3$ pairwise kidney-allocation comparisons (60  total comparisons). We count for each transcript how many of its
    $3$ comparisons exhibit each of five conflict criteria, and report the mean per transcripts ($\pm 1$ SE).}
    \label{fig:kidney}
\end{figure}

We separately hand and LLM\footnote{LLM coding was done by Claude-opus-4-6 using a structured-output schema
that enforces valid JSON.} coded each comparison according to whether they fit any element of a multi-label list of the five indecision criteria described in \Cref{fig:kidney}. The codebook is in Table \ref{tab:codebook}.
\begin{table}[h!]
\centering
\caption{Indecision codebook, as given verbatim to the coding
model.}
\label{tab:codebook}
\begin{tabular}{@{}p{0.22\linewidth}p{0.7\linewidth}@{}}
\toprule
Field / key & Definition (as given to the model) \\
\midrule
\texttt{explicit} & Explicit difficulty statements --- e.g.\ ``this is hard'',
``tough one'', ``I'm torn'', ``I'm not sure'', ``it's close''. \\[0.4em]
\texttt{self\_correction} & Self-correction, or changing their mind
mid-sentence. \\[0.4em]
\texttt{back\_and\_forth} & Extended back-and-forth weighing both options before
deciding. \\[0.4em]
\texttt{hedging} & Long hedging before stating a choice --- e.g.\ ``I think
maybe'', ``probably'', ``I guess''. \\[0.4em]
\texttt{discomfort} & Expressing discomfort with having to choose at all. \\
\bottomrule
\end{tabular}
\end{table}

\subsection{Proof of Proposition \ref{prop:linear-rule-is-aggregate-optimal}} \label{app:linear-rule-is-aggregate-optimal}

\begin{proof}
Fix any $\psi$ and any $\omega^* \in \Delta^{m-1}$. By definition of $F_{\omega^*}$, at every $x$ it chooses the maximizer of $\langle \omega^*,\psi(x,F_{\omega^*}(x)) \rangle$; thus, for every $x\in\mathcal X$ and every $F \in \rules$,
\[
\langle \omega^*,\psi(x,F_{\omega^*}(x))\rangle
\ge
\langle \omega^*,\psi(x,F(x))\rangle .
\]
Fix any $F \in \rules$. Then, summing over $x$ and dividing by $|\mathcal X|$ gives
\[
\frac{1}{|\mathcal X|}\sum_{x\in\mathcal X}
\langle \omega^*,\psi(x,F_{\omega^*}(x))\rangle
\ge
\frac{1}{|\mathcal X|}\sum_{x\in\mathcal X}
\langle \omega^*,\psi(x,F(x))\rangle .
\]
By Definition 4.1, the left-hand side is exactly $U_{M_{\omega^*}}(F_{\omega^*})$ and the right-hand side is $U_{M_{\omega^*}}(F)$. Since $F$ was arbitrary,
\[
F_{\omega^*}\in \arg\max_{F\in\mathcal F} U_{M_{\omega^*}}(F). \qedhere
\]
\end{proof}

\subsection{Proof of Theorem \ref{thm:linear-srum-relationship}} \label{app:linear-srum-relationship}
\begin{proof}
Fix $\psi, h, \beta > 0$, and a query $q=(y,y';x)$. Fix any linear S-RUM defined by $\theta$. Then, the probability of reporting $y\succ y'$ is
\[
S_{h_\beta}(q;V_\theta)(y\succ y')
=
h\!\left(\beta\left(V_\theta(x,y)-V_\theta(x,y')\right)\right)
=
h\!\left(\beta\langle \theta,\psi(x,y)-\psi(x,y')\rangle\right).
\]

Now consider the linear priority model $M_\omega$ with $\omega=\theta/\|\theta\|_1$. By \Cref{def:linear-perfectly-separable-model}, for each priority $j$ and any background rule $F$,
\[
\Delta^F_j(q)
=
u_j(F_{x\to y})-u_j(F_{x\to y'})
=
\frac{1}{|\cX|}
\left(\psi_j(x,y)-\psi_j(x,y')\right).
\]
This quantity is independent of the background rule $F$, so any unanimous rule aggregator returns this same value. Hence the directional score gap is
\[
\kappa_{M_\omega}(q)
=
s^+_{M_\omega}(q)-s^-_{M_\omega}(q)
=
\sum_{j=1}^d \omega_j \Delta^F_j(q)
=
\frac{1}{|\cX|}
\left\langle \omega,\psi(x,y)-\psi(x,y')\right\rangle .
\]
Substituting $\omega=\theta/\|\theta\|_1$ gives
\[
\kappa_{M_\omega}(q)
=
\frac{1}{|\cX|\|\theta\|_1}
\left\langle \theta,\psi(x,y)-\psi(x,y')\right\rangle .
\]
Therefore, if $\beta'=\beta|\cX|\|\theta\|_1$, then
\[
h_{\beta'}(\kappa_{M_\omega}(q))
=
h\!\left(\beta'\kappa_{M_\omega}(q)\right)
=
h\!\left(\beta\langle \theta,\psi(x,y)-\psi(x,y')\rangle\right),
\]
which equals the S-RUM response probability for $y\succ y'$. The probability of $y'\succ y$ also matches, since in the zero-threshold case both models assign it the complementary probability.

It remains to show that the induced aggregate-optimal rules agree. For any rule $F$,
\[
U_{M_\omega}(F)
=
\sum_{j=1}^d \omega_j u_j(F)
=
\frac{1}{|\cX|}\sum_{x\in\cX}
\left\langle \omega,\psi(x,F(x))\right\rangle .
\]
Using $\omega=\theta/\|\theta\|_1$, this becomes
\[
U_{M_\omega}(F)
=
\frac{1}{|\cX|\|\theta\|_1}
\sum_{x\in\cX}
\left\langle \theta,\psi(x,F(x))\right\rangle
=
\frac{1}{|\cX|\|\theta\|_1}
\sum_{x\in\cX}
V_\theta(x,F(x)).
\]
Thus $U_{M_\omega}(F)$ is a positive scalar multiple of the linear S-RUM objective for every $F$. Therefore, the two objectives have the same maximizers over any $\rules'\subseteq\rules$.
\end{proof}

\subsection{Proof that weight recovery implies rule recovery} \label{app:weight-approx}

\begin{lemma}
\label{lem:weight-recovery-rule-recovery}
Let $\mathcal F' \subseteq\mathcal F$ be any rule class. Suppose
$\|\widehat\omega-\omega^*\|_1\le \varepsilon$, and let
\[
\widehat F\in \arg\max_{F\in\mathcal F'} U_{\widehat\omega}(F),
\qquad
F^*\in \arg\max_{F\in\mathcal F'} U_{\omega^*}(F).
\]
Then
\[
U_{\omega^*}(\widehat F)\ge U_{\omega^*}(F^*)-2\varepsilon.
\]
\end{lemma}
\begin{proof}
Note that we may shift each $u_j$ by a constant without changing any utility differences or aggregate-optimal rules. In particular, because $|u_j(F) - u_j(F')|\leq 1$ for all $F \in \rules$, we can shift these utilities so that $|u_j(F)|\in[0,1]$ for all $j,F$, and hence $|u_j(F)|\le 1$. For every rule \(F\),
\[
|U_{\widehat\omega}(F)-U_{\omega^\ast}(F)|
=
\left|
\sum_{j=1}^d
(\widehat\omega_j-\omega^\ast_j)u_j(F)
\right|
\le
\sum_{j=1}^d
|\widehat\omega_j-\omega^\ast_j|\,|u_j(F)|
\le
\|\widehat\omega-\omega^\ast\|_1
\le
\varepsilon.
\]
Since \(\widehat F\) maximizes \(U_{\widehat\omega}\) over \(\mathcal F'\),
\[
U_{\widehat\omega}(\widehat F)
\ge
U_{\widehat\omega}(F^\ast).
\]
Therefore,
\[
U_{\omega^\ast}(\widehat F)
\ge
U_{\widehat\omega}(\widehat F)-\varepsilon
\ge
U_{\widehat\omega}(F^\ast)-\varepsilon
\ge
U_{\omega^\ast}(F^\ast)-2\varepsilon.
\]
Thus \(\widehat F\) is \(2\varepsilon\)-optimal.
\end{proof}

\subsection{Regret Computation Methods}
\label{app:regret-computation}

\subsubsection{Empirical estimate of $\operatorname{Avg-Regret}$}
The input distribution is fixed once and shared across all methods and oracles, so
that differences in regret reflect the methods rather than the sampled inputs. It
consists of $M=300$ inputs $x$, each with $|\cY(x)|=5$ candidate outputs whose
feature vectors $\psi(x,y)\in\R^{5}$ are drawn i.i.d.\ from $\mathrm{Uniform}[0,1]$;
the same sample is reused for every $(\omega^*,\widehat\omega)$ pair.

For an input $x$, the learned rule selects
$F_{\widehat\omega}(x)=\arg\max_{y\in\cY(x)}\langle\widehat\omega,\psi(x,y)\rangle$,
and we define
\[
  \Reg_x=\max_{y\in\cY(x)}\langle\omega^*,\psi(x,y)\rangle
         -\langle\omega^*,\psi\!\left(x,F_{\widehat\omega}(x)\right)\rangle,
  \qquad
  \Range_x=\max_{y\in\cY(x)}\langle\omega^*,\psi(x,y)\rangle
          -\min_{y\in\cY(x)}\langle\omega^*,\psi(x,y)\rangle .
\]
Note $\Range_x$ depends on $\omega^*$ alone. We estimate the average regret as the
\emph{ratio of two sample means} over the shared inputs $\{x_1,\dots,x_M\}$,
\[
  \widehat{\operatorname{Avg\text{-}Regret}}(\widehat\omega;\omega^*)
  =\frac{\tfrac1M\sum_{i=1}^M \Reg_{x_i}}{\tfrac1M\sum_{i=1}^M \Range_{x_i}}
  =\frac{\sum_{i=1}^M \Reg_{x_i}}{\sum_{i=1}^M \Range_{x_i}} ,
\]
i.e.\ both $\Reg_x$ and $\Range_x$ are evaluated on the same $300$ Monte-Carlo inputs
(no closed form for the range), and we pool numerator and denominator before dividing. Bars in Figure~\ref{fig:bald_6} report the mean $\pm 1$ SE of this quantity over the oracle
ground truths.

\subsubsection{Computing the $\operatorname{WC-Regret}$}
\label{app:rwc}

We derive the closed form used to compute the worst-case regret
$\operatorname{WC-Regret}_{M_{\omega^*}}(F_{\widehat{\omega}})$. Recall that
\[
\operatorname{WC-Regret}_{M_{\omega^*}}(F_{\widehat{\omega}})
=
\frac{\sup_x \Reg_x(\widehat{\omega};\omega^*)}{\sup_x \Range_x(\omega^*)},
\]
where both suprema are taken over the worst-case input domain
$x\in([0,1]^d)^5$.

\paragraph{Computing the numerator.}
For fixed $\omega^*$ and $\widehat{\omega}$, the worst-case regret can be computed by the linear program
\[
\max_{\delta\in[-1,1]^d}
\left\{
\langle \omega^*,\delta\rangle:
\langle \widehat{\omega},\delta\rangle\le 0
\right\}.
\]
The vector $\delta$ represents the feature difference $z-z'$ between the true-optimal output $z$
and the output $z'$ selected by the learned rule. The constraint
$\langle \widehat{\omega},\delta\rangle\le 0$ says that the learned rule weakly prefers $z'$ to
$z$, while the objective $\langle \omega^*,\delta\rangle$ is the regret under the true weights.

\begin{lemma}[Worst-case regret as a linear program]
\label{lem:wc-regret-pairwise}
Let $\mathcal Z=[0,1]^d$. Suppose each input $x$ has a feasible output set
$\cY(x)\subseteq \mathcal Z$ consisting of $k$ distinct feature vectors. For
$\omega\in\Delta^{d-1}$, let
$F_\omega(x)\in\arg\max_{z\in \cY(x)}\langle \omega,z\rangle$, with arbitrary
fixed tie-breaking. Suppose the worst-case domain ranges over all possible
collections of $k$ distinct feature vectors in $[0,1]^d$. Then, for any fixed
$\widehat{\omega},\omega^*$,
\[
\sup_x \Reg_x(\widehat\omega;\omega^*)
=
\sup_{z,z'\in[0,1]^d}
\left\{
\langle \omega^*,z-z'\rangle:
\langle \widehat\omega,z'\rangle
\ge
\langle \widehat\omega,z\rangle
\right\}
=
\max_{\delta\in[-1,1]^d}
\left\{
\langle \omega^*,\delta\rangle:
\langle \widehat\omega,\delta\rangle\le 0
\right\}.
\]
\end{lemma}

\begin{proof}
We first prove the equality between the worst-case regret and the pairwise optimization. Fix any
input $x$, and let $z=F_{\omega^*}(x)$ and $z'=F_{\widehat\omega}(x)$. Since
$F_{\widehat\omega}(x)$ selects $z'$ from $\cY(x)$, we have
$\langle \widehat\omega,z'\rangle\ge \langle \widehat\omega,z\rangle$. Moreover,
$\Reg_x(\widehat\omega;\omega^*)=\langle \omega^*,z\rangle-\langle \omega^*,z'\rangle
=\langle \omega^*,z-z'\rangle$. Thus every regret value achieved by some input $x$ is feasible
for the pairwise optimization, so
\[
\sup_x \Reg_x(\widehat\omega;\omega^*)
\le
\sup_{z,z'\in[0,1]^d}
\left\{
\langle \omega^*,z-z'\rangle:
\langle \widehat\omega,z'\rangle
\ge
\langle \widehat\omega,z\rangle
\right\}.
\]

For the reverse inequality, define
\[
B :=
\sup_{z,z'\in[0,1]^d}
\left\{
\langle \omega^*,z-z'\rangle:
\langle \widehat\omega,z'\rangle
\ge
\langle \widehat\omega,z\rangle
\right\}.
\]
We show that $\sup_x \Reg_x(\widehat\omega;\omega^*)\ge B$. If $B=0$, this is immediate because
regret is nonnegative. Now suppose $B>0$. Fix $\varepsilon>0$. By the definition of supremum,
there exist $z,z'\in[0,1]^d$ such that
$\langle \widehat\omega,z'\rangle\ge \langle \widehat\omega,z\rangle$ and
$\langle \omega^*,z-z'\rangle>B-\varepsilon/3$. Taking $\varepsilon$ small enough, this implies
$\langle \omega^*,z\rangle>\langle \omega^*,z'\rangle$.

The learned rule may be indifferent between $z$ and $z'$, so we perturb the pair to make the
learned rule's preference strict. For $\rho\in(0,1)$, define
$\tilde z=(1-\rho)z$ and $\tilde z'=(1-\rho)z'+\rho\mathbf 1$. Then
$\tilde z,\tilde z'\in[0,1]^d$. Since $\widehat\omega\in\Delta^{d-1}$,
\[
\langle \widehat\omega,\tilde z'\rangle
-
\langle \widehat\omega,\tilde z\rangle
=
(1-\rho)\langle \widehat\omega,z'-z\rangle+\rho
>0.
\]
Thus the learned rule strictly prefers $\tilde z'$ to $\tilde z$. The true-regret objective changes
to
$\langle \omega^*,\tilde z-\tilde z'\rangle
=(1-\rho)\langle \omega^*,z-z'\rangle-\rho$, which converges to
$\langle \omega^*,z-z'\rangle$ as $\rho\downarrow 0$. Hence we can choose $\rho>0$ small enough
that $\langle \omega^*,\tilde z-\tilde z'\rangle>B-2\varepsilon/3$. In particular,
$\langle \omega^*,\tilde z\rangle>\langle \omega^*,\tilde z'\rangle$.

Now choose $k-2$ additional distinct outputs $r_1,\ldots,r_{k-2}\in[0,1]^d$, distinct from
$\tilde z$ and $\tilde z'$, with sufficiently small coordinates so that, for every $\ell$,
$\langle \omega^*,r_\ell\rangle<\langle \omega^*,\tilde z\rangle$ and
$\langle \widehat\omega,r_\ell\rangle<\langle \widehat\omega,\tilde z'\rangle$. Such outputs
exist because $[0,1]^d$ contains infinitely many distinct points arbitrarily close to the origin,
while $\langle \omega^*,\tilde z\rangle>0$ and
$\langle \widehat\omega,\tilde z'\rangle>0$.

By the worst-case-domain assumption, there is an input $x$ whose feasible output set is
$\cY(x)=\{\tilde z,\tilde z',r_1,\ldots,r_{k-2}\}$. By construction,
$F_{\omega^*}(x)=\tilde z$ and $F_{\widehat\omega}(x)=\tilde z'$. Therefore,
$\Reg_x(\widehat\omega;\omega^*)=\langle \omega^*,\tilde z-\tilde z'\rangle
>B-2\varepsilon/3$. Since $\varepsilon>0$ was arbitrary, we have
$\sup_x \Reg_x(\widehat\omega;\omega^*)\ge B$. Combining this with the forward inequality proves
the first equality.

It remains to convert the pairwise optimization to the linear program in $\delta$. Let
$\delta=z-z'$. Then $z,z'\in[0,1]^d$ implies $\delta\in[-1,1]^d$, and the learned-rule constraint
becomes $\langle \widehat\omega,z'\rangle\ge \langle \widehat\omega,z\rangle$, equivalently
$\langle \widehat\omega,\delta\rangle\le 0$. The objective becomes
$\langle \omega^*,z-z'\rangle=\langle \omega^*,\delta\rangle$. Thus every feasible pair
$(z,z')$ induces a feasible $\delta$ with the same objective value.

Conversely, every feasible $\delta\in[-1,1]^d$ satisfying
$\langle \widehat\omega,\delta\rangle\le 0$ can be written as $\delta=z-z'$ for some
$z,z'\in[0,1]^d$: set $z_j=\max\{\delta_j,0\}$ and $z'_j=\max\{-\delta_j,0\}$ for each
coordinate $j$. Then $z,z'\in[0,1]^d$, $z-z'=\delta$, and
$\langle \widehat\omega,z'\rangle\ge \langle \widehat\omega,z\rangle$. Hence this pair is
feasible for the pairwise problem and achieves the same objective value as $\delta$. Therefore,
\[
\sup_{z,z'\in[0,1]^d}
\left\{
\langle \omega^*,z-z'\rangle:
\langle \widehat\omega,z'\rangle
\ge
\langle \widehat\omega,z\rangle
\right\}
=
\max_{\delta\in[-1,1]^d}
\left\{
\langle \omega^*,\delta\rangle:
\langle \widehat\omega,\delta\rangle\le 0
\right\}.
\qedhere\]
\end{proof}

We now show that the worst-case utility range is $1$.

\begin{lemma}
Fix any $\omega^*\in\Delta^{d-1}$. Then $\sup_x \Range_x(\omega^*)=1$.
\end{lemma}

\begin{proof}
The normalizing range is the largest utility spread attainable on a single input
$x\in([0,1]^d)^5$:
\[
\sup_x \Range_x(\omega^*)
=
\sup_x
\left(
\max_{y\in\cY(x)}
\langle \omega^*,\psi(x,y)\rangle
-
\min_{y\in\cY(x)}
\langle \omega^*,\psi(x,y)\rangle
\right).
\]
Equivalently, this is
\[
\sup_x \Range_x(\omega^*)
=
\sup_x
\max_{y,y'\in\cY(x)}
\left\langle
\omega^*,
\psi(x,y)-\psi(x,y')
\right\rangle.
\]

We first upper bound this quantity. Fix any input $x$ and any pair
$y,y'\in\cY(x)$. Since $x\in([0,1]^d)^5$, both feature vectors $\psi(x,y)$ and
$\psi(x,y')$ lie in $[0,1]^d$. Therefore
\[
\left\langle
\omega^*,
\psi(x,y)-\psi(x,y')
\right\rangle
\le
\max_{z,z'\in[0,1]^d}
\langle \omega^*,z-z'\rangle.
\]
Since this holds for every input $x$ and every pair $y,y'\in\cY(x)$,
$\sup_x \Range_x(\omega^*)\le \max_{z,z'\in[0,1]^d}\langle \omega^*,z-z'\rangle$.

Because $\omega^*\in\Delta^{d-1}$, all coordinates of $\omega^*$ are nonnegative. Thus, for any
$z,z'\in[0,1]^d$, $\langle \omega^*,z-z'\rangle\le
\langle \omega^*,\mathbf 1-\mathbf 0\rangle=\|\omega^*\|_1$. This upper bound is attained by
$z=\mathbf 1$ and $z'=\mathbf 0$, so
$\max_{z,z'\in[0,1]^d}\langle \omega^*,z-z'\rangle=\|\omega^*\|_1$.

Finally, this upper bound is attainable in the worst-case input domain. Since the supremum ranges
over all $x\in([0,1]^d)^5$, consider an input whose feasible output set contains one candidate with
feature vector $\mathbf 1$ and another with feature vector $\mathbf 0$. For this input,
$\Range_x(\omega^*)\ge \langle \omega^*,\mathbf 1-\mathbf 0\rangle=\|\omega^*\|_1$. Combining
the upper and lower bounds gives $\sup_x \Range_x(\omega^*)=\|\omega^*\|_1$. Since
$\omega^*\in\Delta^{d-1}$, $\|\omega^*\|_1=1$. Therefore the denominator of
$\operatorname{WC-Regret}_{M_{\omega^*}}(F_{\widehat{\omega}})$ is $1$.
\end{proof}

Combining the two lemmas gives the linear program we solve:
\[
\operatorname{WC-Regret}_{M_{\omega^*}}(F_{\widehat{\omega}})
=
\max_{\delta\in[-1,1]^d}
\left\{
\langle \omega^*,\delta\rangle:
\langle \widehat{\omega},\delta\rangle\le 0
\right\}.
\]

\subsection{Active Learning Algorithm Details}
\label{app:algo-details}

This appendix gives the algorithmic details for the active learning procedures used in \Cref{sec:speed}. Throughout, $R^*$ denotes the true response model that generates the individual's observed response, and $R$ denotes the learner's assumed response model used for posterior updates. The learner's unknown parameter vector is denoted by $\vartheta$. In the simplest methods, $\vartheta=\omega$; in the variants that learn thresholds or a flexible noise model, $\vartheta$ also contains those additional parameters. 

\paragraph{Queries and candidate pool.}
At each active learning round, the learner does not optimize over all possible pairwise queries. Instead, it draws a finite \emph{candidate pool} \[ \mathcal C_t=\{q_t^{(1)},\ldots,q_t^{(C)}\} \] of $C$ randomly sampled pairwise queries, scores each query in this pool, and asks the query with the highest score. In the simulations, a query is generated by independently sampling two candidate feature vectors from $[0,1]^d$ with $d=5$.

Let the feature difference be $\delta(q)=\psi(x,y)-\psi(x,y')$. As established 
for the linear perfectly separable model (after dropping the $1/|\cX|$ factor), the decisiveness and total evidence 
are respectively $\kappa_{M_\omega}(q)=\langle\omega,\delta(q)\rangle$ and 
$r_{M_\omega}(q)=\sum_{j=1}^d\omega_j\,|\delta_j(q)|$. In the linear simulations $r_M(q)\le 1$, so we restrict $\tau_r\in[0,1]$. 

\paragraph{BALD score.}
Queries are selected using Bayesian Active Learning by Disagreement (BALD), a standard acquisition
rule that scores a query by the expected information its response provides about the unknown
parameter \cite{houlsby2011bayesian}. At each round, the learner maintains
$N_{\mathrm{post}}$ posterior samples
$
\{\vartheta_{t-1}^{(i)}\}_{i=1}^{N_{\mathrm{post}}}
$
from the current posterior after transcript $\mathcal T_{t-1}$. To estimate the BALD score of a
candidate query, the learner uses a subsample of size $N_{\mathrm{BALD}}\le N_{\mathrm{post}}$ from this
posterior sample set. For a candidate query $q$, the learner evaluates the response distribution
$
R(\cdot\mid q,\vartheta_{t-1}^{(i)})
$
under each sampled posterior draw. The Monte Carlo BALD estimate is
\begin{equation}
\label{eq:bald}
\widehat{\mathrm{BALD}}_t(q)
=
\mathrm{Ent}\!\left(
\frac{1}{N_{\mathrm{BALD}}}\sum_{i=1}^{N_{\mathrm{BALD}}}
R(\cdot\mid q,\vartheta_{t-1}^{(i)})
\right)
-
\frac{1}{N_{\mathrm{BALD}}}\sum_{i=1}^{N_{\mathrm{BALD}}}
\mathrm{Ent}\!\left(R(\cdot\mid q,\vartheta_{t-1}^{(i)})\right),
\end{equation}
where $\mathrm{Ent}$ is Shannon entropy. The first term is the learner's marginal
uncertainty about the response to $q$, averaging over posterior uncertainty about $\vartheta$; the
second term subtracts the expected response noise that would remain if $\vartheta$ were known.
Thus, BALD rewards queries where different values of $\vartheta$ give highly different, \textit{confident} answers---these are exactly the questions that can help pin down $\vartheta$ the quickest. In the main experiments, we use
$
N_{\mathrm{post}}=200$ and $N_{\mathrm{BALD}}=50.
$
In the flexible-noise variants
(\textsc{Utilize-4}$^\dagger$, \textsc{Ignore}$^\dagger$,
\textsc{Correct}$^\dagger$), we use $N_{\mathrm{BALD}}=30$ for computational efficiency.

\paragraph{Posterior updating and MCMC} After each observed response, the learner must update its posterior over the unknown parameter
vector $\vartheta$. Because this posterior generally cannot be sampled from exactly, we approximate it using Markov-chain Monte Carlo (MCMC): a standard family of sampling methods that constructs a random walk over parameter values whose long-run distribution is the desired posterior. In each round, after observing a new response, the learner needs a fresh sample
from the updated posterior. We initialize the MCMC random walk
using information from the previous round's posterior, which we call a warm start. The first $B$ steps of the chain are discarded and called ``burn-in," because these early steps
can still depend strongly on the chain's initialization. We then keep the next $N_{\mathrm{post}}$ states of the chain and use them as the posterior sample set
for the next BALD step.

\begin{algorithm}[h]
\caption{BALD Active Learning Template}
\label{alg:bald-template}
\begin{tabbing}
\qquad \= \qquad \= \kill
\textbf{Input:} query budget \(T\), candidate-pool size \(C\), posterior-sample count
\(N_{\mathrm{post}}\), \\
BALD subsample size \(N_{\mathrm{BALD}}\), MCMC burn-in \(B\), prior \(\pi_0\), \\
query distribution \(\mathcal D_q\), learner response model \(R\), true response model \(R^*\). \\[0.25em]
Initialize transcript \(\mathcal T_0\gets\emptyset\). \\[0.25em]
Draw posterior samples \(\{\vartheta_0^{(i)}\}_{i=1}^{N_{\mathrm{post}}}\sim\pi_0\). \\[0.5em]
\textbf{for} \(t=1,\ldots,T\) \textbf{do} \\[0.25em]
\> (1) Draw candidate pool:
$\mathcal{C}_t=\{q_t^{(1)},\ldots,q_t^{(C)}\},
\quad
q_t^{(c)}\stackrel{\mathrm{i.i.d.}}{\sim}\mathcal D_q.
$
\\[-0.25em]
\> (2) Estimate \(\widehat{\mathrm{BALD}}_t(q_t^{(c)})\) for each \(q_t^{(c)}\in\mathcal C_t\)
using \eqref{eq:bald} with \(N_{\mathrm{BALD}}\) posterior samples. \\[-0.25em]
\> (3) Choose the next query:
$q_t\in\arg\max_{q\in\mathcal C_t}\widehat{\mathrm{BALD}}_t(q).$
\\[-0.25em]
\> (4) Observe response: $\triangleright_t\sim R^*(\cdot\mid q_t;M_{\omega^*}).$
\\[-0.25em]
\> (5) Update transcript according to whether the method wants to skip or retain that query. \\
\> \quad For most methods, \(\mathcal T_t\gets\mathcal T_{t-1}\cup\{(q_t,\triangleright_t)\}\). \\
\> \quad For \textsc{Ignore} methods, indecisive responses are not retained. \\[-0.25em]
\> (6) Sample from the learner posterior $\pi_t(\vartheta)\propto \pi_0(\vartheta)
\prod_{(q,\triangleright)\in\mathcal T_t} R(\triangleright\mid q,\vartheta)$
\\using MCMC with \(B\) burn-in iterations and \(N_{\mathrm{post}}\) retained samples, obtaining
$\{\vartheta_t^{(i)}\}_{i=1}^{N_{\mathrm{post}}}.$
\\Let
$\omega_t^{(i)} := \omega(\vartheta_t^{(i)})$
denote the priority-weight component of the \(i\)-th posterior draw. \\[0.25em]
\textbf{end for} \\[0.5em]
\textbf{Output:} $\widehat\omega_T=\frac{1}{N_{\mathrm{post}}}
\sum_{i=1}^{N_{\mathrm{post}}}\omega_T^{(i)}.$
\end{tabbing}
\end{algorithm}

\paragraph{Sampling methods used.}
All posterior updates in Line~(6) of Algorithm~\ref{alg:bald-template} are performed using MCMC. We use three standard MCMC methods. Metropolis--Hastings is a basic MCMC update: the sampler proposes a random change to the current parameter value and accepts or rejects that proposal based on how plausible it is under the posterior \cite{metropolis1953equation,hastings1970monte}. Metropolis-within-Gibbs is the version used when the parameter vector has several components \cite{gelfand1990sampling, tierney1994markov, chib1995understanding}. Instead of proposing a change to the entire vector at once, the sampler updates one component at a time, such as $\omega$, then $\tau_r$, then $\gamma$, holding the others fixed during each update. When a component cannot be sampled exactly, that component is updated using a Metropolis--Hastings step. Third, for parameters constrained to lie
on a simplex, such as the priority weights $\omega$, we use hit-and-run proposals, which are
designed for random-walk sampling in bounded convex sets and therefore avoid leaving the
simplex \cite{smith1984efficient}.

\paragraph{Method specifications.}
All methods instantiate Algorithm~\ref{alg:bald-template}. Unless otherwise stated, the query
distribution $\mathcal D_q$, candidate-pool size $C$, posterior sample size $N_{\mathrm{post}}$,
BALD subsample size $N_{\mathrm{BALD}}$, and MCMC burn-in $B$ are shared across methods and
given in \Cref{tab:bald-hyperparams}. The methods differ only in five parts of the algorithm:
the learned parameter vector and prior, the learner response model $R$, the true response model
$R^*$, the transcript retention rule, and the posterior sampler:

\begin{enumerate}
    \item \textbf{Learned parameter vector and prior.}
    The methods use one of three learned parameterizations.

    \begin{itemize}
        \item \textit{Known-noise, known-threshold methods}
        (\textsc{Utilize-4}, \textsc{Utilize-3}, \textsc{Ignore}, and \textsc{Correct}).
        The learner estimates only the priority weights:
        $\vartheta=\omega\in\Delta^{d-1}$, with
        $\omega\sim\mathrm{Dirichlet}(\mathbf{1})$.
        The thresholds $(\tau_r,\tau_\kappa)$, the link function $h$, and the logistic inverse
        temperature $\beta$ are known.

        \item \textit{Unknown-noise methods}
        (\textsc{Utilize-4}$^\dagger$, \textsc{Ignore}$^\dagger$, and
        \textsc{Correct}$^\dagger$).
        The learner estimates both the priority weights and a flexible noise model:
        \[
        \vartheta=(\omega,\eta),
        \qquad
        \eta=(\alpha,\mu,\sigma).
        \]
        Here $\alpha=(\alpha_1,\alpha_2,\alpha_3)\in\Delta^2$ are the weights of a
        three-component Gaussian mixture, $\mu=(\mu_1,\mu_2,\mu_3)$ are the component means, and
        $\sigma=(\sigma_1,\sigma_2,\sigma_3)$ are the component standard deviations. These
        parameters define the Gaussian-mixture CDF
        \[
        F_{\mathrm{GMM}}(t;\eta)
        =
        \sum_{k=1}^3 \alpha_k
        \Phi\!\left(\frac{t-\mu_k}{\sigma_k}\right),
        \]
        where $\Phi$ is the standard normal CDF.

        The flexible-noise learner uses $F_{\mathrm{GMM}}$ as the CDF of an additive noise term,
        not as a symmetric two-sided link. In the four-response model below, conflict is computed
        as the probability that a noisy latent margin falls inside the conflict band. This gives a
        CDF difference, which is nonnegative by monotonicity of $F_{\mathrm{GMM}}$, even when the
        Gaussian mixture is asymmetric.

        We use the same uniform Dirichlet prior on $\omega$ as above. For the mixture parameters,
        we use weakly informative priors: the component weights have a uniform Dirichlet prior,
        the component means have centered Gaussian priors, and the log standard deviations have
        Gaussian priors. Each $\sigma_k$ is constrained to lie in
        $[\sigma_{\min},\sigma_{\max}]$ to avoid numerically degenerate mixture components. The
        exact hyperparameter values are listed in \Cref{tab:bald-hyperparams}.

        \item \textit{Threshold-learning method}
        (\textsc{Utilize-4}$^\diamond$).
        This method is the version of \textsc{Utilize-4} that does not assume the response
        thresholds are known. The learner estimates
        $\vartheta=(\omega,\tau_r,\gamma)$,
        where $\omega$ is the priority-weight vector, $\tau_r$ is the total-evidence threshold,
        and $\gamma$ is the learner's estimate of the conflict threshold $\tau_\kappa$.

        We allow $\tau_r$ to vary continuously over $[0,1]$. For the conflict threshold, we use a
        finite grid of possible values,
        \[
        \Gamma=
        \left\{
        0,\frac{\gamma_{\max}}{K_\gamma-1},
        \frac{2\gamma_{\max}}{K_\gamma-1},
        \ldots,\gamma_{\max}
        \right\}.
        \]
        Here $K_\gamma$ is the number of grid points, and $\gamma_{\max}=0.95$ is the largest
        allowed value. Thus, instead of learning an arbitrary real-valued conflict threshold, the
        learner chooses among $K_\gamma$ evenly spaced candidate values between $0$ and $0.95$.

        The prior on $\omega$ is uniform over the simplex,
        $\omega\sim\mathrm{Dirichlet}(\mathbf{1})$, and the prior on $\tau_r$ is uniform over
        $[0,1]$, $\tau_r\sim\mathrm{Uniform}[0,1]$. For $\gamma$, we put a prior directly on the
        grid $\Gamma$. This prior assigns zero mass to the two grid endpoints and positive mass
        to interior grid values. Thus, it excludes the degenerate boundary cases $\gamma=0$ and
        $\gamma=\gamma_{\max}$ while placing a weak symmetric prior over the remaining interior
        grid values. See details in Appendix~\ref{app:learn-tau}.
    \end{itemize}

    \item \textbf{Learner's assumed response model $R$ used in Line~(2) and Line~(6).}
    The learner response model is the model used both to score candidate queries by BALD and to
    compute the likelihood in the posterior update.

    \begin{itemize}
        \item \textit{Four-response logistic learner}
        (\textsc{Utilize-4}).
        The learner uses the four-response model
        \[
        R(\cdot\mid q,\omega)
        =
        R^\circ_{h^{\mathrm{logit}}_\beta;\tau_r,\tau_\kappa}
        (\cdot\mid q;M_\omega),
        \]
        where $h^{\mathrm{logit}}_\beta(t)=(1+\exp(-\beta t))^{-1}$.

        \item \textit{Three-response generic-indecision learner}
        (\textsc{Utilize-3}).
        The learner uses the collapsed three-response model
        \[
        R(\cdot\mid q,\omega)
        =
        R^{\oslash}_{h^{\mathrm{logit}}_\beta;\tau_r,\tau_\kappa}
        (\cdot\mid q;M_\omega),
        \]
        over alphabet $\{\succ,\prec,\oslash\}$, where $\oslash$ denotes generic indecision.
        This model is formalized in \Cref{def:generic-indecision}.

        \item \textit{Four-response flexible-noise learner}
        (\textsc{Utilize-4}$^\dagger$).
        The learner uses a flexible latent-noise model for the four-response probabilities. For a
        query $q$, let $\delta=\kappa_{M_\omega}(q)$ be the signed aggregate evidence and let
        $r=r_{M_\omega}(q)$ be its total-evidence magnitude. The learner models the response as depending on a noisy latent margin $\delta+\varepsilon$, where $\varepsilon$ has CDF
        $F_{\mathrm{GMM}}(\cdot;\eta)$.

        If $r<\tau_r$, the query has insufficient total evidence, so
        $R(\sim\mid q,\omega,\eta)=1$. If $r\ge\tau_r$, the conflict band is
        $[-\tau_\kappa r,\tau_\kappa r]$, and the response probabilities are
        \[
        \begin{aligned}
        R(\succ\mid q,\omega,\eta)
        &=
        1-F_{\mathrm{GMM}}(\tau_\kappa r-\delta;\eta),\\
        R(\prec\mid q,\omega,\eta)
        &=
        F_{\mathrm{GMM}}(-\tau_\kappa r-\delta;\eta),\\
        R(\bowtie\mid q,\omega,\eta)
        &=
        F_{\mathrm{GMM}}(\tau_\kappa r-\delta;\eta)
        -
        F_{\mathrm{GMM}}(-\tau_\kappa r-\delta;\eta),\\
        R(\sim\mid q,\omega,\eta)
        &=0.
        \end{aligned}
        \]
        These are exactly the probabilities that the noisy margin $\delta+\varepsilon$ lies above
        the conflict band, below the conflict band, or inside the conflict band.

        This construction gives valid probabilities for any Gaussian-mixture parameters $\eta$,
        including asymmetric mixtures. Since $\tau_\kappa r\ge 0$, the upper endpoint
        $\tau_\kappa r-\delta$ is at least the lower endpoint $-\tau_\kappa r-\delta$. Because
        $F_{\mathrm{GMM}}$ is a CDF and hence monotone, the conflict probability
        $F_{\mathrm{GMM}}(\tau_\kappa r-\delta;\eta)
        -F_{\mathrm{GMM}}(-\tau_\kappa r-\delta;\eta)$ is nonnegative. The other two
        probabilities are a lower-tail probability and an upper-tail probability, and the three
        terms sum to one.

        \item \textit{Four-response threshold-learning learner}
        (\textsc{Utilize-4}$^\diamond$).
        The learner uses
        \[
        R(\cdot\mid q,\omega,\tau_r,\gamma)
        =
        R^\circ_{h^{\mathrm{logit}}_\beta;\tau_r,\gamma}
        (\cdot\mid q;M_\omega).
        \]
        Further details are given in \Cref{app:learn-tau}.

        \item \textit{Binary logistic learner}
        (\textsc{Ignore} and \textsc{Correct}).
        The learner assumes a binary Bradley--Terry response model:
        \[
        R(\succ\mid q,\omega)
        =
        h^{\mathrm{logit}}_\beta(\kappa_{M_\omega}(q)),
        \qquad
        R(\prec\mid q,\omega)
        =
        h^{\mathrm{logit}}_\beta(-\kappa_{M_\omega}(q)).
        \]
        BALD scores are computed under this binary model, and the posterior is updated only on
        retained binary labels.

        \item \textit{Binary flexible-noise learner}
        (\textsc{Ignore}$^\dagger$ and \textsc{Correct}$^\dagger$).
        The learner uses the corresponding binary latent-noise likelihood. For
        $\delta=\kappa_{M_\omega}(q)$, the noisy latent margin is again
        $\delta+\varepsilon$, with $\varepsilon$ distributed according to
        $F_{\mathrm{GMM}}(\cdot;\eta)$. The binary response probabilities are
        \[
        R(\succ\mid q,\omega,\eta)
        =
        \Pr(\delta+\varepsilon>0)
        =
        1-F_{\mathrm{GMM}}(-\delta;\eta),
        \qquad
        R(\prec\mid q,\omega,\eta)
        =
        \Pr(\delta+\varepsilon<0)
        =
        F_{\mathrm{GMM}}(-\delta;\eta).
        \]
        These probabilities are valid for any Gaussian-mixture parameters because they are a CDF
        tail probability and its complement. As in the four-response case, the learner does not
        require the Gaussian mixture to be symmetric.
    \end{itemize}

    \item \textbf{True response model $R^*$ used in Line~(4).}
    The true response model is the data-generating model used to sample the observed response
    $\triangleright_t$.

    \begin{itemize}
        \item \textit{Four-response logistic individual}
        (\textsc{Utilize-4}, \textsc{Utilize-4}$^\dagger$,
        \textsc{Utilize-4}$^\diamond$, \textsc{Ignore}, and \textsc{Ignore}$^\dagger$).
        Responses are generated from the four-response model at the true weights and true
        thresholds:
        \[
        R^*(\cdot\mid q)
        =
        R^\circ_{h^{\mathrm{logit}}_\beta;\tau_r^*,\tau_\kappa^*}
        (\cdot\mid q;M_{\omega^*}).
        \]
        For methods that treat thresholds as known, we set
        $(\tau_r,\tau_\kappa)=(\tau_r^*,\tau_\kappa^*)$.

        \item \textit{Collapsed three-response individual}
        (\textsc{Utilize-3}).
        Responses are first generated from the four-response logistic model at the true weights
        and true thresholds, and then collapsed to the alphabet $\{\succ,\prec,\oslash\}$ by
        mapping both indecisive responses to generic indecision. Equivalently,
        \[
        R^*(\cdot\mid q)
        =
        R^{\oslash}_{h^{\mathrm{logit}}_\beta;\tau_r^*,\tau_\kappa^*}
        (\cdot\mid q;M_{\omega^*}),
        \]
        as formalized in \Cref{def:generic-indecision}.

        \item \textit{Binary forced-choice individual}
        (\textsc{Correct} and \textsc{Correct}$^\dagger$).
        Responses are generated directly from the binary zero-threshold logistic model at the true
        weights:
        \[
        R^*(\succ\mid q)
        =
        h^{\mathrm{logit}}_\beta(\kappa_{M_{\omega^*}}(q)),
        \qquad
        R^*(\prec\mid q)
        =
        h^{\mathrm{logit}}_\beta(-\kappa_{M_{\omega^*}}(q)).
        \]
        These methods are best-case forced-choice benchmarks: unlike \textsc{Ignore}, the data
        are generated by the same binary response family that the learner assumes.
    \end{itemize}

    \item \textbf{Transcript retention rule in Line~(5).}
    The transcript retention rule specifies which observed responses are used in the posterior
    update.

    \begin{itemize}
        \item \textit{Retain all responses}
        (\textsc{Utilize-4}, \textsc{Utilize-3}, \textsc{Utilize-4}$^\dagger$,
        \textsc{Utilize-4}$^\diamond$, \textsc{Correct}, and \textsc{Correct}$^\dagger$).
        Every observed response is retained:
        $\mathcal T_t=\mathcal T_{t-1}\cup\{(q_t,\triangleright_t)\}$.
        For \textsc{Utilize-3}, this includes generic-indecision responses $\oslash$. For
        \textsc{Correct} and \textsc{Correct}$^\dagger$, every response is already binary.

        \item \textit{Discard indecisive responses}
        (\textsc{Ignore} and \textsc{Ignore}$^\dagger$).
        Decisive responses are retained:
        $\mathcal T_t=\mathcal T_{t-1}\cup\{(q_t,\triangleright_t)\}$ if
        $\triangleright_t\in\{\succ,\prec\}$. Indecisive responses are discarded:
        $\mathcal T_t=\mathcal T_{t-1}$ if $\triangleright_t\in\{\sim,\bowtie\}$.
        The query still counts against the query budget even when its response is discarded.
    \end{itemize}

    \item \textbf{Posterior sampler in Line~(6).}
    The sampler is chosen according to what parameters the learner must learn.

    \begin{itemize}
        \item \textit{Simplex-only samplers}
        (\textsc{Utilize-4}, \textsc{Utilize-3}, \textsc{Ignore}, and \textsc{Correct}).
        The posterior is over $\omega\in\Delta^{d-1}$. We use Metropolis--Hastings with
        hit-and-run proposals on the simplex.

        \item \textit{Flexible-noise samplers}
        (\textsc{Utilize-4}$^\dagger$, \textsc{Ignore}$^\dagger$, and
        \textsc{Correct}$^\dagger$).
        The posterior is over $(\omega,\eta)$, where $\eta=(\alpha,\mu,\sigma)$ parameterizes the
        Gaussian-mixture noise CDF $F_{\mathrm{GMM}}(\cdot;\eta)$. We update the different parts
        of $(\omega,\eta)$ separately. We use hit-and-run proposals for $\omega$ and $\alpha$,
        because both are vectors on the simplex. For the remaining mixture parameters, $\mu$ and
        $\log\sigma$, we use random-walk Metropolis updates: the sampler proposes a small random
        change to the current value and accepts the change if it is sufficiently plausible under
        the posterior.

        \item \textit{Threshold-learning sampler}
        (\textsc{Utilize-4}$^\diamond$).
        In this variant, the learner is uncertain not only about the priority weights $\omega$,
        but also about the two response thresholds. Thus, after each new response, the learner
        updates a posterior over $(\omega,\tau_r,\gamma)$, where $\tau_r$ is the total-evidence
        threshold and $\gamma$ is the learner's discretized version of the conflict threshold
        $\tau_\kappa$. Because multiple parameters are being learned, the algorithm uses
        Metropolis-within-Gibbs.

        The sampler updates these three quantities one at a time. First, it updates $\omega$, the
        priority-weight vector. Since $\omega$ must remain in the simplex, we use a hit-and-run
        proposal. Second, it updates $\tau_r$. Since $\tau_r$ is a single number in $[0,1]$, we
        propose a small random change to its current value and reflect proposals that fall outside
        $[0,1]$ back into the interval. Third, it updates $\gamma$. Since $\gamma$ is restricted to
        the finite grid $\Gamma$, we compute the posterior probability of each grid value
        conditional on the current $\omega$ and $\tau_r$, and then sample a new value of $\gamma$
        from this finite distribution.
    \end{itemize}
\end{enumerate}

\begin{table}[h]
\centering
\caption{Hyperparameters for the BALD experiments.}
\label{tab:bald-hyperparams}
\begin{tabular}{lll}
\toprule
Symbol & Meaning & Value \\
\midrule
$d$ & feature dimension & $5$ \\
$T$ & query budget & $100$ \\
$C$ & candidate-pool size & $50$ \\
$N_{\mathrm{post}}$ & retained posterior samples & $200$ \\
$N_{\mathrm{BALD}}$ & posterior samples per BALD score & $50$ \\
$B$ & MCMC burn-in samples & $200$ \\
$\alpha_{\omega^*}$ & true-weight concentration & $0.2$ \\
$\mathcal D_q$ & candidate-query distribution & $\psi(x,y),\psi(x,y') \stackrel{\mathrm{i.i.d.}}{\sim}\mathrm{Unif}([0,1]^d)$ \\
$\beta$ & logistic inverse temperature & $10$ \\
$(\tau_r^*,\tau_\kappa^*)$ & true threshold grid & $\{0,0.2,0.4,0.6,0.8\}^2$ \\
$K_{\mathrm{GMM}}$ & \# GMM noise components & $3$ \\
$[\sigma_{\min},\sigma_{\max}]$ & GMM scale bounds & $[0.05,5]$ \\
GMM proposal sizes & proposal sizes for learning $(\mu,\log\sigma)$ & $(0.3,0.3)$ \\
$N_{\mathrm{BALD}}^\dagger$ & BALD samples for GMM variants & $30$ \\
$K_\gamma$ & conflict-threshold grid size & $10$ \\
$\gamma_{\max}$ & conflict-threshold grid max & $0.95$ \\
$\mathrm{sd}_{\tau_r}$ & proposal size for learning $\tau_r$ & $0.1$ \\
\bottomrule
\end{tabular}
\end{table}

The proposal-size hyperparameters control the size of the random moves used by the MCMC sampler.
For example, when updating $\tau_r$, the sampler proposes a new value by adding a small random
perturbation to the current value:
$\tau_r'=\mathrm{reflect}_{[0,1]}(\tau_r+\epsilon)$, where
$\epsilon\sim\mathcal N(0,\mathrm{sd}_{\tau_r}^2)$. Here
$\mathrm{reflect}_{[0,1]}$ maps a real number back into $[0,1]$ by bouncing it off the endpoints:
for example, $1.05$ is mapped to $0.95$, and $-0.05$ is mapped to $0.05$. This reflection only
ensures that the proposed value is valid. The proposal is still accepted or rejected afterward
using the usual Metropolis acceptance probability. Thus, reflection handles feasibility, while
Metropolis rejection handles posterior plausibility. The value $\mathrm{sd}_{\tau_r}$ controls how
large the proposed moves in $\tau_r$ tend to be. Similarly, the GMM proposal sizes control the
size of the random moves proposed for the mixture means $\mu$ and log standard deviations
$\log\sigma$.

All experiments in Sections~\ref{sec:accuracy} and~\ref{sec:speed}, unless otherwise noted, use
$40$ independently sampled true oracles $\omega^\ast\sim\mathrm{Dirichlet}(0.2 \,\mathbf 1)$ in
$d=5$, logistic inverse temperature $\beta=10$, candidate-pool size $C=50$,
$N_{\mathrm{post}}=200$ retained posterior samples, $B=200$ warm-started burn-in samples, and
BALD acquisition. In Section \ref{sec:accuracy}, we learn the $\widehat \omega$'s using the \textsc{Correct} version of the BALD algorithm detailed in Appendix \ref{app:algo-details}. Because we are trying to test how bad our learning and regret are when we assume the typical Bradley-Terry response model, the learner assumes  \[
        R(\cdot\mid q,\omega)
        =
        R^\circ_{h^{\mathrm{logit}}_\beta;\tau_r,\tau_\kappa}
        (\cdot\mid q;M_\omega),
        \]
The true response models $R^*$ are detailed in the above section. To ensure indecision is present but not overwhelming, we learn all $\widehat \omega$'s under the regime $\tau_r=\tau_\kappa=0.25$. Since we are not testing speed of learning in this section, we run all learners to convergence: $|\widehat \omega_{t+1} - \widehat \omega_t| < \varepsilon=0.01$ for 5 consecutive queries. The same random seeds for the BALD acquisition steps are used across all four true response models $R^*$'s for every $\omega^*$ in the 40 runs, so the four methods are directly comparable.

\begin{remark}[Consistency of the finite-pool BALD approximation]
\label{rem:consistency}
The ideal BALD rule would compute the exact BALD score of every query
$q\in\mathcal Q^{pc}$ and ask a query with maximal score. Our implementation
approximates this rule in two ways: it estimates each BALD score using finitely many posterior
samples, and it only calculates the BALD score for a finite pool of
candidate queries.

At round $t$, the learner draws a candidate pool
\[
\mathcal C_t=\{q_t^{(1)},\ldots,q_t^{(C)}\},
\qquad
q_t^{(c)}\stackrel{\mathrm{i.i.d.}}{\sim}\mathcal D_q,
\]
where $\mathcal D_q$ is the distribution used to sample candidate pairwise
queries. The learner then estimates the BALD score of each query in
$\mathcal C_t$ by averaging over $N_{\mathrm{BALD}}$ posterior samples, and asks
the query with the largest estimated score.

First consider the posterior-sampling approximation. For a fixed finite pool
$\mathcal C_t$, the exact BALD scores involve expectations over the learner's
posterior. Our implementation replaces these expectations with empirical averages over
posterior samples. This is a standard sample-average approximation
to an acquisition objective \cite{houlsby2011bayesian,kim2015guide,wilson2018maximizing}. If these posterior samples are representative of the posterior, then these Monte
Carlo score estimates converge to the exact BALD scores as $N_{\mathrm{BALD}}\to\infty$.

Now for the finite-pool approximation. Fix $\varepsilon>0$, and let
$A_{\varepsilon,t}$ be the set of queries whose exact BALD score is within
$\varepsilon$ of the best possible score at round $t$:
\[
A_{\varepsilon,t}
=
\left\{
q\in\mathcal Q^{pc}:
\mathrm{BALD}_t(q)
\ge
\sup_{q'\in\mathcal Q^{pc}}\mathrm{BALD}_t(q')-\varepsilon
\right\}.
\]
Suppose the query-sampling distribution can actually sample such near-optimal
queries, meaning
$
\mathcal D_q(A_{\varepsilon,t})>0.
$
Then the probability that none of the $C$ independently sampled candidates lies
in $A_{\varepsilon,t}$ is
$
\left(1-\mathcal D_q(A_{\varepsilon,t})\right)^C,
$
and the probability that the candidate pool contains at least one
$\varepsilon$-near-optimal query is
\[
1-\left(1-\mathcal D_q(A_{\varepsilon,t})\right)^C,
\]
which converges to $1$ as $C\to\infty$.

Thus, increasing $N_{\mathrm{BALD}}$ improves the accuracy of the BALD score estimates, while increasing $C$ improves coverage of the query space. If the
candidate-query distribution can sample near-optimal queries, our approximation of BALD is consistent.
\end{remark}

\subsection{Learning $\tau_r$ and $\tau_\kappa$}
\label{app:learn-tau}

This section details the posterior update used by
\textsc{Utilize-4}$^\diamond$, the version of \textsc{Utilize-4} that does not
assume the response thresholds are known. Instead of learning only the priority
weights $\omega$, the learner jointly estimates
$
\vartheta=(\omega,\tau_r,\gamma).
$
Here $\tau_r$ is the total-evidence threshold, and $\gamma$ is the learner's
estimate of the conflict threshold $\tau_\kappa$.

For a query $q$, recall that
\[
r_{M_\omega}(q)=\sum_{j=1}^d \omega_j|\delta_j(q)|,
\qquad
\kappa_{M_\omega}(q)=\sum_{j=1}^d \omega_j\delta_j(q).
\]
Both quantities are linear functions of $\omega$. The two thresholds enter the
response model differently. The threshold $\tau_r$ is a single cutoff for total
evidence, so we learn it as a continuous parameter in $[0,1]$. By contrast,
$\tau_\kappa$ appears in comparisons of the form
$
|\kappa_{M_\omega}(q)| < \tau_\kappa r_{M_\omega}(q).
$
Because $\tau_\kappa$ is multiplied by $r_{M_\omega}(q)$, learning
$\tau_\kappa$ continuously together with $\omega$ would introduce bilinear
terms. To keep the posterior update simple and stable, we learn the conflict
threshold over a finite grid. Specifically, the learner chooses
$
\gamma\in\Gamma,
$
where
\[
\Gamma=
\left\{
0,\frac{\gamma_{\max}}{K_\gamma-1},
\frac{2\gamma_{\max}}{K_\gamma-1},
\ldots,\gamma_{\max}
\right\}.
\]
In the experiments, $\gamma_{\max}=0.95$ and $K_\gamma=10$. Thus, $\gamma$ is
chosen from ten evenly spaced candidate values between $0$ and $0.95$. Note that this grid does not actually perfectly overlap with the true values of $\tau_\kappa$.

At active-learning round $t$, the learner has observed a transcript
$
\mathcal T_t=\{(q_s,\triangleright_s)\}_{s=1}^t,
$
where $q_s$ is the query asked in round $s$ and $\triangleright_s$ is the
observed response. For any candidate parameter value
$(\omega,\tau_r,\gamma)$, the likelihood of this transcript is
\[
\mathcal L_t(\omega,\tau_r,\gamma)
=
\prod_{s=1}^t
R^\circ_{h_\beta;\tau_r,\gamma}(\triangleright_s\mid q_s;M_\omega).
\]
The prior factorizes across the three learned quantities:
$
\pi_0(\omega,\tau_r,\gamma)
=
\pi_0(\omega)\pi_0(\tau_r)\pi_0(\gamma).
$
We use a uniform prior over the simplex for $\omega$ and a uniform prior over
$[0,1]$ for $\tau_r$.
For $\gamma$, we put a prior directly on the grid $\Gamma$:
\[
\pi_0(\gamma)\propto [u(1-u)]^{1/2},
\qquad
u=\gamma/\gamma_{\max}.
\]
This prior assigns zero mass to the two grid endpoints and positive mass to interior grid values.
Thus, it excludes the degenerate boundary cases $\gamma=0$ and $\gamma=\gamma_{\max}$ while
placing a weak symmetric prior over the remaining interior grid values.

The posterior after transcript $\mathcal T_t$ is therefore
\[
\pi_t(\omega,\tau_r,\gamma)
\propto
\pi_0(\omega)\pi_0(\tau_r)\pi_0(\gamma)
\mathcal L_t(\omega,\tau_r,\gamma).
\]
Algorithm~\ref{alg:learn-tau} describes one MCMC sweep (a single increment in $t$) for sampling from this
posterior. A sweep means one pass through the three components of the parameter:
first update $\omega$, then update $\tau_r$, then update $\gamma$. In the BALD
template, Line~(6) runs many such sweeps: the first $B$ sweeps are discarded as
burn-in, and the next $N_{\mathrm{post}}$ states are retained as posterior
samples.

\begin{algorithm}[h!]
\caption{Threshold-Learning Posterior Update (one MCMC sweep)}
\label{alg:learn-tau}
\begin{tabbing}
\qquad \= \qquad \= \kill
\textbf{Input:} transcript $\mathcal T_t=\{(q_s,\triangleright_s)\}_{s=1}^t$; current state
$(\omega,\tau_r,\gamma)_t$; grid $\Gamma=\{\gamma_1,\ldots,\gamma_{K_\gamma}\}$; \\
\quad prior $\pi_0$; link $h_\beta$; proposal size sd$_{\tau_r}$ \\[0.5em]

\textbf{(1) Update $\omega$.} \\[0.25em]
\> Propose a new priority-weight vector $\omega'$ using a hit-and-run step on \\
\> $\Omega=\{\omega\in\mathbb R^d_{\ge0}:\|\omega\|_1=1\}$. \\[0.25em]
\> Accept $\omega'$ with probability \\
\> $\displaystyle
a_\omega
=
\min\!\left\{
1,\
\frac{
\pi_0(\omega')\mathcal L_t(\omega',\tau_r,\gamma)
}{
\pi_0(\omega)\mathcal L_t(\omega,\tau_r,\gamma)
}
\right\}.
$ \\[0.25em]
\> If the proposal is rejected, keep the current value of $\omega$. \\[0.75em]

\textbf{(2) Update $\tau_r$.} \\[0.25em]
\> Propose a small random change to $\tau_r$: \\
\> $\displaystyle
\tau_r'=\mathrm{reflect}_{[0,1]}(\tau_r+\epsilon),
\qquad
\epsilon\sim\mathcal N(0,\mathrm{sd}_{\tau_r}^2).
$ \\[0.25em]
\> The reflection keeps the proposal inside the interval $[0,1]$. \\[0.25em]
\> Accept $\tau_r'$ with probability \\
\> $\displaystyle
a_\tau
=
\min\!\left\{
1,\
\frac{
\pi_0(\tau_r')\mathcal L_t(\omega,\tau_r',\gamma)
}{
\pi_0(\tau_r)\mathcal L_t(\omega,\tau_r,\gamma)
}
\right\}.
$ \\[0.25em]
\> If the proposal is rejected, keep the current value of $\tau_r$. \\[0.75em]

\textbf{(3) Update $\gamma$.} \\[0.25em]
\> Since $\gamma$ lies on the finite grid $\Gamma$, compute one posterior weight per grid point: \\[0.25em]
\> $\displaystyle
a_k=\pi_0(\gamma_k)\mathcal L_t(\omega,\tau_r,\gamma_k),
\qquad
k=1,\ldots,K_\gamma.
$ \\[0.25em]
\> Normalize these weights: \\
\> $\displaystyle
p_k=\frac{a_k}{\sum_{\ell=1}^{K_\gamma}a_\ell}.
$ \\[0.25em]
\> Draw the new value of $\gamma$ from the grid distribution 
$\Pr(\gamma=\gamma_k\mid\omega,\tau_r,\mathcal T_t)=p_k.
$ \\[0.5em]

\textbf{Output:} updated state $(\omega,\tau_r,\gamma)_{t+1}$.
\end{tabbing}
\end{algorithm}

The three updates have simple roles. The update for $\omega$ moves around the
simplex of valid priority weights. The update for $\tau_r$ makes a small
continuous move in the interval $[0,1]$. The update for $\gamma$ checks all grid
values and resamples $\gamma$ according to their posterior probabilities. Because
$\gamma$ is resampled inside the same Markov chain as $\omega$ and $\tau_r$, the
posterior samples represent uncertainty about all three quantities jointly.

In the Bayesian active learning step, the learner uses these posterior samples
exactly as in Algorithm~\ref{alg:bald-template}, except that each posterior draw
now contains $(\omega,\tau_r,\gamma)$ rather than only $\omega$. For a candidate
query $q$, define
\[
\rho_{\omega,\tau_r,\gamma,q}(\triangleright)
=
R^\circ_{h_\beta;\tau_r,\gamma}(\triangleright\mid q;M_\omega).
\]
The BALD score is
\[
\mathrm{BALD}_t(q)
=
\mathrm{Ent}\!\left(
\mathbb E_{(\omega,\tau_r,\gamma)\sim\pi_t}
[\rho_{\omega,\tau_r,\gamma,q}]
\right)
-
\mathbb E_{(\omega,\tau_r,\gamma)\sim\pi_t}
\left[
\mathrm{Ent}(\rho_{\omega,\tau_r,\gamma,q})
\right].
\]
In implementation, these expectations are approximated by averaging over the
posterior samples retained from the MCMC chain, as described in
Algorithm~\ref{alg:bald-template}.

After $T$ rounds, the learner reports the posterior means
\[
\widehat\omega_T=\mathbb E_{\pi_T}[\omega],
\qquad
\widehat\tau_{r,T}=\mathbb E_{\pi_T}[\tau_r],
\qquad
\widehat\tau_{\kappa,T}=\mathbb E_{\pi_T}[\gamma].
\]

\subsection{Formal Indecision Response Models}
\label{app:formal-indecision}
In all these definitions, fix thresholds $\tau_r\in[0,1]$, $\tau_\kappa\in[0,1)$, and a link function
$h_\beta$.

\begin{definition}[\textbf{50/50 Response Model}]
\label{def:random-under-indecision}
The 50/50 query response model
\[
R^{\mathrm{50/50}}_{h_\beta;\tau_r,\tau_\kappa}
:
\mathcal M
\to
\left(\mathcal Q^{pc}\to\Delta(\mathcal W^{pc})\right)
\]
is defined by
\[
R^{\mathrm{50/50}}_{h_\beta;\tau_r,\tau_\kappa}(q;M)
=
\begin{cases}
R^\circ_{h_\beta;\tau_r,\tau_\kappa}(q;M),
&\text{if } L_M(q)\in\{y\succ^\ast y',\,y\prec^\ast y'\},\\[4pt]
\mu^{\mathrm{50/50}}(q),
&\text{if } L_M(q)\in\{y\sim^\ast y',\,y\bowtie^\ast y'\},
\end{cases}
\]
where $\mu^{\mathrm{50/50}}(q)\in\Delta(\mathcal W^{pc}(q))$ is the distribution
given by
\[
\mu^{\mathrm{50/50}}(q)(y\succ y')=\frac12,
\qquad
\mu^{\mathrm{50/50}}(q)(y\prec y')=\frac12,
\qquad
\mu^{\mathrm{50/50}}(q)(y\sim y')
=
\mu^{\mathrm{50/50}}(q)(y\bowtie y')
=
0.
\]
Thus, when the latent state is indecisive, the respondent chooses between the two decisive responses uniformly at random.
\end{definition}

\begin{definition}[\textbf{Lexicographic Response Model}]
\label{def:lex}

The lexicographic query response model
\[
R^{\mathrm{lex}}_{h_\beta;\tau_r,\tau_\kappa}
:
\mathcal M
\to
\left(\mathcal Q^{pc}\to\Delta(\mathcal W^{pc})\right)
\]
is defined by
\[
R^{\mathrm{lex}}_{h_\beta;\tau_r,\tau_\kappa}(q;M)
=
\begin{cases}
R^\circ_{h_\beta;\tau_r,\tau_\kappa}(q;M),
&\text{if } L_M(q)\in\{y\succ^\ast y',\,y\prec^\ast y'\},\\[4pt]
\mu^{\mathrm{lex}}(q),
&\text{if } L_M(q)\in\{y\sim^\ast y',\,y\bowtie^\ast y'\}.
\end{cases}
\]
The lexicographic distribution
$\mu^{\mathrm{lex}}(q)\in \Delta(\mathcal W^{pc}(q))$ is defined as follows. 

Fix a true model $M^*=(u,\omega^*)$. 
Let $\lambda^{\omega^*}=(\lambda_1,\ldots,\lambda_m)$ be any ordering of priorities such that
\[
\omega^*_{\lambda_1}\ge \omega^*_{\lambda_2}\ge \cdots \ge \omega^*_{\lambda_m},
\]
with ties broken arbitrarily but fixed throughout. 

For any priority $j$ and query $q$, let $s_j(q) = \phi^{\mathrm{rules}}\!\left(
(\Delta^F_j(q))_{F\in\mathcal F}
\right)$ be the evidence from priority $j$ on $q$. 

We say that priority $j$ is \emph{decisive on $q$} if $s_j(q)\neq 0$. The \emph{highest-ranked decisive priority} on $q$ is the first priority in the weight-induced ordering $\lambda^{\omega^*}$ that is decisive on $q$. Formally, define
\[
\ell_M^{\mathrm{lex}}(q)
:=
\min\left\{
\ell\in[m]:
s_{\lambda_\ell}(q)\neq 0
\right\}.
\]
When $\ell_M^{\mathrm{lex}}(q)$ is defined (i.e., the set is non-empty), let
$
j^*(q):=\lambda_{\ell_M^{\mathrm{lex}}(q)}
$
denote the highest-ranked decisive priority. Then, the response distribution follows the sign of the
aggregate directional evidence of the highest-ranked decisive priority:
\[
\mu^{\mathrm{lex}}(q)(y\succ y')
=
\begin{cases}
1 & \text{if } s_{j^*(q)}(q)>0,\\
0 & \text{if } s_{j^*(q)}(q)<0,
\end{cases}
\qquad
\mu^{\mathrm{lex}}(q)(y\prec y')
=
\begin{cases}
0 & \text{if } s_{j^*(q)}(q)>0,\\
1 & \text{if } s_{j^*(q)}(q)<0.
\end{cases}
\]
If $\ell_M^{\mathrm{lex}}(q)$ is undefined, meaning no priority is decisive on $q$,
then the lexicographic response distribution chooses randomly between the two decisive responses:
\[
\mu^{\mathrm{lex}}(q)(y\succ y')=\frac{1}{2},
\qquad
\mu^{\mathrm{lex}}(q)(y\prec y')=\frac{1}{2}.
\]
In all cases,
\[
\mu^{\mathrm{lex}}(q)(y\sim y')=
\mu^{\mathrm{lex}}(q)(y\bowtie y')=0.
\]

Thus, on latently decisive queries, the individual follows the baseline response model. On latently indecisive queries, the individual falls back to the decisive recommendation of the highest-weight priority that has nonzero aggregate directional evidence on the query; if no priority has nonzero evidence, the individual chooses uniformly at random between $y\succ y'$ and $y\prec y'$.
\end{definition}

\begin{definition}[\textbf{Self-Similarity Response Model}]
\label{def:biased}
Let the individual's true feature vector be $v\in[0,1]^d$ and fix a query $q=(y,y';x)$.
Then, the self-similarity query response model is
\[
R^{\mathrm{self}}_{h_\beta;\tau_r,\tau_\kappa,v}
:
\mathcal M
\to
\left(\mathcal Q^{pc}\to\Delta(\mathcal W^{pc})\right)
\]
is defined by
\[
R^{\mathrm{self}}_{h_\beta;\tau_r,\tau_\kappa,v}(q;M)
=
\begin{cases}
R^\circ_{h_\beta;\tau_r,\tau_\kappa}(q;M),
&\text{if } L_M(q)\in\{y\succ^\ast y',\,y\prec^\ast y'\},\\[4pt]
\mu^v(q),
&\text{if } L_M(q)\in\{y\sim^\ast y',\,y\bowtie^\ast y'\}.
\end{cases}
\]
Where $\mu^v$ is defined as follows. Let the profile-distance gap be
\[
b_v(q)
:=
\|\psi(x,y')-v\|_2-\|\psi(x,y)-v\|_2.
\]
In words, $y$ is closer to the respondent's profile than $y'$ when $b_v(q)>0$ and farther when $b_v(q)<0$.
Let $\mu^v(q)\in\Delta(\mathcal W^{pc}(q))$ be the distribution
\[
\mu^v(q)(y\succ y')
=
\begin{cases}
1, & \text{if } b_v(q)>0,\\
\frac12, & \text{if } b_v(q)=0,\\
0, & \text{if } b_v(q)<0,
\end{cases}
\qquad
\mu^v(q)(y\prec y')
=
\begin{cases}
0, & \text{if } b_v(q)>0,\\
\frac12, & \text{if } b_v(q)=0,\\
1, & \text{if } b_v(q)<0,
\end{cases}
\]
and
\[
\mu^v(q)(y\sim y')
=
\mu^v(q)(y\bowtie y')
=
0.
\]
In words, when the latent state is indecisive, the respondent deterministically chooses the alternative
whose feature representation is closest in Euclidean distance to their assigned
profile, breaking exact ties uniformly at random.
\end{definition}

\begin{definition}[\textbf{Utilize-3 Response Model}]
\label{def:generic-indecision}
Let the coarsened pairwise-comparison response alphabet be
\[
\mathcal W^{pc,\oslash}(q)=\{\succ,\prec,\oslash\},
\]
where $\oslash$ denotes generic indecision, without distinguishing between
indifference and conflict.
The \textsc{Utilize-3} query response model
\[
R^{\oslash}_{h_\beta;\tau_r,\tau_\kappa}
:
\mathcal M
\to
\left(\mathcal Q^{pc}\to\Delta(\mathcal W^{pc,\oslash})\right)
\]
is obtained from the
baseline model $R^\circ_{h_\beta;\tau_r,\tau_\kappa}$ by
coarsening the two indecisive responses into a single generic-indecision
response. The decisive probabilities are unchanged,
\[
R^{\oslash}_{h_\beta;\tau_r,\tau_\kappa}(q;M)(y\succ y')
=
R^\circ_{h_\beta;\tau_r,\tau_\kappa}(q;M)(y\succ y'),
\quad
R^{\oslash}_{h_\beta;\tau_r,\tau_\kappa}(q;M)(y\prec y')
=
R^\circ_{h_\beta;\tau_r,\tau_\kappa}(q;M)(y\prec y'),
\]
while the generic-indecision response pools the two indecisive states,
\[
R^{\oslash}_{h_\beta;\tau_r,\tau_\kappa}(q;M)(\oslash)
=
R^\circ_{h_\beta;\tau_r,\tau_\kappa}(q;M)(y\sim y')
+
R^\circ_{h_\beta;\tau_r,\tau_\kappa}(q;M)(y\bowtie y').
\]
\end{definition}

\subsection{Additional Experiments for Section \ref{sec:accuracy}} \label{app:moreacc}
\subsubsection{Exploring directional biases in $\widehat{\omega}$}
\Cref{fig:biasdecomp} shows the plots demonstrating the expected biases in learned $\widehat{\omega}$ under the three deviating response conditions from \Cref{sec:accuracy}. Each plotted quantity is first computed
within a single simulation run, using that run's true weights \(\omega^*\),
learned weights \(\widehat\omega\), and, in one case, the self
vector \(v\). We then report the mean and standard error over the 40 runs.

\begin{figure}[!]
    \centering
    \includegraphics[width=0.8\linewidth]{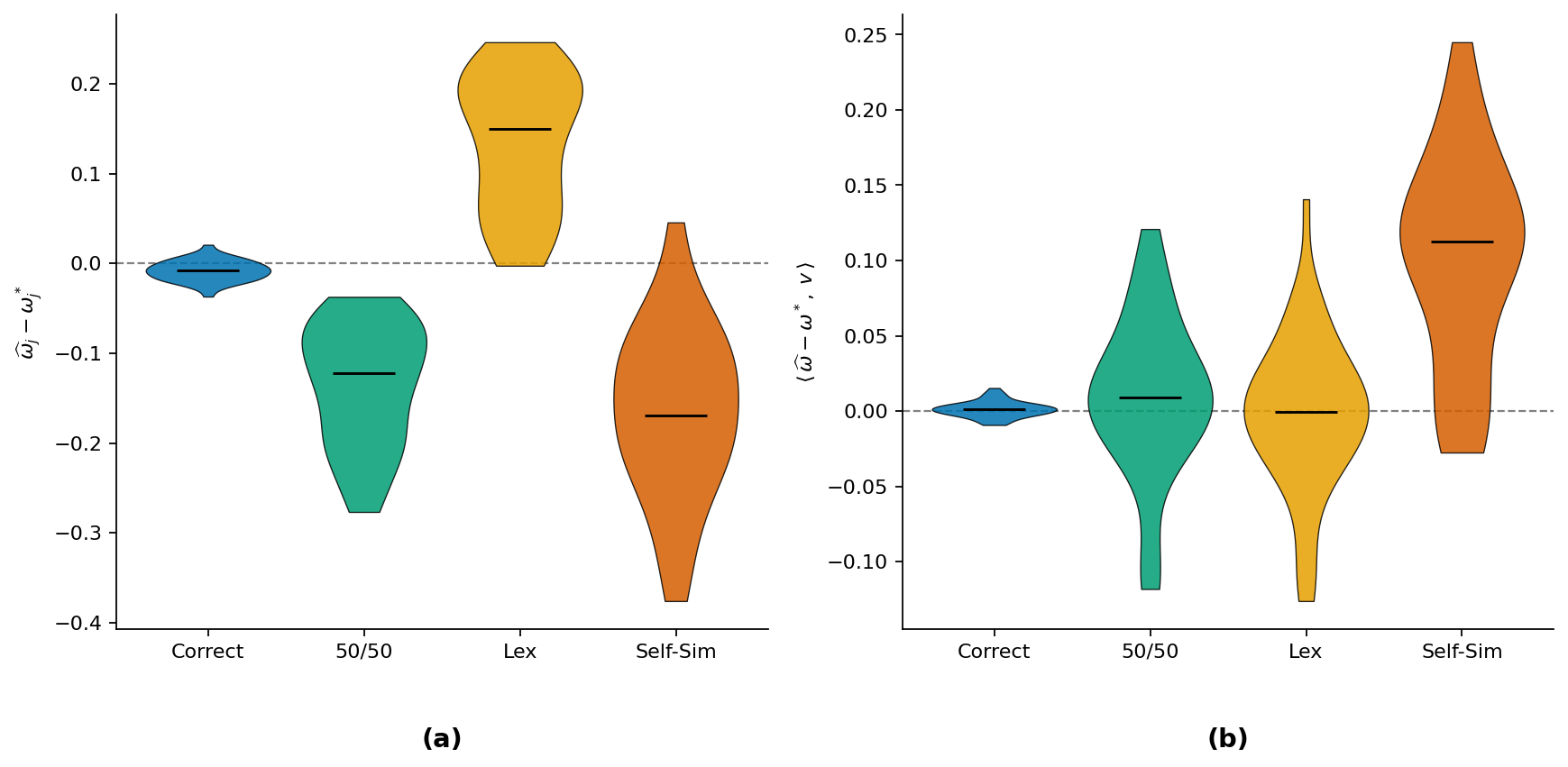}
    \caption{\textbf{Decomposing the $\ell_1$ weight-recovery error across response conditions.} Means and standard errors computed over 40 runs. \textbf{(a)} Bias toward own top feature, \textbf{(b)} Bias toward own self vector.}
\label{fig:biasdecomp}
\end{figure}

Panel (a) measures bias toward the true top priority: Let $j^*\in\arg\max_j \omega^*_j$.
For each run, we compute $\widehat\omega_{j^*}-\omega^*_{j^*}$; 
positive values mean that the learned model overweights the priority that was
already most important under \(\omega^*\). This is the distortion expected under
\textsc{Lexicographic} behavior.

Panel (b) measures bias toward the respondent's self vector. For each run, we $\langle \widehat\omega-\omega^*,v\rangle.$
Positive values mean that the learned model shifts weight toward features that are large in the respondent's own self vector. This is the distortion expected under \textsc{Self-Similarity} behavior.

\subsubsection{Cosine Similarity}
In \Cref{fig:cosine}, we show the cosine similarity of $\widehat{\omega}$ and $\omega^*$ to demonstrate the directional recovery conditioned on the four types of true response models.

\begin{figure}[!]
    \centering
    \includegraphics[width=0.4\linewidth]{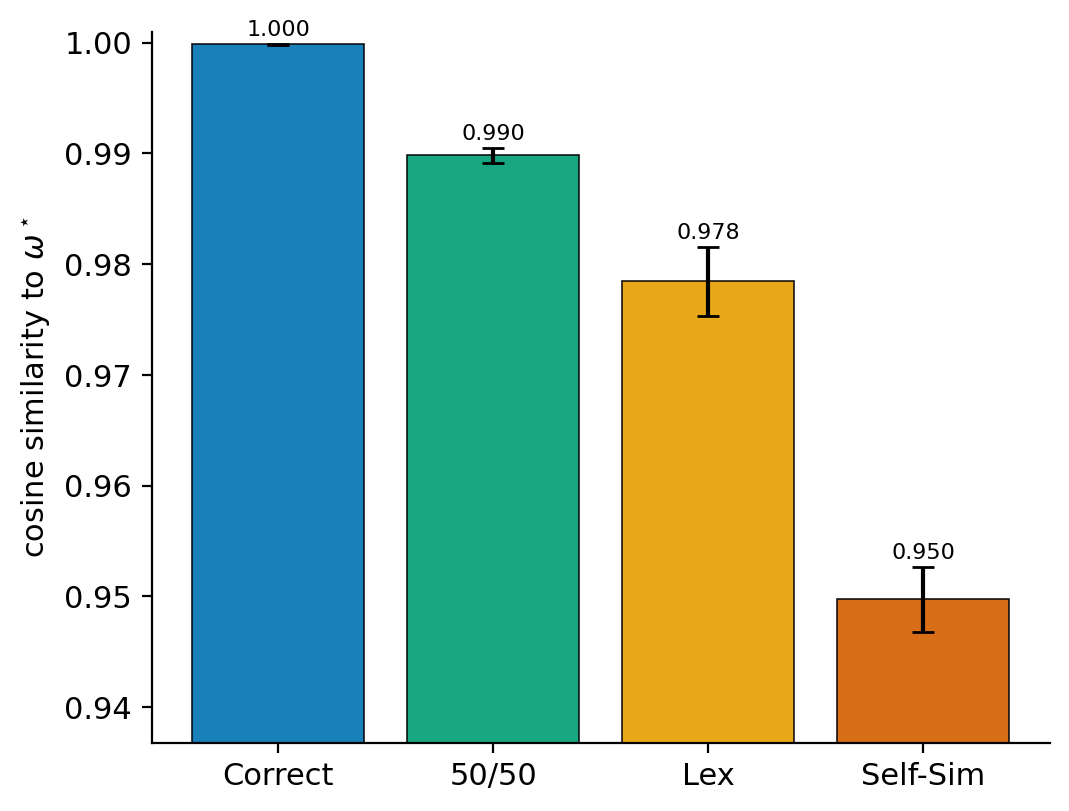}
    \caption{
    \textbf{Cosine similarity between the learned weights \(\widehat\omega\) and true
    weights \(\omega^*\).} Error bars show \(\pm 1\) standard error over 40 runs.
    }
    \label{fig:cosine}
\end{figure}

\subsection{Additional Experimental Results for Section \ref{sec:speed}}
\label{app:big-grid}
We now present the full $\tau$-grids behind the learning speed results of
\Cref{sec:speed}. The eight methods are \textsc{Utilize-4}, \textsc{Utilize-3},
\textsc{Utilize-4}$^\dagger$,
\textsc{Utilize-4}$^\diamond$, \textsc{Ignore}, \textsc{Ignore}$^\dagger$,
\textsc{Correct}, and \textsc{Correct}$^\dagger$.
\subsubsection{Unknown Noise Variants} 
We include unknown-noise variants of \textsc{Utilize} and \textsc{Ignore},
denoted with a $\dagger$. In these variants the individual's true response model
still uses logistic noise, but the learner is not given the logistic link or its
scale; it knows only that the noise distribution is a three-component mixture of
Gaussians,
$
\varepsilon \sim \sum_{k=1}^3 w_k\,\mathcal N(\mu_k,\sigma_k^2)$ where $\sum_{k=1}^3 w_k=1.$
This family is far more flexible than the one-parameter logistic family
$F_{\mathrm{logistic}}(t)=1/(1+\exp(-\beta t))$. It contains the probit link
exactly --- a single Gaussian component recovers $\Phi(t/\sigma)$ --- and, because
finite Gaussian mixtures approximate any continuous noise CDF arbitrarily well as
the number of components grows, it approximates the logistic link. It does not reproduce the logistic CDF exactly: a finite mixture
has Gaussian tails, lighter than the logistic's exponential tails. The point of
the $\dagger$ variants is exactly this misspecification --- the oracle's noise is
logistic while the learner fits a Gaussian mixture --- so that the experiment tests
whether the added flexibility lets the learner recover $\omega$ without being told
the true noise family. The unknown-noise learner thus jointly infers the
preference weights and a model of the response noise. We test these variants to demonstrate that the learning gains achieved by \textsc{Utilize-3} and \textsc{Utilize-4} remain \textit{even if we weaken the learner's assumptions about the noise structure}, which is potentially useful if people are not responding according to the specific noise assumptions assumed by Bradley-Terry.

\begin{figure}[!]
    \centering
    \includegraphics[width=\linewidth]{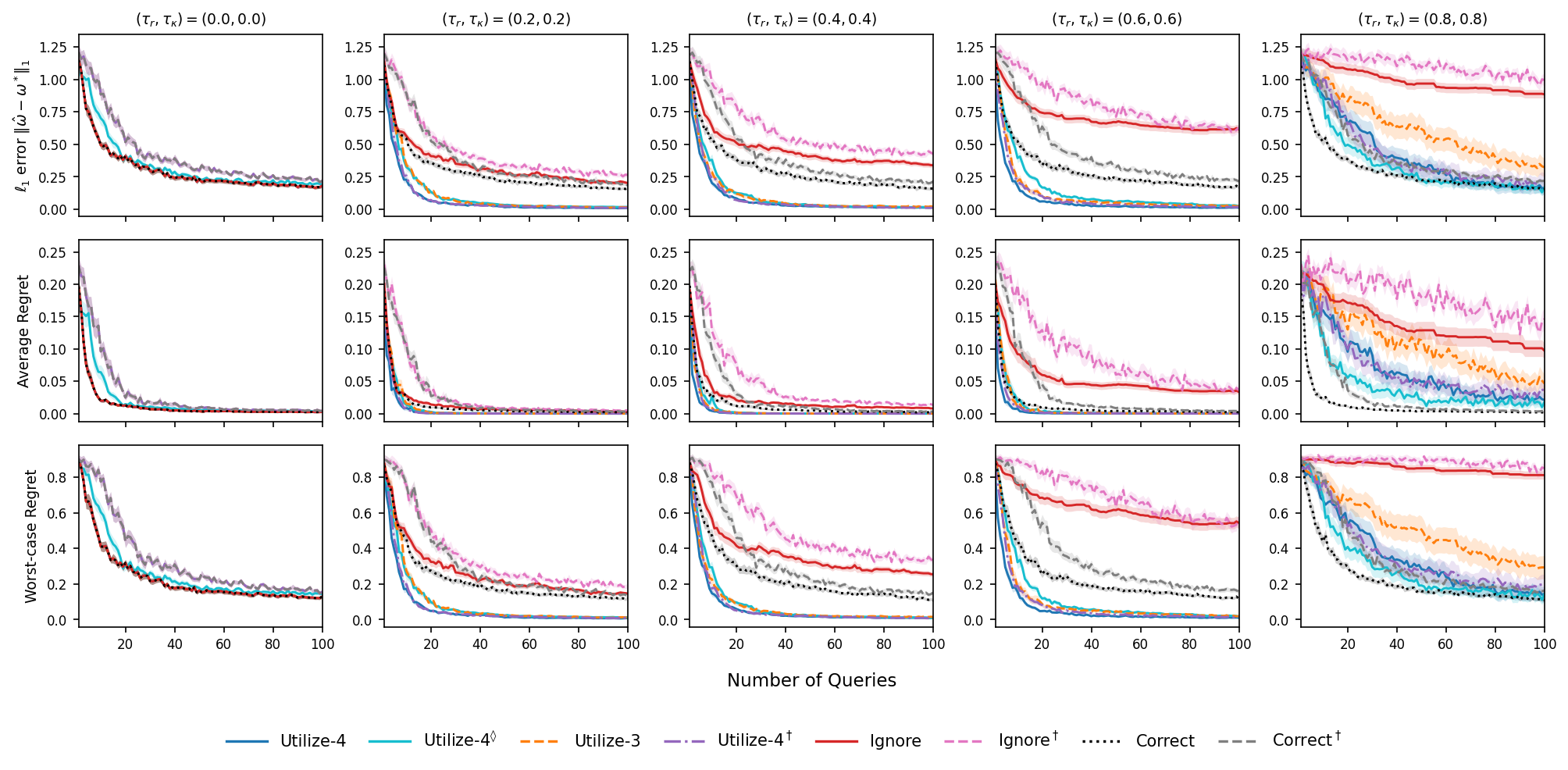}
    \caption{Performance of all methods across the diagonal threshold regimes versus number of queries.
    \textbf{Top:} $\ell_1$ error $\|\widehat\omega-\omega^\ast\|_1$. \textbf{Middle:} Average regret on the
    uniform-$[0,1]^5$ distribution with independently drawn features. \textbf{Bottom:}
    Worst-case single-decision regret.}
    \label{fig:diag-summary}
\end{figure}

\begin{figure}[!]
    \centering
    \includegraphics[width=\linewidth]{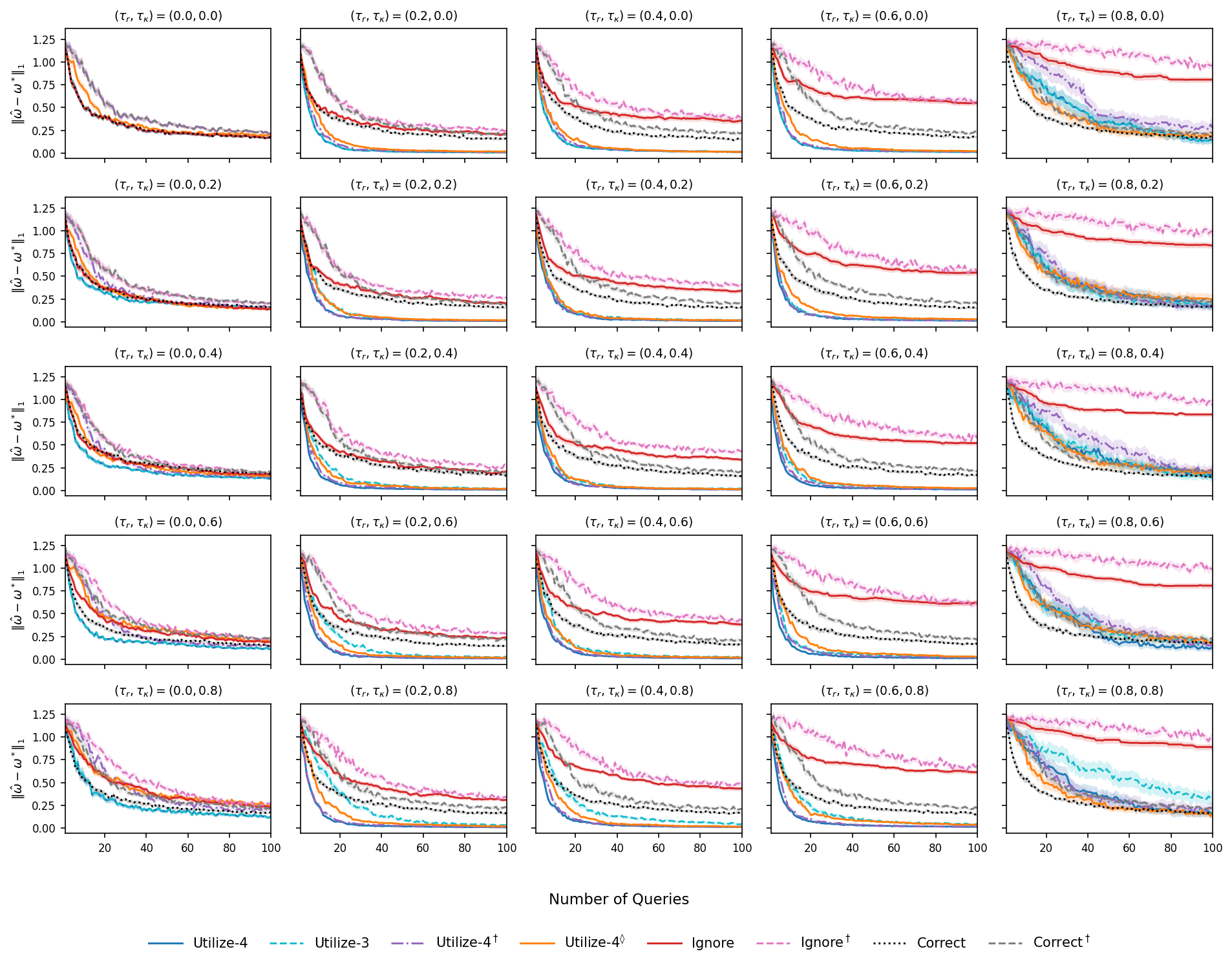}
    \caption{$\ell_1$ error $\|\widehat\omega-\omega^\ast\|_1$ versus number of
    queries, across the full $5\times5$ grid of thresholds.}
    \label{fig:grid-l1}
\end{figure}

\begin{figure}[!]
    \centering
    \includegraphics[width=\linewidth]{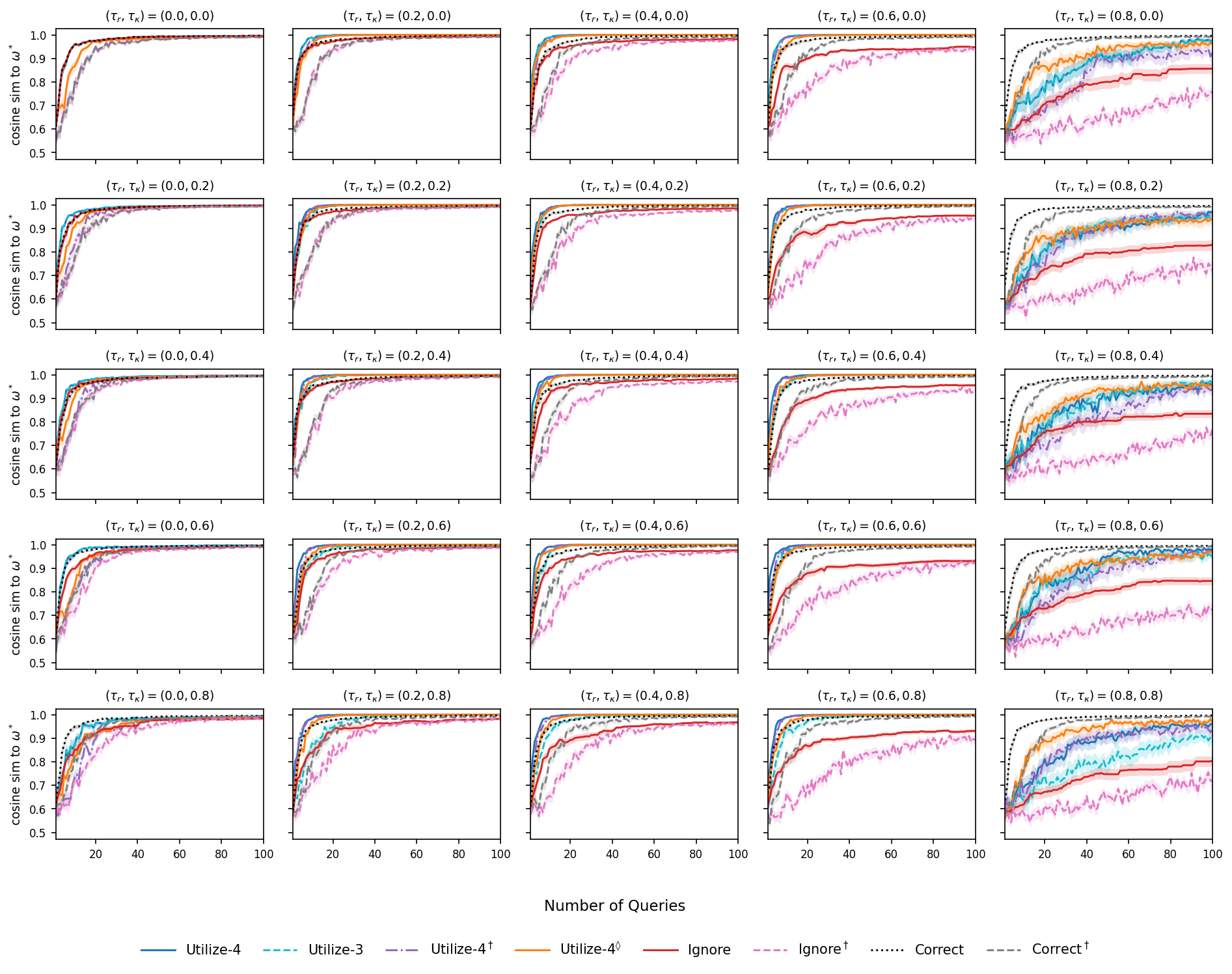}
    \caption{Cosine similarity $\cos(\widehat\omega,\omega^\ast)$ versus number of
    queries, across the full $5\times5$ threshold grid.}
    \label{fig:grid-cos}
\end{figure}

\begin{figure}[!]
    \centering
    \includegraphics[width=\linewidth]{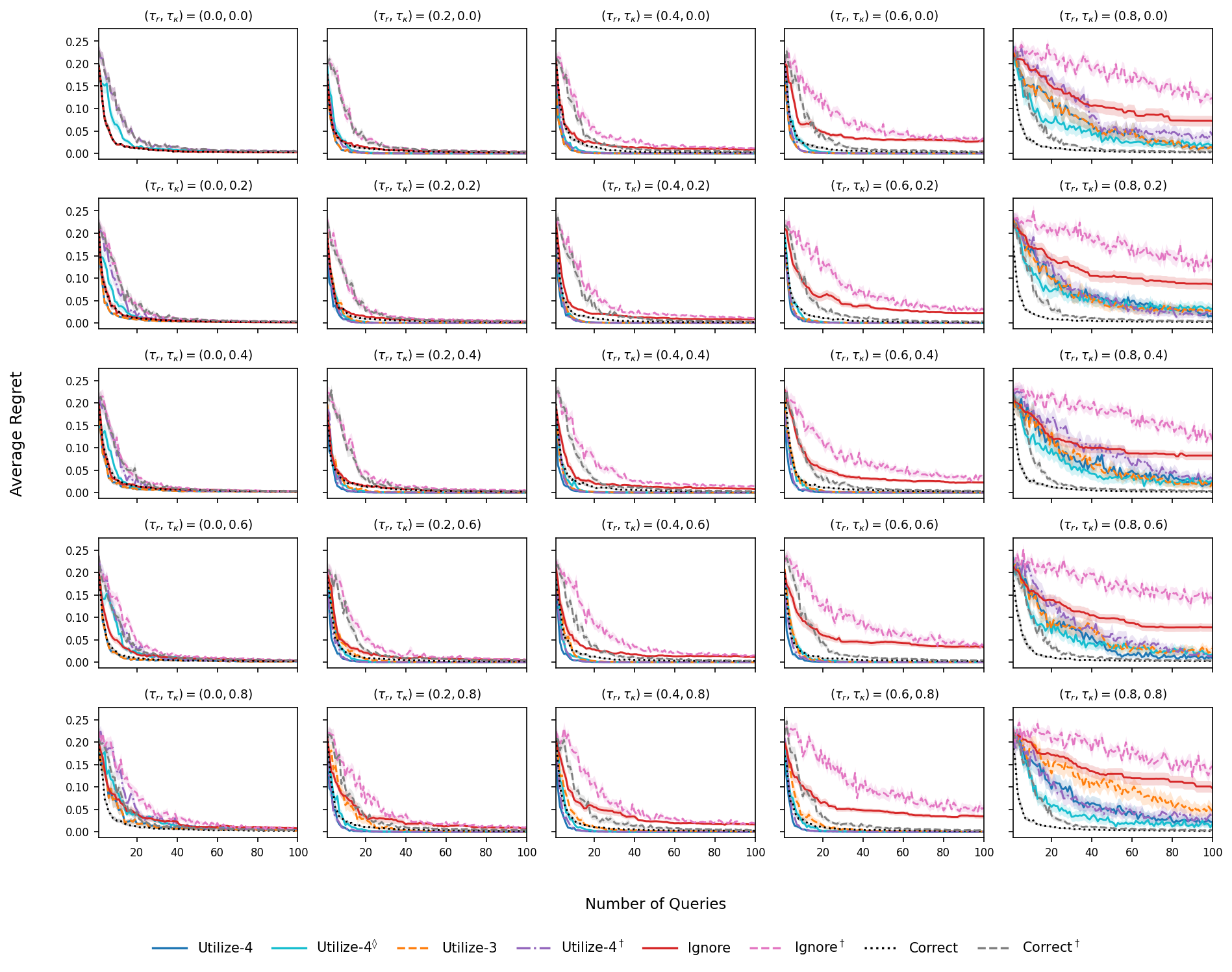}
    \caption{Average regret versus query number on the uniform-$[0,1]^5$ distribution
     where features are all independent, across the full $5\times5$ threshold grid.}
    \label{fig:grid-avgreg}
\end{figure}

\begin{figure}[!]
    \centering
    \includegraphics[width=\linewidth]{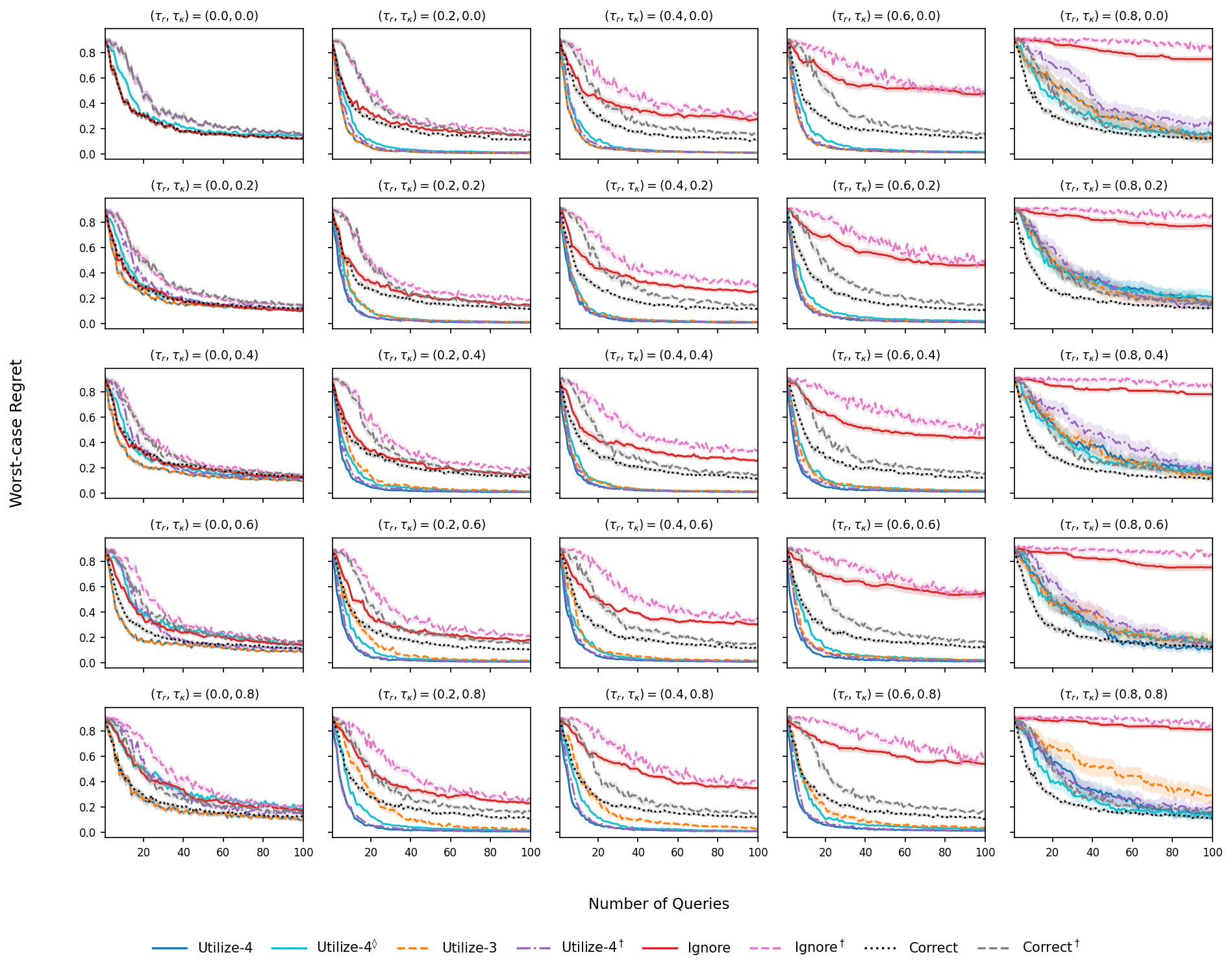}
    \caption{Worst-case single-decision regret versus query number, across the full
    $5\times5$ threshold grid.}
    \label{fig:grid-wcr}
\end{figure}

\begin{figure}[!]
    \centering
    \includegraphics[width=0.8\linewidth]{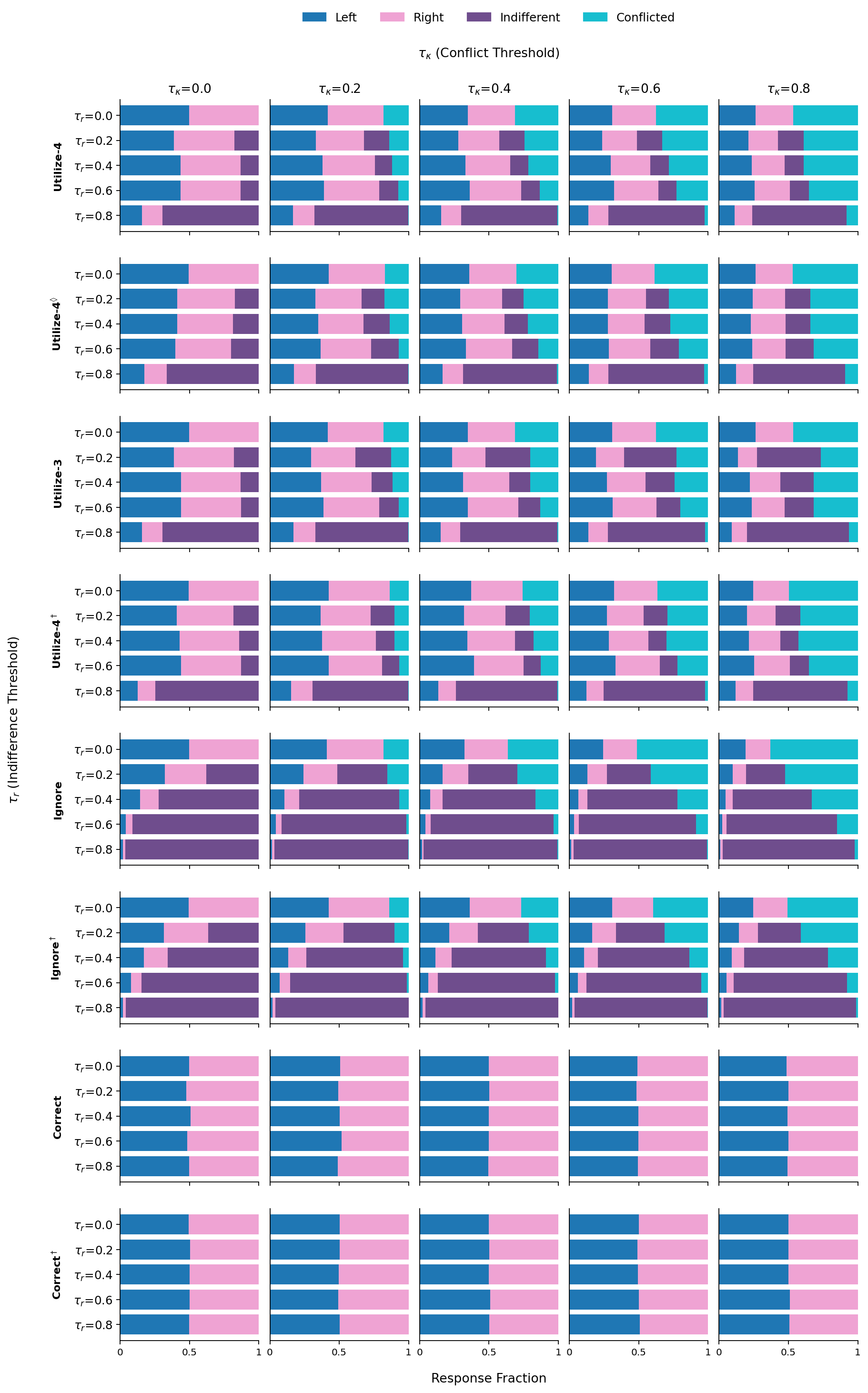}
    \caption{Fraction of each response type (Left $\succ$, Right $\prec$,
    Indifferent $\sim$, Conflict $\bowtie$) elicited by the active learner,
    across the full $5\times5$ threshold grid, for all eight methods. Within each
    method's block, rows index the indifference threshold $\tau_r$ and columns the
    incomparability threshold $\tau_\kappa$.}
    \label{fig:respdist-grid}
\end{figure}

\clearpage
\subsection{Rule-Level Preferences}
\label{app:rule-level-indecision}

The preceding results focus on the importance of measuring indecision in the technical task of learning the model $M$: forced comparisons can leave information on the table, or worse, distort what a learner
recovers about the individual's priorities if the individual deviates from the assumed model to resolve indecision. But even when the individual resolves indecision via the zero-threshold response model \(R^\circ_{h_\beta;0,0}\) (thus avoiding learning errors), forced comparisons still erase information that may be normatively important: truly decisive responses and forced responses from an indecisive latent state are recorded as identical, but their moral authority may not be, varying depending on whether the individual regarded the comparison as clear, morally weighty but conflicted, or too low-stakes to warrant a meaningful distinction.

When we learn a single rule that best imitates the individual's forced-choice behavior, we can also lose something else: information about where the individual is indifferent or conflicted between \textit{rules}. Making this claim requires us to define a global preference relation over rules that encompasses these forms of indecision, and then show that local pairwise comparison queries are informative about these rule-level relations. To do so, we define a global rule-level preference relation analogously to our latent states model at the query level.

\begin{definition}[Rule-Level Latent Relation]
\label{def:rule-level-latent-relation} 
Fix a priority model \(M=(u,\omega)\) and thresholds
\(\tau_r,\tau_\kappa\). For two rules \(F,F'\in\rules\), define the rule-level
directional evidence scores
\[
s^+_M(F,F')
=
\sum_{j\in[m]} \omega_j [u_j(F)-u_j(F')]_+,
\qquad
s^-_M(F,F')
=
\sum_{j\in[m]} \omega_j [u_j(F')-u_j(F)]_+.
\]
Let
\[
r_M(F,F')=s^+_M(F,F')+s^-_M(F,F'),
\qquad
\kappa_M(F,F')=s^+_M(F,F')-s^-_M(F,F').
\]
The rule-level latent relation between \(F\) and \(F'\), denoted
\(F \triangleright^*_M F'\), is defined by
\[
F \triangleright^*_M F'
=
\begin{cases}
F\succ^*_M F'
&
\text{if } r_M(F,F')-\tau_r\ge 0
\text{ and }
\kappa_M(F,F')-\tau_\kappa r_M(F,F')\ge 0,
\\[0.4em]
F\prec^*_M F'
&
\text{if } r_M(F,F')-\tau_r\ge 0
\text{ and }
-\kappa_M(F,F')-\tau_\kappa r_M(F,F')\ge 0,
\\[0.4em]
F\bowtie^*_M F'
&
\text{if } r_M(F,F')-\tau_r\ge 0
\text{ and }
|\kappa_M(F,F')|-\tau_\kappa r_M(F,F')<0,
\\[0.4em]
F\sim^*_M F'
&
\text{if } r_M(F,F')-\tau_r<0.
\end{cases}
\]
As in \Cref{def:threshold-response-model}, when \(\tau_\kappa=\kappa_M(F,F')=0\), the
decisive cases overlap; in that degenerate case, either decisive relation may
be assigned arbitrarily.
\end{definition}

Now, we show that in perfectly separable models, latent states at the query level correspond exactly to pairwise preference relations between the corresponding projections. This means that if we can learn the latent state thresholds $\tau_r,\tau_\kappa$ --- which are erased under forced choice --- we simultaneously learn the global preference relation described above, including regions of conflict and indifference between rules.\footnote{This equivalence also leads to another interpretation: that latent query states arise from an underlying rule-level preference relation, where the person picks an arbitrary background rule and compares the two projections, and then simply answers according to the relation between rules.} The proof follows directly from perfect separability. \footnote{When $M$ is not perfectly separable, one can still derive a relationship between rule preferences and query responses, one has to aggregate over background rules, which is exactly the role played by $\phi^{\mathrm{rules}}$.}

\begin{lemma}[Latent query states as rule-level relations]
\label{lem:query-rule-latent-equivalence}
Fix a perfectly separable model \(M\in \Msep\). For every query
\(q=(y,y';x)\), every \(F\in\rules\), and every latent state
\(\triangleright^*\in\{\succ^*,\prec^*,\sim^*,\bowtie^*\}\),
\[
L_M(y,y';x)=y\triangleright^* y'
\quad\Longleftrightarrow\quad
F_{x\to y}\triangleright^*_M F_{x\to y'} .
\]
\end{lemma}
\begin{proof}
\textit{Query-level directional evidence scores.} Fix a query \(q=(y,y';x)\). Since \(M\)
is perfectly separable, for each priority \(j\in[m]\), there exists a constant
\(\delta_j(q)\) such that
\[
\Delta^G_j(q)
=
u_j(G_{x\to y})-u_j(G_{x\to y'})
=
\delta_j(q)
\qquad
\forall G\in\rules.
\]
Because \(\prules\) is unanimous,
\[
\prules\bigl((\Delta^G_j(q))_{G\in\rules}\bigr)
=
\delta_j(q).
\]
Therefore the query-level directional scores satisfy
\[
s^+_M(q)
=
\sum_{j\in[m]}
\omega_j
\left[
\prules\bigl((\Delta^G_j(q))_{G\in\rules}\bigr)
\right]_+
=
\sum_{j\in[m]}\omega_j[\delta_j(q)]_+,
\]
and
\[
s^-_M(q)
=
\sum_{j\in[m]}
\omega_j
\left[
-\prules\bigl((\Delta^G_j(q))_{G\in\rules}\bigr)
\right]_+
=
\sum_{j\in[m]}\omega_j[-\delta_j(q)]_+.
\]
\textit{Rule-Level Directional Evidence Scores.} Fix a background rule \(F\in\rules\); by perfect separability, $u_j(F_{x\to y})-u_j(F_{x\to y'})=\delta_j(q).$
Then, the rule-level directional scores between \(F_{x\to y}\) and
\(F_{x\to y'}\) are
\[
s^+_M(F_{x\to y},F_{x\to y'})
=
\sum_{j\in[m]}
\omega_j
[u_j(F_{x\to y})-u_j(F_{x\to y'})]_+
=
\sum_{j\in[m]}\omega_j[\delta_j(q)]_+,
\]
and
\[
s^-_M(F_{x\to y},F_{x\to y'})
=
\sum_{j\in[m]}
\omega_j
[u_j(F_{x\to y'})-u_j(F_{x\to y})]_+
=
\sum_{j\in[m]}\omega_j[-\delta_j(q)]_+.
\]
It follows that
\[
s^+_M(q)=s^+_M(F_{x\to y},F_{x\to y'}),
\qquad
s^-_M(q)=s^-_M(F_{x\to y},F_{x\to y'}).
\]
and thus the query-level and rule-level values of \(r_M\) and
\(\kappa_M\) are identical. Applying the same thresholds
\(\tau_r,\tau_\kappa\) therefore yields the same latent state.
\end{proof}

\appendix
\end{document}